\documentclass{article}
\usepackage[preprint]{neurips_2026}

\usepackage[utf8]{inputenc} 
\usepackage[T1]{fontenc}    
\usepackage{hyperref}       
\usepackage{url}            
\usepackage{booktabs}       
\usepackage{amsfonts}       
\usepackage{nicefrac}       
\usepackage{microtype}      
\usepackage{xcolor}         

\usepackage{enumitem}
\usepackage{algorithm}
\usepackage{algpseudocode}
\usepackage{enumitem}
\usepackage{comment}
\usepackage{placeins}
\usepackage{amsmath}
\usepackage{amsthm}
\usepackage{amssymb}
\usepackage{pgfplots}
\usepackage{pgfplotstable}
\usepackage{graphicx}
\usepackage{graphviz}
\usepackage{subcaption}
\usetikzlibrary{math}
\usetikzlibrary{shapes,arrows}
\usetikzlibrary{arrows.meta,positioning,calc,backgrounds,fit,decorations.pathreplacing}
\usetikzlibrary{backgrounds,fit}
\tikzset{>={Latex[width=2.5mm,length=2.5mm]}}
\usetikzlibrary{positioning}
\tikzstyle{block}=[draw opacity=0.7,line width=1.4cm]
\tikzset{arrow_e/.style = {->,> = latex'}}
\newtheorem{thm}{Theorem}[section]
\newtheorem{prop}{Proposition}[section]
\newtheorem{rem}{Remark}[section]

\newtheorem{lem}{Lemma}[section]
\newtheorem{defn}{Definition}

\newtheorem{assump}{Assumption}

\usepackage{algorithm}
\usepackage{algpseudocode}
\usepackage{booktabs}
\usepackage{multirow}

\usepackage{tikz}
\usetikzlibrary{shapes.geometric, arrows.meta, positioning, fit}

\tikzstyle{doc} = [rectangle, draw, text centered, minimum height=1.2em, minimum width=2.5em, fill=white]
\tikzstyle{db} = [cylinder, draw, shape border rotate=90, aspect=0.25, minimum height=2em, minimum width=1.5em, fill=blue!10]
\tikzstyle{llm} = [draw, thick, regular polygon, regular polygon sides=6, minimum size=1.5em, fill=orange!20]
\tikzstyle{arrow} = [thick, ->, >=Stealth]
\tikzstyle{block} = [rectangle, draw, fill=yellow!10, text width=6em, text centered, minimum height=3em]

\usepackage[most]{tcolorbox}

\newtcolorbox{insightbox}{
enhanced,
colback=gray!15,
colframe=white,
boxrule=0pt,
arc=0mm,
left=10pt,
right=10pt,
top=8pt,
bottom=8pt,
borderline west={3pt}{0pt}{blue!60!black}
}

\tcbset{
  sharp corners=all,
  arc=2mm,
  boxrule=0.5pt,
  colframe=black!40,
  colback=black!2,
  left=6pt,right=6pt,top=6pt,bottom=6pt,
  enhanced,
  breakable
}
\newtcolorbox{promptbox}[2][]{%
  title={#2},
  fonttitle=\bfseries,
  colbacktitle=black!5,
  coltitle=black,
  attach boxed title to top left={yshift=-2mm, xshift=2mm},
  boxed title style={boxrule=0pt, colframe=black!50},
  #1
}

\title{Hierarchical Prompt-Domain Control and Learning for Resource-Constrained Agentic Language Models}

\author{%
  Joan Vendrell Gallart \\
  Department of Mechanical and Aerospace\\
  University of California Irvine\\
  Irvine, CA 92617-4322, USA \\
  \texttt{jvendrel@uci.edu} \\
  \AND
  Russell Bent \\
  T-5 Los Alamos National Laboratory \\
  Los Alamos, NM 88220, USA \\
  \texttt{rbent@lanl.gov} \\
  \And
  Michael Grosskopf \\
  CAI-4 Los Alamos National Laboratory \\
  Los Alamos, NM 88220, USA \\
  \texttt{ mikegros@lanl.gov} \\
}

\begin{document}

\maketitle

\begin{abstract}
Large Language Models are increasingly deployed inside agentic systems, where they must follow structured protocols, adapt to evolving states, and operate under memory, latency, and cost constraints. In such regimes, prompt extension is unreliable: growing contexts can push compact models outside their effective prompt domain, while deployment-time fine-tuning remains limited by scarce data and compute. We propose a hierarchical control-and-learning framework in which a compact model is first distilled to learn the required output schema, then supervised online by an oracle-controller loop. The controller monitors protocol validity and semantic performance, projects accumulated histories into a feasible prompt domain, and triggers lightweight oracle-supervised fine-tuning under drift. This separates \emph{schema learning} for communication compatibility from \emph{semantic adaptation} for task-level correction. We formalize prompt-domain feasibility and attention-induced saturation, motivating control of the effective prompt state rather than reliance on nominal context length. Using Multi-Fidelity Bayesian Optimization as a controlled sequential testbed, we characterize a core deployment failure mode and show improved reliability and cost-efficiency over non-hierarchical, distillation-only, and non-distilled baselines.
\end{abstract}

\section{Introduction}
\label{sec:intro}

\emph{Agentic AI} represents an emerging paradigm in which \emph{Large Language Models} (LLMs) interact with external tools, environments, and other agents to perform coordinated reasoning and decision-making tasks. By distributing functionality across specialized components, these systems support structured tool use, planning, and experimentation. Their reliability, however, depends not only on the quality of individual models, but also on mechanisms that regulate communication, monitor failures, and adapt behavior over repeated interactions~\cite{YR-YL-TJ-XX:25}.

A central design choice in agentic systems is the selection of the underlying LLMs. Larger models often provide stronger reasoning and instruction-following capabilities, but their deployment may be limited by memory, latency, privacy, energy consumption, or monetary cost. Compact models are attractive in constrained settings, especially when decisions are repetitive, local, or protocol-driven. Yet, inside sequential agentic loops, compact models can fail in a way not fully captured by standard compression or prompting perspectives: as interaction histories evolve, the prompt state may drift outside the regime where the model reliably follows the required protocol or makes useful task-level decisions. We call this \emph{prompt-domain drift}. Because agentic models operate inside dynamical systems, such drift cannot always be removed by stronger pre-training or larger offline datasets; deployment itself can generate states that require online correction.

Existing adaptation strategies do not fully resolve this closed-loop problem. Prompt engineering~\cite{PS-AS-SS-VJ-SM-AC:24}, retrieval-augmented generation (\texttt{RAG})~\cite{PL-EP-AP-FP-VK-NG-HK-ML-WY-TR-SR-DK:20}, and chain-of-thought prompting~\cite{JW-XW-DS-MB-BI-FX-EC-QL-DZ:22} modify or expand the input, but do not control whether the resulting prompt remains feasible for the deployed model. Fine-tuning provides direct adaptation, but is costly under scarce online data. Knowledge distillation transfers behavior from larger to smaller models~\cite{XX-ML-TC-TS-RC-LJ-XC-TD-TZ:24}, but typically targets static output imitation rather than the interaction skills required by agentic systems~\cite{KL-DY-ZQ-WY-YW-HL-JL-PZ-YL-LZ:25}. This motivates separating two requirements that are often conflated: \emph{schema} constraints, which ensure communication compatibility, and \emph{semantic} constraints, which determine whether the action is appropriate for the current state. We therefore ask: \emph{how can compact language models be controlled and adapted online so that they remain both structurally valid and semantically useful under resource-constrained agentic deployment?}

We answer this question by proposing a hierarchical prompt-domain control and adaptive learning framework. The architecture combines three components: an oracle model, a compact adaptive student, and a controller. The student is first distilled offline to learn the required output schema, ensuring that it can communicate with the surrounding agentic system. During deployment, the controller monitors both protocol validity and semantic performance, projects accumulated interaction histories into a feasible prompt domain, and triggers lightweight oracle-supervised fine-tuning when drift is detected. Thus, the oracle-controller loop separates what should be learned offline from what must be corrected online: schema is treated as a structural requirement, while semantic behavior is adapted locally as the task state evolves, see Figure~\ref{fig:control_architecture}.

\begin{figure*}[t]
\centering
\captionsetup{font=scriptsize}

\resizebox{0.95\textwidth}{!}{%
\begin{tikzpicture}[
    >=Latex,
    thick,
    node distance=1.7cm and 1.8cm,
    every node/.style={transform shape},
    scale=0.92,
    block/.style={
        draw=violet!70!black,
        rounded corners=4pt,
        fill=violet!8,
        minimum width=3.0cm,
        minimum height=1.15cm,
        align=center
    },
    fbblock/.style={
        draw=gray!70!black,
        rounded corners=4pt,
        fill=gray!10,
        minimum width=3.6cm,
        minimum height=1.05cm,
        align=center
    },
    sum/.style={
        draw,
        circle,
        minimum size=7mm,
        inner sep=0pt,
        fill=white
    },
    dashededge/.style={->, dashed, violet!70!black}
]

\node (input) {\small \textbf{Reference / Task State} $\mathcal{S}_t$};
\node[sum, right=1.2cm of input] (sum) {};
\node[block, right=1.8cm of sum] (controller) {\textbf{Controller} $\mathcal{C}$\\[-1mm]
\scriptsize projection \& update logic};
\node[block, right=2.1cm of controller] (student) {\textbf{Adaptive Student} $\mathcal{U}_\theta$\\[-1mm]
\scriptsize local policy $\pi_\theta$};
\node[right=1.4cm of student] (output) {\small \textbf{Output} $y_t$};

\node[fbblock, below=1.9cm of student] (oracle) {\textbf{Oracle / Feedback Evaluator} $\mathcal{O}_\phi$\\[-1mm]
\scriptsize semantic validity + task drift};

\draw[->] (input) -- (sum);
\draw[->] (sum) -- node[above] {\small error $e_t$} (controller);
\draw[->] (controller) -- 
node[above, align=center] {\small projected \\ \small prompt $\hat{\mathcal{S}}_t$} 
(student);
\draw[->] (student) -- (output);

\coordinate (ysplit) at ($(output.south)+(0,-0.85)$);
\draw[-] (output.south) -- (ysplit);

\draw[->] (ysplit) |- node[pos=0.25, right] {\small semantic component} (oracle.east);
\draw[->] (oracle.west) -| node[pos=0.23, below] {\small semantic feedback / drift} (sum.south);

\draw[->] (ysplit) -| node[pos=0.38, below] {\small schema / protocol signal} (controller.south);

\draw[dashededge] (oracle.north) -- node[right] {\small oracle call if needed} (student.south);

\node at ($(sum.west)+(-0.18,0.20)$) {\scriptsize $+$};
\node at ($(sum.south)+(0.20,-0.18)$) {\scriptsize $-$};

\end{tikzpicture}%
}

\caption{Control-theoretic view of the proposed framework during deployment. The adaptive student acts in the forward path and produces the output $y_t$. This response is decomposed into a schema/protocol component, sent directly to the controller, and a semantic component, evaluated by the oracle in the feedback path. The oracle returns a semantic feedback or drift signal, while the controller regulates prompt projection and determines when corrective intervention is required.}
\label{fig:control_architecture}
\end{figure*}
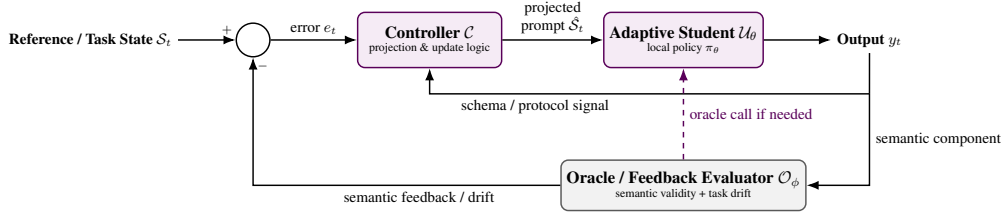

Our formulation is related to schema-guided and verifier-guided agentic learning. Schema-guided multi-agent reasoning systems highlight the role of explicit intermediate representations for reliable coordination~\cite{YC-HS-NG-YJ-JS-PV-BL:25}, while local reinforcement learning approaches such as PivotRL~\cite{JY-DM-BH-RG-CW-SD-KB-AG-OK-JJ-JZ-VS:26} show that agentic post-training can focus updates on informative intermediate states, see Appendix~\ref{appendix:connection_literature}. Our framework follows this direction but differs in where feasibility is enforced: rather than relying only on output-level verification or reward shaping, it imposes feasibility at the prompt-state level through an explicit domain geometry and a projection operator.

We instantiate the framework in a Multi-Fidelity Bayesian Optimization (MFBO) environment. MFBO is not intended as exhaustive empirical coverage across agentic tasks such \cite{ZS-XF-ZH-ZX-LR-SA-CX-OT-BY-FD-AU-NG:23,MS-XY-MC-YB-AT-MH:21}; instead, it provides a controlled sequential decision-making testbed where structured outputs, iterative prompt growth, and resource-constrained inference arise simultaneously. This makes it a useful setting for isolating prompt-domain drift and evaluating whether hierarchical supervision stabilizes compact-model behavior. Similar proposal--evaluation--refinement structures appear in hyperparameter optimization and Bayesian optimization~\cite{JS-HL-RA:12,BS-KS-ZG-RPA:16}, neural architecture search~\cite{BZ-QL:17,ER-IA-SL-QVL:19}, active learning and adaptive data selection~\cite{BS:09}, experimental design for scientific discovery~\cite{TL:19,MG-RB-RS-IM-AL-ND-EL:25}, mixture-of-experts routing and adaptive model selection~\cite{GU:26}, and reinforcement learning for long-horizon agentic tasks~\cite{LO-JW-XJ-DM-SW-CL-JS-RL:22,JY-DM-BH-RG-CW-SD-KB-AG-OK-JJ-JZ-VS:26}. See Appendix~\ref{appendix:connection_literature} for a broader discussion of how the MFBO setting captures core challenges from current agentic-AI literature.

Our contributions are as follows:
\begin{itemize}[leftmargin=0.5em,itemsep=0.25em,topsep=0.25em]
    \item \textbf{(C1)} We formulate resource-constrained agentic deployment as a hierarchical control-and-learning problem that separates offline schema learning from online semantic adaptation.
    \item \textbf{(C2)} We introduce an oracle--student--controller architecture that monitors protocol validity, semantic performance, prompt-domain drift, and adaptive fine-tuning.
    \item \textbf{(C3)} We define a feasibility-aware prompt projection mechanism that maps growing interaction histories into a bounded prompt domain, mitigating saturation and uncontrolled context expansion.
    \item \textbf{(C4)} We evaluate the framework in a controlled MFBO setting and show that hierarchical supervision improves reliability and cost-efficiency relative to non-hierarchical, distillation-only, and non-distilled variants.
\end{itemize}

The remainder of the paper is organized as follows. Section~\ref{sec:prelim} introduces preliminary knowledge while Section~\ref{sec:problem} formalizes the agentic prompt-state setting. Section~\ref{sec:architecture} presents the hierarchical control-and-learning architecture. Section~\ref{sec:hierarchy} develops the controller and prompt-domain projection mechanism. Section~\ref{sec:case} instantiates the framework in MFBO and reports the empirical study. Section~\ref{sec:conclusions} concludes with limitations and directions for broader evaluation.

\section{Preliminaries}
\label{sec:prelim}

We use discrete optimization language to describe feasible prompt evolution. Let $\mathbb{V}$ be a ground set of atomic elements, such as tokens, phrases, tool calls, or semantic units, and let  $\mathbb{V}^*$ denote the set of finite sequences over $\mathbb{V}$. For a sequence  $\mathsf{S}=(s_1,\ldots,s_\ell)$, let $|\mathsf{S}|=\ell$. We write $\mathsf{S}^1\oplus\mathsf{S}^2$ for concatenation and $\mathsf{S}^1\preceq\mathsf{S}^2$ when  $\mathsf{S}^1$ is a prefix of $\mathsf{S}^2$.

In sequential agentic tasks, not every sequence of states or outputs is admissible: feasibility may depend on the full prefix history. We model such path-dependent constraints by a feasible family  $\mathcal{I}\subseteq\mathbb{V}^*$.

\begin{defn}[Greedoid~\cite{BK-LL:81}]
A pair $(\mathbb{V},\mathcal{I})$ is a \emph{greedoid} if:
\textbf{(G1)} $\emptyset\in\mathcal{I}$;
\textbf{(G2)} for any nonempty $\mathsf{S}\in\mathcal{I}$, there exists $s\in\mathsf{S}$ such that 
$\mathsf{S}\ominus\{s\}\in\mathcal{I}$; and
\textbf{(G3)} for any $\mathsf{S}_1,\mathsf{S}_2\in\mathcal{I}$ with 
$|\mathsf{S}_1|<|\mathsf{S}_2|$, there exists $s\in\mathsf{S}_2\ominus\mathsf{S}_1$ such that 
$\mathsf{S}_1\oplus s\in\mathcal{I}$.
\end{defn}

For an objective $f:\mathbb{V}^*\to\mathbb{R}_{\ge 0}$, we say that $f$ is \emph{string submodular}  if for any $\mathsf{S}^1\preceq\mathsf{S}^2$ and any $s\in\mathbb{V}$ not appearing in  $\mathsf{S}^2$,
\begin{equation}
\label{eqn:string_submodularity}
f(\mathsf{S}^1\oplus s)-f(\mathsf{S}^1) \ge f(\mathsf{S}^2\oplus s)-f(\mathsf{S}^2).
\end{equation}
The function is \emph{normalized} if $f(\emptyset)=0$ and \emph{prefix-monotone} if  $f(\mathsf{S}^2)\ge f(\mathsf{S}^1)$ whenever $\mathsf{S}^1\preceq\mathsf{S}^2$. String  submodularity captures diminishing returns over ordered prefixes and provides a compact language  for reasoning about controlled prompt growth and feasible sequential refinement.

\section{Problem Setting}
\label{sec:problem}

We model an agentic AI system as a directed interaction graph 
$\mathcal{G}=(\mathcal{V},\mathcal{E})$, where each node 
$v\in\mathcal{V}$ is either an LLM-based agent or an external tool/environment with local policy $f_v$, and each edge $(i,j)\in\mathcal{E}$ is an information channel. At iteration $t$, $w_{i,j}(t)\in\{0,1\}$ indicates whether communication from $i$ to $j$ is active. Starting from an input prompt $\rho$ at $v_0$, an active execution path evolves through $f(\mathcal{G},\rho)=f_{v_n}\circ f_{v_{n-1}}\circ\cdots\circ f_{v_0}(\rho),$ and terminates with output $y$ at a terminal node $v_n$. Each local policy must map between compatible domains,
$f_{v_t}:\mathcal{D}_{v_t}\rightarrow\mathcal{D}_{v_{t+1}},$ with $ x_{t+1}=f_{v_t}(x_t)$ and $ x_t\in\mathcal{D}_{v_t},$ otherwise downstream communication may fail.

\begin{defn}[Feasible prompt space]
A prompt $\rho$ belongs to the feasible prompt space $\mathcal{D}$ of a model $\pi_\theta$ if its response $y=f_\theta(\rho)$ satisfies the structural and semantic constraints required by the target system, such as a valid JSON object, tool call, numerical decision, or non-empty tensor. Otherwise, $\rho$ is out-of-domain.
\label{defn:faulty}
\end{defn}

We decompose feasibility as
$\mathcal{D}=\mathcal{D}_{\mathrm{sch}}\cap\mathcal{D}_{\mathrm{sem}},$ where $\mathcal{D}_{\mathrm{sch}}$ captures schema or protocol validity and $\mathcal{D}_{\mathrm{sem}}$ captures task-level semantic validity. At the local level, each LLM-based agent $\ell$ is abstracted as a Semi-Markov Decision Process $\mathcal{M}^{\ell}_t=(\mathcal{S}^{\ell}_t,\mathcal{A}^{\ell}_t,P^{\ell}_t,R^{\ell}_t),$
where the state is the active context, the action selects the next response or node, the transition is induced by $\pi_{\theta_t}$, and the reward is task-dependent~\cite{RS-DP-SS:99}. In the teacher--student setting, the oracle policy $\pi_\phi$ supervises the adaptive policy $\pi_\theta$.

Formally, given a prompt-response pair $(\rho,y)$, we encode oracle-aligned behavior through
\[
R_w(\rho,y)
=
\sum_{t=1}^T
w_t(\rho,y_{<t})
\log\pi_{\phi}(y_t|\rho,y_{<t}),
\]
where $w_t$ selects the tokens relevant to the desired supervision. Taking $w_t\equiv1$ gives the semantic reward $R_s(\rho,y)=\log\pi_\phi(y|\rho)$, while choosing $w_t$ to emphasize protocol-defining tokens gives the schema reward $R_p(\rho,y)$. The total reward is
\[
    R(\rho,y)
    =
    \alpha R_s(\rho,y)
    +
    \beta R_p(\rho,y),
    \qquad
    \alpha+\beta=1.
\]
The deployment objective is to learn an adaptive policy that maximizes oracle-aligned behavior while remaining feasible and resource-constrained:
\begin{equation}
\label{eqn:main_optimiziation_problem}
\begin{aligned}
    \max_{\theta}\quad
    J(\theta)
    =
    \mathbb{E}_{\rho\sim\mathcal{D}}
    \mathbb{E}_{y\sim\pi_{\theta}(\cdot|\rho)}
    [R(\rho,y)] \quad \mathrm{s.t.}\quad
    &\rho\in\mathcal{D}_{\mathrm{sch}}\cap\mathcal{D}_{\mathrm{sem}},\\
    &\mathsf{Res}(\rho,\theta)\le \mathcal{C},
\end{aligned}
\end{equation}
where $\mathsf{Res}(\rho,\theta)$ collects memory, latency, and cost, and $\mathcal{C}$ denotes the corresponding resource budget. Section~\ref{sec:architecture} instantiates~\eqref{eqn:main_optimiziation_problem} by separating offline schema distillation from online semantic adaptation.

\section{Architecture}
\label{sec:architecture}

This section introduces the proposed architecture for resource-constrained agentic deployment. As shown in Figure~\ref{fig:control_architecture}, an oracle $\mathcal{O}(\phi,\cdot)$ supervises an adaptive model $\mathcal{U}(\theta,\cdot)$ only when needed, while a controller monitors the prompt state, enforces feasibility, and regulates corrective updates.

\begin{assump}[Limited static transfer~\cite{JK-SM-TH-TB-BC-RC-SG-AR-JW-DA:20}]
    \label{assump:dataset_size}
    Effective transfer of information to model parameters $|\mathcal{T}_T|$ grows sub-linearly with dataset and model size:
    \[
        |\mathcal{T}_T|
        =
        k\,|\mathcal{T}_F|^{\delta_1}|\theta_i|^{\delta_2},
    \]
    where $|\mathcal{T}_F|$ is the dataset size, $|\theta_i|$ is the number of model parameters, 
    $k\approx 0.1$~\cite{JH-SB-AM:22}, and 
    $\delta_1\in[0.25,0.35]$, $\delta_2\in[0.3,0.5]$ are task-dependent constants.
\end{assump}

Assumption~\ref{assump:dataset_size} formalizes a basic limitation of static transfer: the amount of information that can be reliably absorbed by a model grows sub-linearly with both dataset size and model size. Thus, learning more behavior requires additional data, memory, time, and computation. Under the resource constraint in~\eqref{eqn:main_optimiziation_problem}, the available data and compute may be insufficient for a compact model to learn both protocol validity and task semantics reliably as a single monolithic objective. We therefore split the learning problem into two components with different roles and time scales.

First, \emph{schema learning} enforces $\rho\in\mathcal{D}_{\mathrm{sch}}$ through offline distillation: protocol validity is a necessary condition for the agentic system to communicate with downstream tools and is largely static once the output interface is fixed. Second, \emph{semantic adaptation} addresses $\rho\in\mathcal{D}_{\mathrm{sem}}$ during deployment: task-level correctness depends on the evolving state of the environment and must remain locally adaptable. The next subsections instantiate these two phases.

\subsection{Phase $1$: Distillation}

A necessary condition for agentic deployment is that the adaptive model communicates in the feasible schema domain $\mathcal{D}_{\mathrm{sch}}$. Since the output interface is fixed by the surrounding agentic system, this component can be learned before deployment. Phase~$1$ therefore initializes $\mathcal{U}$ through offline schema distillation: given prompt $\mathsf{S}_i$, the oracle generates target states
$\mathcal{S}_{i+1}=\mathcal{O}(\phi,\mathsf{S}_i)$, forming a supervised dataset $\mathcal{T}$. The adaptive model is then trained to imitate the protocol structure of these oracle responses, rather than to fully compress the oracle's task-level reasoning; see Appendix~\ref{appendix:case_study}.

Following supervised fine-tuning (SFT)~\cite{DG-YH-KL-LC-XM-DL-WW-YZ-CZ-JZ:24}, Algorithm~\ref{alg:distillation} minimizes a KL divergence between a protocol-reweighted oracle policy and the adaptive policy:
\begin{equation}
    \label{eq:imitation-loss}
    \mathcal{L}_{\mathrm{dis}}(\theta,\mathcal{B})
    =
    \mathbb{E}_{(\mathcal{S}_t,\mathcal{S}_{t+1})\sim\mathcal{B}}
    \left[
    D_{\mathrm{KL}}\!\left(
    \bar{\pi}_{\phi}(\cdot\mid\mathcal{S}_t)
    \;\Vert\;
    \pi_{\theta}(\cdot\mid\mathcal{S}_{t+1})
    \right)
    \right],
\end{equation}
where $\bar{\pi}_{\phi}$ up-weights protocol-defining tokens,
$\bar{\pi}_{\phi}(s\mid\mathcal{S}_t)
=\frac{\pi_{\phi}(s\mid\mathcal{S}_t)w_t(\mathcal{S}_t;y_{<t}\circ s)}
{\sum_{s'\in\mathbb{V}}\pi_{\phi}(s'\mid\mathcal{S}_t)w_t(\mathcal{S}_t;y_{<t}\circ s')}.$

Thus, Phase~$1$ solves the schema component of~\eqref{eqn:main_optimiziation_problem}: it enforces structural compatibility with downstream tools while avoiding the data and compute cost of full semantic transfer~\cite{CH-CL-CY-HN-YF-AR-RK-CY-TP:23}. In terms of the reward decomposition, this corresponds to the protocol-only regime $\alpha=0,\beta=1$.

\begin{prop}
    \label{prop:equivalence_loss}
    Minimizing $\mathcal{L}_{\mathrm{dis}}(\theta,\mathcal{B})$ provides a surrogate of \eqref{eqn:main_optimiziation_problem} when $\beta = 1$.
\end{prop}
\begin{proof}
    See Appendix~\ref{appendix:proofs}.
\end{proof}

\begin{algorithm}
    \scriptsize
    \caption{Offline Schema Distillation}
    \label{alg:distillation}
    \begin{algorithmic}[1]
        \State \textbf{Input:} Dataset $\mathcal{T}$, epochs $\text{max}\_\text{epochs}$, adaptive model $\mathcal{U}(\theta,\cdot)$.
        \For{$t\in\text{max}\_\text{epochs}$}
            \For{$b\in\mathcal{T}$}
                \State $\mathcal{B}=\{(\mathcal{S}_i,\mathcal{O}(\phi,\mathcal{S}_i))\}_{i=1}^{\kappa}\leftarrow\mathcal{T}(b)$
                \State $\mathcal{L}_{\mathrm{dis}}(\theta,\mathcal{B})\leftarrow$ Equation~\eqref{eq:imitation-loss}
                \State $\theta\leftarrow\theta-\lambda_r\nabla_{\theta}\mathcal{L}_{\mathrm{dis}}(\theta,\mathcal{B})$
            \EndFor
        \EndFor
        \State \textbf{Return:} $\mathcal{U}(\theta,\cdot)$
    \end{algorithmic}
\end{algorithm}

\subsection{Phase $2$: Adaptive Fine-Tuning}

While offline schema distillation gives a structurally valid initialization, it does not solve the dynamic part of~\eqref{eqn:main_optimiziation_problem}. During deployment, the prompt state evolves, numerical effects accumulate, and the adaptive policy may deviate from the oracle on task-level decisions. Phase~$2$ therefore addresses the semantic component $\rho\in\mathcal{D}_{\mathrm{sem}}$ through lightweight adaptive fine-tuning. The controller monitors semantic performance and prompt feasibility during execution; it collects oracle--student pairs in a bounded buffer $\mathcal{T}$ and updates $\mathcal{U}(\theta,\cdot)$ locally when drift is detected.

\begin{rem}
    The buffer dataset size is bounded as $\underline{d}\leq|\mathcal{T}|\leq\bar{d}$ to preserve learning stability and respect resource constraints.
    \label{rem:bounded_dataset}
\end{rem}

\begin{algorithm}
    \scriptsize
    \caption{Online Semantic Adaptation}
    \label{alg:finetuning}
    \begin{algorithmic}[1]
        \State \textbf{Input:} Initial prompt $\rho$, adaptive model $\mathcal{U}(\theta,\cdot)$, oracle model $\mathcal{O}(\phi,\cdot)$.
        \State $\mathcal{S}_0\leftarrow\rho$
        \While{not finishing condition}
            \State $\mathcal{T}\leftarrow\{\}$
            \While{$|\mathcal{T}|\leq\kappa$ and drift not detected}
                \State $\mathcal{S}_{t}\leftarrow\mathcal{O}(\phi,\mathcal{S}_t)$
                \State $\hat{\mathcal{S}}_{t}\leftarrow\mathcal{U}(\theta,\mathcal{S}_t)$
                \State $\mathcal{T}\leftarrow\mathcal{T}\cup\{(\mathcal{S}_{t},\hat{\mathcal{S}}_{t})\}$
            \EndWhile
            \State $\mathcal{L}_{\mathrm{ft}}(\theta,\mathcal{T})\leftarrow$ Equation~\eqref{eq:consistency-loss}
            \State $\theta\leftarrow\theta-\lambda_c\nabla_{\theta}\mathcal{L}_{\mathrm{ft}}(\theta,\mathcal{T})$
        \EndWhile
        \State \textbf{Return:} $\mathcal{U}(\theta,\cdot)$
    \end{algorithmic}
\end{algorithm}

As shown in Algorithm~\ref{alg:finetuning}, the buffer is constructed online from paired oracle and adaptive responses, so that updates focus on states actually encountered during deployment. The adaptive model is corrected using the semantic consistency loss
\begin{equation}
    \label{eq:consistency-loss}
    \mathcal{L}_{\mathrm{ft}}(\theta,\mathcal{T})
    =
    \mathbb{E}_{\mathcal{S}_t \sim \mathcal{T}}
    \left[
        D_{\mathrm{KL}}\!\left(
        \pi_{\phi}(\cdot \mid \mathcal{S}_t)
        \;\Vert\;
        \pi_{\theta}(\cdot \mid \hat{\mathcal{S}}_t)
        \right)
    \right],
\end{equation}
where $\hat{\mathcal{S}}$ stands for the adaptive model generated states. Thus, Phase~$2$ instantiates the semantic component of~\eqref{eqn:main_optimiziation_problem}, corresponding to the semantic-only regime $\alpha=1,\beta=0$. In contrast to Phase~$1$, which learns the static schema required for communication, Phase~$2$ performs local policy correction so that the compact model remains aligned with the oracle as the task state evolves.

\section{Hierarchical Refinements}
\label{sec:hierarchy}

As shown in Figure~\ref{fig:control_architecture}, the previous section addressed model initialization and adaptive training through the oracle--student feedback loop. We now introduce the third component of the framework: the controller. Its role is to regulate the prompt state before it is passed to the adaptive model, ensuring that deployment remains within a feasible operating region. This motivates two complementary questions: how long a prompt can be under the available resources, and how the accumulated state should be projected when that limit is approached.

\begin{lem}[Prompt Length]
  \label{lem::cache}
  For a sequence of length $|\rho|$, batch size $B$, $L$ layers, and $n_{\mathrm{kv}}$ heads of dimension $d_h$ with precision $b_{\mathrm{kv}}$, and given available device memory $M_{\max}$, model weights $M_w$, and runtime overhead $M_{\mathrm{misc}}$, a feasible prompt length $|\rho|_{\mathrm{feasible}}$ is upper-bounded by
  \begin{equation}
  \label{eq:prompt-upper-bound}
  |\rho|_{\mathrm{feasible}}
  \leq
  \min \left\{
    S_{\mathrm{model}},
    \left\lfloor 
      \frac{M_{\max}-M_w-M_{\mathrm{misc}}}
           {B\,L\,(2\,n_{\mathrm{kv}}\,d_h)\,b_{\mathrm{kv}}}
    \right\rfloor
  \right\}.
\end{equation}
\end{lem}

\begin{proof}
    By~\cite{RP-SD-AC-JD-JB-AL-JH-KX-SA-JD:23}, the required cache memory is 
    $M_{\mathrm{KV}}(|\rho|)=B\,|\rho|\,L\,(2\,n_{\mathrm{kv}}\,d_h)\,b_{\mathrm{kv}}$, which gives the stated bound.
\end{proof}

Lemma~\ref{lem::cache} gives a hardware-level constraint: if the prompt exceeds the KV-cache-supported range, generation may fail directly. However, empirical failures can occur much earlier than this upper bound. The reason is that long prompts can also induce \emph{model-internal saturation}. The prompt is technically admissible, but the model no longer uses the task-relevant subsequence with sufficient influence, Theorem~\ref{thm:logit_saturation}.

\begin{thm}[Prompt saturation through logit-level attention dilution]
\label{thm:logit_saturation}
Let $(\mathbb{V},\mathcal{I})$ be a greedoid over textual sequences, and let
$\mathsf{S}_{1:\ell}=(s_1,\ldots,s_\ell)\in\mathcal{I}$ be a feasible prompt sequence. For a query position $t$, let $\mathsf{R}\subseteq\mathsf{S}_{1:\ell}$ be a fixed task-relevant subsequence with $|\mathsf{R}|=m$. If the attention scores are uniformly bounded as in Lemma~\ref{lem:attention_dilution}, and the corresponding value vectors are uniformly bounded, then the logit-level contribution induced by $\mathsf{R}$ satisfies
$$\lim_{\ell\to\infty}\left\|\Delta L_{\mathsf{R}}^{(\ell)}\right\|=0.$$
Consequently, the influence of any fixed task-relevant subsequence on the output logits vanishes as the feasible prompt length grows.
\end{thm}
\begin{proof}
    See Appendix~\ref{appendix:proofs}.
\end{proof}

Together, these results motivate controlling the \emph{effective} prompt state, rather than only respecting the nominal context length. A prompt can fail because it exceeds memory limits, or because relevant information becomes too weakly represented to support reliable generation. In both cases, the controller should prevent the adaptive model from operating outside its stable prompt regime.

\begin{prop}[Attention-induced degradation ratio]
\label{prop:attention_degradation_ratio}
Let $|\rho|_{\mathrm{feasible}}$ be the hardware-feasible prompt length of
$\rho$ from Lemma~\ref{lem::cache}. Suppose that feasibility of
the prompt domain $\mathcal{D}$ requires at least $\tau>0$ attention mass on a
task-relevant subsequence $\mathsf{R}$ of size $m$. If, for every layer $\ell$
and head $h$, the query and key vectors satisfy
$\|q_t^{(\ell,h)}\|\le Q_{\ell,h},~\|k_j^{(\ell,h)}\|\le K_{\ell,h}$,
then the attention-induced degradation ratio satisfies
\begin{equation}
\label{eqn:degradation_ratio}
\Delta_{\rho}
:=\frac{|\rho|_{\mathrm{sat}}}{|\rho|_{\mathrm{feasible}}}
\le
\frac{m}{\tau |\rho|_{\mathrm{feasible}}}
\exp\left(
\max_{\ell,h}
\frac{2Q_{\ell,h}K_{\ell,h}}{\sqrt{d_h}}
\right).
\end{equation}
\end{prop}
\begin{proof}
    See Appendix~\ref{appendix:proofs}.
\end{proof}

\paragraph{Prompt projection.}
The preceding results show that feasibility is not determined by a single
context-length constraint: the prompt may fail either by exceeding the
hardware-supported range in Lemma~\ref{lem::cache}, or by entering the
saturation regime characterized in Theorem~\ref{thm:logit_saturation}. The
controller therefore acts on the accumulated sequence before it is passed to the
adaptive model. It applies a projection operator
\begin{equation}
\label{eqn:projection_operator}
\mathcal{P}_{\mathcal{O}}:\mathbb{V}^*\rightarrow\mathcal{D},
\qquad
\widetilde{\mathsf{S}}_t=\mathcal{P}_{\mathcal{O}}(\mathsf{S}_t),
\end{equation}
where $\mathsf{S}_t\in\mathbb{V}^*$ is the accumulated sequence and
$\mathcal{D}\subseteq\mathbb{V}^*$ is the feasible prompt domain of $\rho$
respecting Proposition~\ref{prop:attention_degradation_ratio}. The role of
$\mathcal{P}_{\mathcal{O}}$ is not merely to shorten the prompt, but to keep the
induced prompt inside the stable operating region while preserving the
information needed for the next decision.

Algorithm~\ref{alg:algorithm_projection} gives the resulting controller loop.
The adaptive model acts on the projected sequence as long as the induced prompt
remains feasible. When the controller detects a violation of $\mathcal{D}$,
either through resource limits or saturation-driven degradation, it projects the
accumulated sequence back into the feasible domain. This induces the
hierarchical behavior of the framework: local decisions are made within feasible
prompt regions, while projection defines higher-level transitions between such
regions.

\begin{algorithm}
    \scriptsize
    \caption{Hierarchical Prompt Projection}
    \label{alg:algorithm_projection}
    \begin{algorithmic}[1]
        \State \textbf{Input:} initial sequence $\mathsf{S}_0$, prompt map $\rho$, adaptive model $\mathcal{U}$, oracle model $\mathcal{O}$.
        \State $\mathsf{S}\leftarrow\mathsf{S}_0$
        \While{not finishing condition}
        \State $\widetilde{\mathsf{S}}\leftarrow\mathcal{P}_{\mathcal{O}}(\mathsf{S})$
        \While{$\rho(\widetilde{\mathsf{S}})\in\mathcal{D}$} 
          \State $s\leftarrow\mathcal{U}(\rho(\widetilde{\mathsf{S}}))$
          \State $\widetilde{\mathsf{S}}\leftarrow\widetilde{\mathsf{S}}\oplus s$
        \EndWhile
        \State $\mathsf{S}\leftarrow\widetilde{\mathsf{S}}$
        \EndWhile
        \State \textbf{Return:} $\mathsf{S}$
    \end{algorithmic}
\end{algorithm}

Although $\mathcal{P}_{\mathcal{O}}$ can be learned and regularized to avoid
out-of-domain projections, see Appendix~\ref{appendix:learned_projection}, we
use a greedy rule-based projection which, not only preserves theoretical guarantees, unlike other literature such as \cite{JH-WQ-LC-YY-QL:23,CP-VF-SP-KL-SW-JG:23}, but also covers general agentic tasks, see Appendix~\ref{appendix:projection_beyond_mfbo}. The motivation is that accumulated prompt information in sequential decision-making often exhibits diminishing returns: after several related evaluations, retaining every element of $\mathsf{S}_t$ is less useful than preserving a representative informative subsequence~\cite{AK-JH:25}. This matches the string-submodular viewpoint of Section~\ref{sec:prelim}, so the controller selects informative elements under the prompt budget rather than applying arbitrary truncation.

\begin{prop}[Greedy information summarization]
\label{prop:greedy_info_summary}
Let $f:\mathbb{V}^*\rightarrow\mathbb{R}_{\ge 0}$ be normalized,
prefix-monotone, and string submodular, and suppose that
$\mathcal{P}_{\mathcal{O}}$ selects a feasible summary
$\widetilde{\mathsf{S}}\in\mathcal{D}$ of bounded length. Then the greedy
projection achieves the classical $(1-e^{-1})$-approximation guarantee with
respect to the best feasible summary under the same budget.
\end{prop}
\begin{proof}
    See~\cite{SA-AM-AM:21}.
\end{proof}

Thus, projection reduces prompt length while retaining a near-optimal
informative summary. Hierarchical refinement can therefore be viewed as a
constrained sequential control policy: the adaptive model makes local sequence
extensions, while the controller prevents memory overflow and internal
saturation by keeping the evolving sequence feasible and informative.
\section{Case Study}
\label{sec:case}

We instantiate the proposed framework in a controlled experimental planning problem based on \emph{Multi-Fidelity Bayesian Optimization} (MFBO). MFBO is not intended as a universal benchmark for agentic AI; rather, it is a compact testbed where the main difficulties studied in this paper arise together: structured outputs, iterative prompt growth, semantic decision-making, and resource-constrained inference. The agent must select, under a finite time budget, both the next point $\mathbf{x}\in X$ to evaluate and the fidelity model $p\in\mathcal{P}$ to use. Each fidelity model has an accuracy level $\alpha(p)$ and an execution time $t_p(\mathbf{x})$. A formal description of the MFBO objective is provided in Appendix~\ref{appendix:MFBM}.

The prompt state $\mathcal{S}_t$ consists of six evolving variables: model information, evaluated points, errors, domain, uncertainty, and remaining time. Because these variables grow throughout the optimization loop, MFBO naturally induces prompt-domain drift and provides a controlled case study for testing projection, schema preservation, and online semantic correction. The adaptive model must return a structured decision of the form \texttt{RESULT=[model,point]}.

We evaluate the three-agent architecture of Figure~\ref{fig:control_architecture} using two adaptive models, \emph{Llama-3.1-8B} and \emph{Mistral-7B}, and two oracle models, \emph{GPT-5} and \emph{GPT-5-nano}, integrated in \texttt{URSA}: The Universal Research and Scientific Agent~\cite{MG-RB-RS-IM-AL-ND-EL:25}. The resulting teacher--student systems are compared against oracle-only, distillation-only, and non-distilled adaptive baselines. The section first studies learning capacity, then prompt feasibility, and finally end-to-end hierarchical performance.


\subsection{Learning Capacity}
\label{subsec:case_learning}

We first evaluate how much data is required for the two learning stages. Phase~$1$ measures protocol transfer: the fraction of responses that satisfy the required output schema,
\begin{equation}
    \label{eqn:accuracy_1}
    \mathrm{Accuracy} =
    \frac{1}{n}\sum_{i=1}^{n}\mathbf{1}_{\{y_i~\text{is acceptable}\}}.
\end{equation}
We fine-tune only a small LoRA subset of the adaptive models~\cite{YJ-DA-HK-TK-EP:25}: 
{\scriptsize \texttt{self\_attn.q\_proj}, \texttt{self\_attn.k\_proj}, \texttt{self\_attn.v\_proj}, \texttt{self\_attn.o\_proj}, \texttt{mlp.gate\_proj}, \texttt{mlp.up\_proj}, \texttt{mlp.down\_proj}}.
Using rank-$32$ adapters, this corresponds to approximately $90.18$M trainable parameters for \emph{Llama-3.1-8B} and $86.80$M for \emph{Mistral-7B}, i.e., about $1.1$--$1.2\%$ of the model parameters. All local experiments are run under a constrained NVIDIA Tesla V100 16GB environment, see Appendix~\ref{app:ursa_mfbo}.

\begin{table}[ht]
  \centering
  \captionsetup{font=scriptsize}
  \caption{Sample efficiency of the distillation process across teacher--student pairs.
  We report the smallest dataset size and training configuration that achieves the best observed format accuracy on the distillation validation set. Full sweeps are in Tables~\ref{tab:datasize_format_gpt5_llama}--\ref{tab:datasize_format_gpt5_mistral}. Time in $h:min$.}
  \label{tab:datasize_summary}
  \scriptsize
  \begin{tabular}{lcccc}
    \toprule
     \textbf{Adaptive Model} & \textbf{Data size} & \textbf{Epochs} & \textbf{Time} & $\boldsymbol{|\mathcal{T}_T|}$ \\
    \midrule
    Llama-3.1-8B & 5000 & 5  & 8:35 & 1588.70 \\
    Mistral-7B  & 5000 & 10 & 8:52 & 755.57\\
    \bottomrule
  \end{tabular}
\end{table}

Phase~$2$ evaluates semantic adaptation. We measure fidelity-model selection using \eqref{eqn:accuracy_1} and point selection using the mean $\ell_1$ error
\begin{equation}
\label{eqn:accuracy_2}
\mathrm{Error}
=
\frac{1}{n}\sum_{i=1}^{n} |\hat{x}_i - x_i|,
\end{equation}
where $\hat{x}_i$ is the adaptive prediction and $x_i$ is the oracle target. Table~\ref{tab:criteria_summary} reports the smallest configuration that gives useful deployment-time adaptation for each teacher--student pair.

\begin{table*}[ht]
  \centering
  \captionsetup{font=scriptsize}
  \caption{Decision-criteria adaptation across teacher--student pairs.
  Acc~(Model) measures fidelity-model selection; Error~(Point) is the mean $\ell_1$ point-selection error, lower is better. Full data-size and epoch sweeps are provided in Tables~\ref{tab:datasize_criteria_llama_gpt5}--\ref{tab:datasize_criteria_mistral_gpt5nano}.}
  \label{tab:criteria_summary}
  \scriptsize
  \begin{tabular}{l l c c c c c}
    \toprule
    \textbf{Oracle Model} & \textbf{Adaptive Model} & \textbf{Data size} & \textbf{Epochs} & \textbf{Acc (Model)} $\uparrow$ & \textbf{Error (Point)} $\downarrow$ & \textbf{Time} $\downarrow$ \\
    \midrule
    gpt-5      & Llama-3.1-8B &  50 & 5 & 0.42 & 22.81 & 5:51 \\
    gpt-5      & Mistral-7B   &  10 & 10 & 0.08 & 61.37 & 3:16 \\
    gpt-5-nano & Llama-3.1-8B &  50 & 5 & 0.41 & 23.11 & 5:48  \\
    gpt-5-nano & Mistral-7B   &  50 & 5 & 0.07 & 58.92 & 5:44 \\
    \bottomrule
  \end{tabular}
\end{table*}

These results support the intended decomposition. Protocol alignment can be induced with a few thousand structured examples, while semantic adaptation is more sensitive to model stability and task dynamics. \emph{Llama-3.1-8B} remains comparatively stable under low-sample updates, whereas \emph{Mistral-7B} shows weaker protocol and semantic adaptation under the same constraints.

\subsection{Prompt Feasibility and Projection}
\label{subsec:case_stability}

We next measure the empirical feasible prompt range of each distilled adaptive model and compare it with the limits predicted by Section~\ref{sec:hierarchy}. For each model, we gradually increase prompt length until responses become invalid, empty, or non-protocol. The hardware-feasible length $|\rho|_{\mathrm{feasible}}$ is computed from Lemma~\ref{lem::cache} using the model parameters in Appendix~\ref{appendix:extended_results} and a \emph{SentencePiece} tokenizer~\cite{RS-BH-AB:16}. Table~\ref{tab:prompt_saturation_summary} also reports the calibrated attention-norm scale implied by Proposition~\ref{prop:attention_degradation_ratio}.

\begin{table}[t]
  \centering
  \scriptsize
  \captionsetup{font=scriptsize}
  \caption{Prompt feasibility and saturation across distilled student models.  $|\rho|_{\mathrm{feasible}}$ is the hardware limit from Lemma~\ref{lem::cache}, while  $|\rho|_{\mathrm{sat}}^{\mathrm{prop}}$ is the Proposition~\ref{prop:attention_degradation_ratio} upper limit, calibrated with a $10\%$ margin over the observed threshold using $m=50$, $\tau=0.1$, and $d_h=128$. Since the bound identifies only $QK$, we report an effective calibration with $Q>K$. Full sweeps are in Tables~\ref{tab:prompt_size_llama}--\ref{tab:prompt_size_mistral}.}
  \label{tab:prompt_saturation_summary}
  \begin{tabular}{l c c c c c c}
    \toprule
    \textbf{Adaptive Model} 
    & $\boldsymbol{|\rho|_{\mathrm{feasible}}}$ 
    & $\boldsymbol{Q_{\mathrm{eff}}}$ 
    & $\boldsymbol{K_{\mathrm{eff}}}$ 
    & $\boldsymbol{|\rho|_{\mathrm{sat}}^{\mathrm{prop}}}$ 
    & $\boldsymbol{|\rho|_{\mathrm{sat}}^{\mathrm{obs}}}$ 
    & $\boldsymbol{\Delta_{\rho}^{\mathrm{obs}}}$ \\
    \midrule
    Llama-3.1-8B 
    & $5.8{\times}10^4$ 
    & $3.25$ 
    & $2.71$ 
    & $2.38{\times}10^3$ 
    & $2160$ 
    & $0.028$ \\
    Mistral-7B   
    & $6.4{\times}10^4$ 
    & $2.22$ 
    & $1.85$ 
    & $1.03{\times}10^3$ 
    & $940$  
    & $0.014$ \\
    \bottomrule
  \end{tabular}
\end{table}

The observed saturation thresholds are far below the theoretical memory-based bounds: Llama-3.1-8B remains reliable only up to about $2.8\%$ of its hardware-feasible range, while Mistral-7B saturates around $1.4\%$. This supports the analysis of Section~\ref{sec:hierarchy}: prompt feasibility is not determined only by the nominal context window or KV-cache limit. Instead, the effective prompt domain can be much smaller because task-relevant information becomes weakly represented before the hardware limit is reached.

The projection mechanism therefore controls the effective prompt state. Given historical tuples $\mathcal{H}=\{(x_i,e_i,\mathcal{D}_i,u_i)\}_{i=1}^{N},$
where $x_i$ is an evaluated point, $e_i$ its error, $\mathcal{D}_i$ its interval, and $u_i$ its uncertainty, the controller partitions the active domain into intervals $\{I_k\}_{k=1}^K$ and summarizes each interval by aggregate statistics,
$\bar e_k=\mathbb{E}_{x_i\in I_k}[e_i]$ and $
\bar u_k=\mathbb{E}_{x_i\in I_k}[u_i].$ This preserves local information through representative evaluations and global information through interval summaries. The resulting hierarchy supports a two-stage decision process: the oracle selects a promising interval $I_{k^\star}=\arg\max_{I_k}\Phi(I_k),$
and the adaptive model refines the next point within $I_{k^\star}$. Thus, projection implements the control rule motivated by Proposition~\ref{prop:attention_degradation_ratio}: reduce prompt length while preserving the task-relevant information needed for the next decision.


\subsection{End-to-End Results}
\label{subsec:results}

Finally, we compare the full hierarchical architecture against oracle-only, distillation-only, and non-distilled adaptive baselines. The MFBO environment uses four fidelity models with increasing accuracy and execution costs of $1$, $2$, $3$, and $4$ minutes, respectively. The maximum experimental budget is fixed, and hierarchy settings are chosen according to the stability limits in Table~\ref{tab:prompt_saturation_summary}.

\begin{table*}[ht]
  \centering
  \captionsetup{font=scriptsize}
  \caption{Teacher--student ablation with OpenAI teachers and local students. Results are reported as mean $\pm$ 95\% confidence interval over matched MFBO seeds. Lower is better for distance to optimal, hierarchical frequency, and cost; higher is better for number of evaluated points.}
  \label{tab:teacher-student-ablation}
  \scriptsize
  \begin{tabular}{lllcccc}
    \toprule
    \textbf{Oracle $\mathcal{O}_\phi$} 
    & \textbf{Adaptive $\mathcal{U}_\theta$} 
    & \textbf{Mode} 
    & \textbf{Dist. to optimal} $\downarrow$ 
    & \textbf{Num. Points} $\uparrow$ 
    & \textbf{Hierar. Freq.} ($\%$) $\downarrow$ 
    & \textbf{Cost} (\$) $\downarrow$ \\
    \midrule

    gpt-5 & Llama-3.1-8B & Hierarchical  
    & $12.74 \pm 0.62$  
    & $121.0 \pm 2.80$ 
    & $3.471 \pm 0.21$ 
    & $2.25 \pm 0.09$ \\

          &              & Distill only  
    & $124.1 \pm 4.10$ 
    & $74.00 \pm 2.20$ 
    & -- 
    & $0.00 \pm 0.00$ \\

          &              & No distill    
    & $435.1 \pm 9.60$ 
    & $12.00 \pm 1.15$ 
    & -- 
    & $0.00 \pm 0.00$ \\

    \cmidrule(lr){1-7}

    gpt-5 & Mistral-7B   & Hierarchical  
    & $71.92 \pm 3.20$  
    & $23.00 \pm 1.50$ 
    & $11.11 \pm 0.55$ 
    & $6.12 \pm 0.24$ \\

          &              & Distill only  
    & $245.6 \pm 7.85$ 
    & $9.00 \pm 0.95$ 
    & -- 
    & $0.00 \pm 0.00$ \\

          &              & No distill    
    & $464.8 \pm 11.4$ 
    & $3.00 \pm 0.42$ 
    & -- 
    & $0.00 \pm 0.00$ \\

    \cmidrule(lr){1-7}

    gpt-5-nano & Llama-3.1-8B & Hierarchical  
    & $13.05 \pm 0.68$  
    & $132.00 \pm 3.10$ 
    & $4.86 \pm 0.28$ 
    & $0.45 \pm 0.02$ \\

    \cmidrule(lr){1-7}

    gpt-5-nano & Mistral-7B   & Hierarchical  
    & $84.86 \pm 3.95$  
    & $70.00 \pm 2.70$ 
    & $6.25 \pm 0.33$ 
    & $0.75 \pm 0.04$ \\

    \cmidrule(lr){1-7}

    gpt-5 & -- & Oracle Only 
    & $8.035 \pm 0.35$  
    & $144.00 \pm 2.10$  
    & $100.00 \pm 0.00$ 
    & $8.25 \pm 0.32$ \\

    \cmidrule(lr){1-7}

    gpt-5-nano & -- & Oracle Only 
    & $10.41 \pm 0.48$ 
    & $149.00 \pm 2.35$ 
    & $100.00 \pm 0.00$ 
    & $2.10 \pm 0.09$ \\

    \bottomrule
  \end{tabular}
\end{table*}

The hierarchical architecture consistently improves over distillation-only and non-distilled variants. The gain is especially clear for \emph{Llama-3.1-8B}, where hierarchical supervision approaches oracle-only performance while requiring oracle calls in only $3.47$--$4.86\%$ of inference steps. This substantially reduces cost while preserving decision quality. In contrast, \emph{Mistral-7B} benefits from hierarchy but remains less stable, consistent with its lower saturation threshold and weaker adaptation results. Overall, the case study supports the main claim: feasibility-aware hierarchical supervision can make compact models more reliable in sequential agentic settings without requiring the oracle at every step.
\section{Conclusions}
\label{sec:conclusions}

This paper introduced a hierarchical prompt-domain control and adaptive learning framework for compact language models in resource-constrained agentic systems. We formulated reliability as a closed-loop problem that separates offline schema learning from online semantic adaptation. The controller monitors protocol validity and semantic drift, projects accumulated histories into a feasible prompt domain, and prevents the adaptive model from leaving its stable regime. The analysis shows that context length and KV-cache feasibility alone are insufficient, since task-relevant information can be diluted before hardware limits are reached. In a controlled MFBO case study within \texttt{\texttt{URSA}}, the proposed architecture improves reliability and cost-efficiency over non-hierarchical, distillation-only, and non-distilled baselines while using far fewer oracle calls than oracle-only deployment. These results identify prompt-domain control as a promising mechanism for scalable, stable, and cost-aware agentic deployment.

\vskip 0.2in
\newpage
\bibliographystyle{plainnat}
\bibliography{main}

\newpage
\appendix
\onecolumn

\section{Extended Results}
\label{appendix:extended_results}

\subsection{Distillation Dataset Size Ablation Study}

Extended results on Distillation process. Accuracy is computed as \eqref{eqn:accuracy_1}. Due to the defined response protocol, only small models are evaluated with indifference to the control model.

\begin{table}[ht]
  \centering
  \caption{Effect of dataset size and number of epochs on the distillation process for \emph{Llama-3.1-8B}. 
  Time is \emph{h}:\emph{min} and 10:00 is the execution limit. The expected transferred data $|\mathcal{T}_T|$ is computed using Assumption~\ref{assump:dataset_size} considering $\delta_1=0.25$ and $\delta_2=0.3$.}
  \label{tab:datasize_format_gpt5_llama}
  \begin{tabular}{lcccc}
    \toprule
    \textbf{Data size} & \textbf{Epochs} & \textbf{$\boldsymbol{|\mathcal{T}_T|}$} & \textbf{Accuracy} $\uparrow$ & \textbf{Time} $\downarrow$ \\
    \midrule
    1000 & 5  & 919.16  & 0.05 & 2:24 \\
         & 10 & --      & 0.10 & 4:48 \\
    2000 & 5  & 1163.43 & 0.05 & 3:56 \\
         & 10 & --      & 0.15 & 5:30 \\
    3000 & 5  & 1335.40 & 0.10 & 5:16 \\
         & 10 & --      & 0.20 & 6:40 \\
    4000 & 5  & 1472.62 & 0.40 & 6:38 \\
         & 10 & --      & 0.60 & 8:15 \\
    5000 & 5  & 1588.70 & 1.00 & 8:35 \\
         & 10 & --      & 1.00 & 10:00 \\
    \bottomrule
  \end{tabular}
\end{table}

\begin{table}[ht]
  \centering
  \caption{Effect of dataset size and number of epochs on the distillation process for \emph{Mistral-7B}. 
  Time is \emph{h}:\emph{min} and 10:00 is the execution limit. The expected transferred data $|\mathcal{T}_T|$ is computed using Assumption~\ref{assump:dataset_size} considering $\delta_1=0.25$ and $\delta_2=0.3$.}
  \label{tab:datasize_format_gpt5_mistral}
  \begin{tabular}{lcccc}
    \toprule
    \textbf{Data size} & \textbf{Epochs} & \textbf{$\boldsymbol{|\mathcal{T}_T|}$} & \textbf{Accuracy} $\uparrow$ & \textbf{Time} $\downarrow$ \\
    \midrule
    1000 & 5  & 505.28  & 0.05 & 2:06 \\
         & 10 & --      & 0.10 & 4:12 \\
    2000 & 5  & 600.88  & 0.10 & 3:26 \\
         & 10 & --      & 0.05 & 4:49 \\
    3000 & 5  & 664.98  & 0.10 & 4:36 \\
         & 10 & --      & 0.00 & 5:50 \\
    4000 & 5  & 714.57  & 0.15 & 5:48 \\
         & 10 & --      & 0.05 & 7:31 \\
    5000 & 5  & 755.57  & 0.20 &  6:45 \\
         & 10 & --      & 0.25 & 8:52 \\
    \bottomrule
  \end{tabular}
\end{table}
\FloatBarrier

\subsection{Prompt Size Ablation Study}

Extended results on ablation study for prompt size where $|\rho|$ is computed
such as \eqref{eq:prompt-upper-bound}. We also report the saturation quantities
from Proposition~\ref{prop:attention_degradation_ratio}. Only small models are
evaluated.

\begin{table}[ht]
  \centering
  \scriptsize
  \caption{Adaptive Model Parameters. 
  $|\rho|_{\mathrm{feasible}}$ is the hardware-feasible prompt length, while
  $|\rho|_{\mathrm{sat}}^{\mathrm{prop}}$ is the saturation upper limit from
  Proposition~\ref{prop:attention_degradation_ratio}, calibrated with a $10\%$
  margin over the observed threshold using $m=50$, $\tau=0.1$, and $d_h=128$.
  Since the bound identifies only $QK$, we report an effective asymmetric
  calibration with $Q_{\mathrm{eff}}>K_{\mathrm{eff}}$.}
  \label{tab:models_parameters}
  \resizebox{\linewidth}{!}{
  \begin{tabular}{l c c c c c c c c c c c c c c}
    \toprule
    \textbf{Model} 
    & $\boldsymbol{L}$ 
    & $\boldsymbol{n_{\text{kv}}}$ 
    & $\boldsymbol{d_\text{h}}$ 
    & $\boldsymbol{b_\text{kv}}$ 
    & Precision 
    & $\boldsymbol{M_\text{max}}$ (GB) 
    & $\boldsymbol{M_\text{w}}$ (GB) 
    & $\boldsymbol{M_\text{misc}}$ (GB) 
    & $\boldsymbol{|\rho|_\text{feasible}}$
    & $\boldsymbol{Q_{\mathrm{eff}}}$
    & $\boldsymbol{K_{\mathrm{eff}}}$
    & $\boldsymbol{|\rho|_{\mathrm{sat}}^{\mathrm{prop}}}$
    & $\boldsymbol{|\rho|_{\mathrm{sat}}^{\mathrm{obs}}}$
    & $\boldsymbol{\Delta_{\rho}^{\mathrm{obs}}}$\\
    \midrule
    Llama-3.1-8B 
    & $32$ & $8$ & $128$ & $2$ & fp16 
    & $24$ & $16$ & $2$ 
    & $4.9{\cdot}10^4$ 
    & $3.25$ & $2.71$ 
    & $2.38{\cdot}10^3$ 
    & $2160$ 
    & $0.044$ \\
    Mistral-7B  
    & $32$ & $8$ & $128$ & $2$ & fp16 
    & $24$ & $14.5$ & $2$ 
    & $6.1{\cdot}10^4$ 
    & $2.22$ & $1.85$ 
    & $1.03{\cdot}10^3$ 
    & $940$ 
    & $0.015$ \\
    \bottomrule
  \end{tabular}}
\end{table}

\begin{table}[ht]
  \centering
  \caption{Response of distilled \emph{Llama-3.1-8B\_5000\_5} as a function of prompt length. 
  The fourth column shows the ratio of the prompt size to the theoretical feasible bound 
  ($|\rho|_{\mathrm{feasible}}\!\approx\!49{,}152$ tokens), while the fifth column reports the ratio to the
  Proposition~\ref{prop:attention_degradation_ratio} saturation estimate
  ($|\rho|_{\mathrm{sat}}^{\mathrm{prop}}\!\approx\!2.38{\cdot}10^3$ tokens).
  ``--'' indicates an empty response.}
  \label{tab:prompt_size_llama}
  \begin{tabular}{l l c c c}
    \toprule
    \textbf{Index} 
    & \textbf{Result} 
    & \textbf{$\boldsymbol{|\rho|}$}  
    & $\boldsymbol{|\rho| / |\rho|_{\mathrm{feasible}}}$ 
    & $\boldsymbol{|\rho| / |\rho|_{\mathrm{sat}}^{\mathrm{prop}}}$ \\
    \midrule
     1 & [3, 9.6323] & 2003 & 0.041 & 0.842 \\
     2 & [2, 3.6239] & 2053 & 0.042 & 0.863 \\
     3 & [3, 7.0572] & 2107 & 0.043 & 0.885 \\
     4 & [3, 4.8469] & 2160 & 0.044 & 0.908 \\
     5 & --          & 2211 & 0.045 & 0.929 \\
     6 & [2, 2.3688] & 2261 & 0.046 & 0.950 \\
     7 & [1, 7.4646] & 2312 & 0.047 & 0.971 \\
     8 & --          & 2365 & 0.048 & 0.994 \\
     9 & --          & 2416 & 0.049 & 1.015 \\
    10 & --          & 2467 & 0.050 & 1.037 \\
    \bottomrule
  \end{tabular}
\end{table}

\begin{table}[ht]
  \centering
  \caption{Response of distilled \emph{Mistral-7B\_5000\_10} as a function of prompt length. 
  The fourth column shows the ratio of the prompt size to the theoretical feasible bound 
  ($|\rho|_{\mathrm{feasible}}\!\approx\!61{,}440$ tokens), while the fifth column reports the ratio to the
  Proposition~\ref{prop:attention_degradation_ratio} saturation estimate
  ($|\rho|_{\mathrm{sat}}^{\mathrm{prop}}\!\approx\!1.03{\cdot}10^3$ tokens).
  ``--'' indicates an empty response and ``x'' stands for non-protocol response.}
  \label{tab:prompt_size_mistral}
  \begin{tabular}{l l c c c}
    \toprule
    \textbf{Index} 
    & \textbf{Result} 
    & \textbf{$\boldsymbol{|\rho|}$} 
    & $\boldsymbol{|\rho| / |\rho|_{\mathrm{feasible}}}$ 
    & $\boldsymbol{|\rho| / |\rho|_{\mathrm{sat}}^{\mathrm{prop}}}$ \\
    \midrule
     1 & [1, 42.1200] & 867  & 0.014 & 0.842 \\
     2 & [3, 35.1600] & 940  & 0.015 & 0.913 \\
     3 & x            & 987  & 0.016 & 0.958 \\
     4 & x            & 1040 & 0.017 & 1.010 \\
     5 & x            & 1093 & 0.018 & 1.061 \\
     6 & x            & 1147 & 0.019 & 1.114 \\
     7 & x            & 1216 & 0.020 & 1.181 \\
     8 & x            & 1267 & 0.021 & 1.230 \\
     9 & x            & 1320 & 0.021 & 1.282 \\
    10 & x            & 1344 & 0.022 & 1.305 \\
    \bottomrule
  \end{tabular}
\end{table}
\FloatBarrier

\subsection{Adaptive Fine-Tuning Dataset Size Ablation Study}

Extended results on the adaptive fine-tuning procedure where \textbf{Acc (Model)} is computed as \eqref{eqn:accuracy_2} and \textbf{Error (Point)} is computed as \eqref{eqn:accuracy_1}. And \textbf{Train Loss} and \textbf{Eval Loss} are computed as \eqref{eqn:main_optimiziation_problem}.

\begin{table}[ht]
  \centering
  \caption{Effect of dataset size and number of epochs on decision-criteria adaptation for \emph{GPT-5} and \emph{Llama-3.1-8B\_5000\_5}. 
  Accuracy (Model) reflects correct fidelity selection; Error (Point) reports mean $\ell_1$ distance (lower is better). 
  Results are reported as mean $\pm$ 95\% confidence interval over matched MFBO seeds. Time is given as \emph{minutes}:\emph{seconds}.}
  \label{tab:datasize_criteria_llama_gpt5}
  \scriptsize
  \begin{tabular}{lcccccc}
    \toprule
    \textbf{Data size} & \textbf{Epochs} & \textbf{Acc (Model)} $\uparrow$ & \textbf{Error (Point)} $\downarrow$
                       & \textbf{Train Loss} $\downarrow$ & \textbf{Eval Loss} $\downarrow$ & \textbf{Time} $\downarrow$ \\
    \midrule
    10  & 5  & $0.25 \pm 0.03$ & $40.96 \pm 1.85$ & $0.7722 \pm 0.0180$ & $0.5284 \pm 0.0105$ &  2:33 \\
        & 10 & $0.35 \pm 0.04$ & $32.44 \pm 1.62$ & $0.5853 \pm 0.0142$ & $0.4717 \pm 0.0094$ &  3:10 \\
        & 50 & $0.52 \pm 0.04$ & $24.12 \pm 1.08$ & $0.1856 \pm 0.0085$ & $0.4236 \pm 0.0081$ &  7:23 \\
    \midrule
    50  & 5  & $0.42 \pm 0.04$ & $22.81 \pm 1.02$ & $0.6231 \pm 0.0150$ & $0.4945 \pm 0.0097$ &  5:51 \\
        & 10 & $0.58 \pm 0.04$ & $18.65 \pm 0.88$ & $0.4456 \pm 0.0117$ & $0.4729 \pm 0.0091$ & 10:37 \\
        & 50 & $0.81 \pm 0.03$ & $12.74 \pm 0.62$ & $0.1240 \pm 0.0062$ & $0.3981 \pm 0.0078$ & 45:06 \\
    \midrule
    100 & 5  & $0.61 \pm 0.04$ & $17.32 \pm 0.81$ & $0.5927 \pm 0.0140$ & $0.4626 \pm 0.0088$ & 10:36 \\
        & 10 & $0.79 \pm 0.03$ & $11.48 \pm 0.56$ & $0.4038 \pm 0.0106$ & $0.3891 \pm 0.0074$ & 17:14 \\
        & 50 & $0.93 \pm 0.02$ & $8.06 \pm 0.42$ & $0.1825 \pm 0.0079$ & $0.3624 \pm 0.0069$ & 82:20 \\
    \bottomrule
  \end{tabular}
\end{table}

\begin{table}[ht]
  \centering
  \caption{Effect of dataset size and number of epochs on decision-criteria adaptation for \emph{GPT-5} and \emph{Mistral-7B\_5000\_10}. 
  Accuracy (Model) reflects correct fidelity selection; Error (Point) reports mean $\ell_1$ distance (lower is better). 
  Results are reported as mean $\pm$ 95\% confidence interval over matched MFBO seeds. Time is given as \emph{minutes}:\emph{seconds}. ``--'' stands for non-feasible responses.}
  \label{tab:datasize_criteria_mistral_gpt5}
  \scriptsize
  \begin{tabular}{lcccccc}
    \toprule
    \textbf{Data size} & \textbf{Epochs} & \textbf{Acc (Model)} $\uparrow$ & \textbf{Error (Point)} $\downarrow$
                       & \textbf{Train Loss} $\downarrow$ & \textbf{Eval Loss} $\downarrow$ & \textbf{Time} $\downarrow$ \\
    \midrule
    10  & 5  & $0.06 \pm 0.02$ & $64.82 \pm 2.90$ & $1.2478 \pm 0.0285$ & $1.1926 \pm 0.0240$ &  2:41 \\
        & 10 & $0.08 \pm 0.02$ & $61.37 \pm 2.74$ & $1.1834 \pm 0.0260$ & $1.1561 \pm 0.0225$ &  3:16 \\
        & 50 & --   &  --   &   --   &   --   &   --  \\
    \midrule
    50  & 5  & --   &  --   &   --   &   --  &   --  \\
        & 10 & --   &  --   &   --   &   --  &   --  \\
        & 50 & $0.13 \pm 0.03$ & $45.63 \pm 2.18$ & $0.8819 \pm 0.0205$ & $0.9648 \pm 0.0196$ & 45:52 \\
    \midrule
    100 & 5  & $0.09 \pm 0.02$ & $52.71 \pm 2.44$ & $1.0124 \pm 0.0232$ & $0.9986 \pm 0.0207$ & 10:52 \\
        & 10 & --   &  --   &   --   &   --   &   --  \\
        & 50 & $0.14 \pm 0.03$ & $41.92 \pm 2.04$ & $0.8427 \pm 0.0191$ & $0.9311 \pm 0.0188$ & 82:48 \\
    \bottomrule
  \end{tabular}
\end{table}

\begin{table}[ht]
  \centering
  \caption{Effect of dataset size and number of epochs on decision-criteria adaptation for \emph{GPT-5-nano} and \emph{Llama-3.1-8B\_5000\_5}. 
  Accuracy (Model) reflects correct fidelity selection; Error (Point) reports mean $\ell_1$ distance (lower is better). 
  Results are reported as mean $\pm$ 95\% confidence interval over matched MFBO seeds. Time is given as \emph{minutes}:\emph{seconds}.}
  \label{tab:datasize_criteria_llama_gpt5nano}
  \scriptsize
  \begin{tabular}{lcccccc}
    \toprule
    \textbf{Data size} & \textbf{Epochs} & \textbf{Acc (Model)} $\uparrow$ & \textbf{Error (Point)} $\downarrow$
                       & \textbf{Train Loss} $\downarrow$ & \textbf{Eval Loss} $\downarrow$ & \textbf{Time} $\downarrow$ \\
    \midrule
    10  & 5  & $0.24 \pm 0.03$ & $41.10 \pm 1.88$ & $0.7750 \pm 0.0182$ & $0.5320 \pm 0.0108$ & 2:31 \\
        & 10 & $0.34 \pm 0.04$ & $33.43 \pm 1.61$ & $0.5900 \pm 0.0143$ & $0.4750 \pm 0.0095$ & 3:08 \\
        & 50 & $0.51 \pm 0.04$ & $24.51 \pm 1.12$ & $0.1900 \pm 0.0088$ & $0.4270 \pm 0.0083$ & 7:18 \\
    \midrule
    50  & 5  & $0.41 \pm 0.04$ & $23.11 \pm 1.05$ & $0.6280 \pm 0.0152$ & $0.4980 \pm 0.0099$ & 5:48 \\
        & 10 & $0.57 \pm 0.04$ & $18.92 \pm 0.91$ & $0.4500 \pm 0.0119$ & $0.4760 \pm 0.0092$ & 10:33 \\
        & 50 & $0.80 \pm 0.03$ & $13.05 \pm 0.68$ & $0.1280 \pm 0.0064$ & $0.4020 \pm 0.0080$ & 44:50 \\
    \midrule
    100 & 5  & $0.60 \pm 0.04$ & $17.60 \pm 0.84$ & $0.5970 \pm 0.0141$ & $0.4660 \pm 0.0090$ & 10:28 \\
        & 10 & $0.78 \pm 0.03$ & $11.71 \pm 0.58$ & $0.4080 \pm 0.0108$ & $0.3920 \pm 0.0075$ & 17:02 \\
        & 50 & $0.92 \pm 0.02$ & $8.26 \pm 0.44$ & $0.1860 \pm 0.0081$ & $0.3660 \pm 0.0070$ & 81:55 \\
    \bottomrule
  \end{tabular}
\end{table}

\begin{table}[ht]
  \centering
  \caption{Effect of dataset size and number of epochs on decision-criteria adaptation for \emph{GPT-5-nano} and \emph{Mistral-7B\_5000\_10}. 
  Accuracy (Model) reflects correct fidelity selection; Error (Point) reports mean $\ell_1$ distance (lower is better). 
  Results are reported as mean $\pm$ 95\% confidence interval over matched MFBO seeds. Time is given as \emph{minutes}:\emph{seconds}. ``--'' stands for non-feasible responses.}
  \label{tab:datasize_criteria_mistral_gpt5nano}
  \scriptsize
  \begin{tabular}{lcccccc}
    \toprule
    \textbf{Data size} & \textbf{Epochs} & \textbf{Acc (Model)} $\uparrow$ & \textbf{Error (Point)} $\downarrow$
                       & \textbf{Train Loss} $\downarrow$ & \textbf{Eval Loss} $\downarrow$ & \textbf{Time} $\downarrow$ \\
    \midrule
    10  & 5  & $0.04 \pm 0.01$ & $71.36 \pm 3.12$ & $1.3142 \pm 0.0302$ & $1.2879 \pm 0.0278$ &  2:18 \\
        & 10 & $0.06 \pm 0.02$ & $67.84 \pm 2.96$ & $1.2561 \pm 0.0284$ & $1.2314 \pm 0.0260$ &  2:55 \\
        & 50 & --   &  --   &   --   &   --   &   --  \\
    \midrule
    50  & 5  & $0.07 \pm 0.02$ & $58.92 \pm 2.70$ & $1.1086 \pm 0.0251$ & $1.1437 \pm 0.0243$ &  5:44 \\
        & 10 & --   &  --   &   --   &   --   &   --  \\
        & 50 & $0.11 \pm 0.03$ & $48.37 \pm 2.26$ & $0.9624 \pm 0.0220$ & $1.0312 \pm 0.0212$ & 46:10 \\
    \midrule
    100 & 5  & $0.08 \pm 0.02$ & $55.18 \pm 2.52$ & $1.0347 \pm 0.0238$ & $1.0679 \pm 0.0225$ & 10:21 \\
        & 10 & --   &  --   &   --   &   --   &   --  \\
        & 50 & $0.13 \pm 0.03$ & $44.86 \pm 2.10$ & $0.8893 \pm 0.0202$ & $0.9827 \pm 0.0194$ & 83:05 \\
    \bottomrule
  \end{tabular}
\end{table}

\FloatBarrier

\subsection{Hierarchical Learning Results}

Graphical comparison of the proposed architecture over MFBM problem against non-hierarchical architecture. 

\begin{figure}[ht]
    \centering

    \begin{subfigure}{0.48\linewidth}
        \centering
        \includegraphics[width=0.92\linewidth,
        height=3.2cm,
        keepaspectratio]{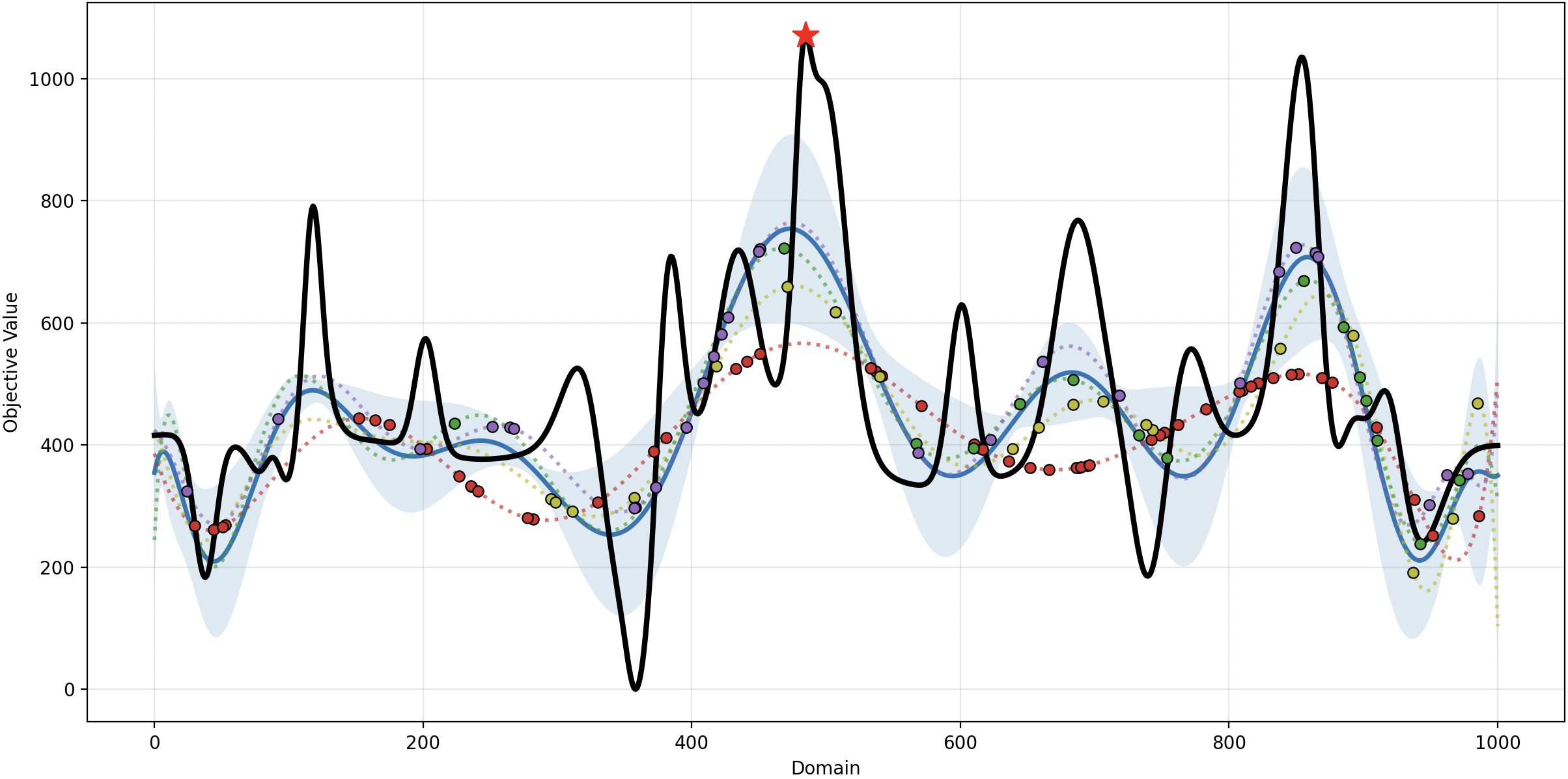}
        \caption{\emph{GPT-5} with \emph{Llama-3.1-8B}}
        \label{fig:case1}
    \end{subfigure}
    \begin{subfigure}{0.51\linewidth}
        \centering
        \includegraphics[width=\linewidth]{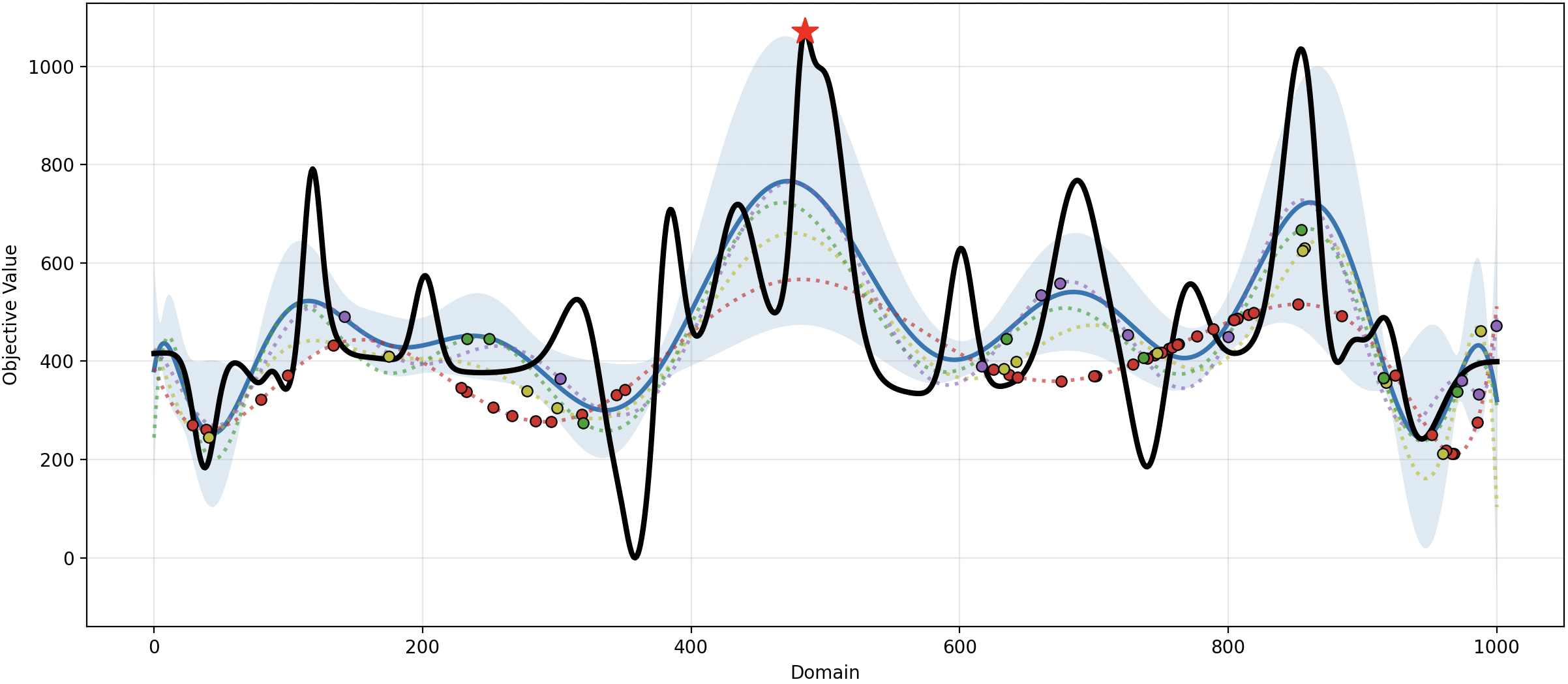}
        \caption{Distilled \emph{Llama-3.1-8B}}
        \label{fig:case2}
    \end{subfigure}

    \vspace{0.8em}

    \begin{subfigure}{0.50\linewidth}
        \centering
        \includegraphics[width=0.92\linewidth,
        height=3.2cm,
        keepaspectratio]{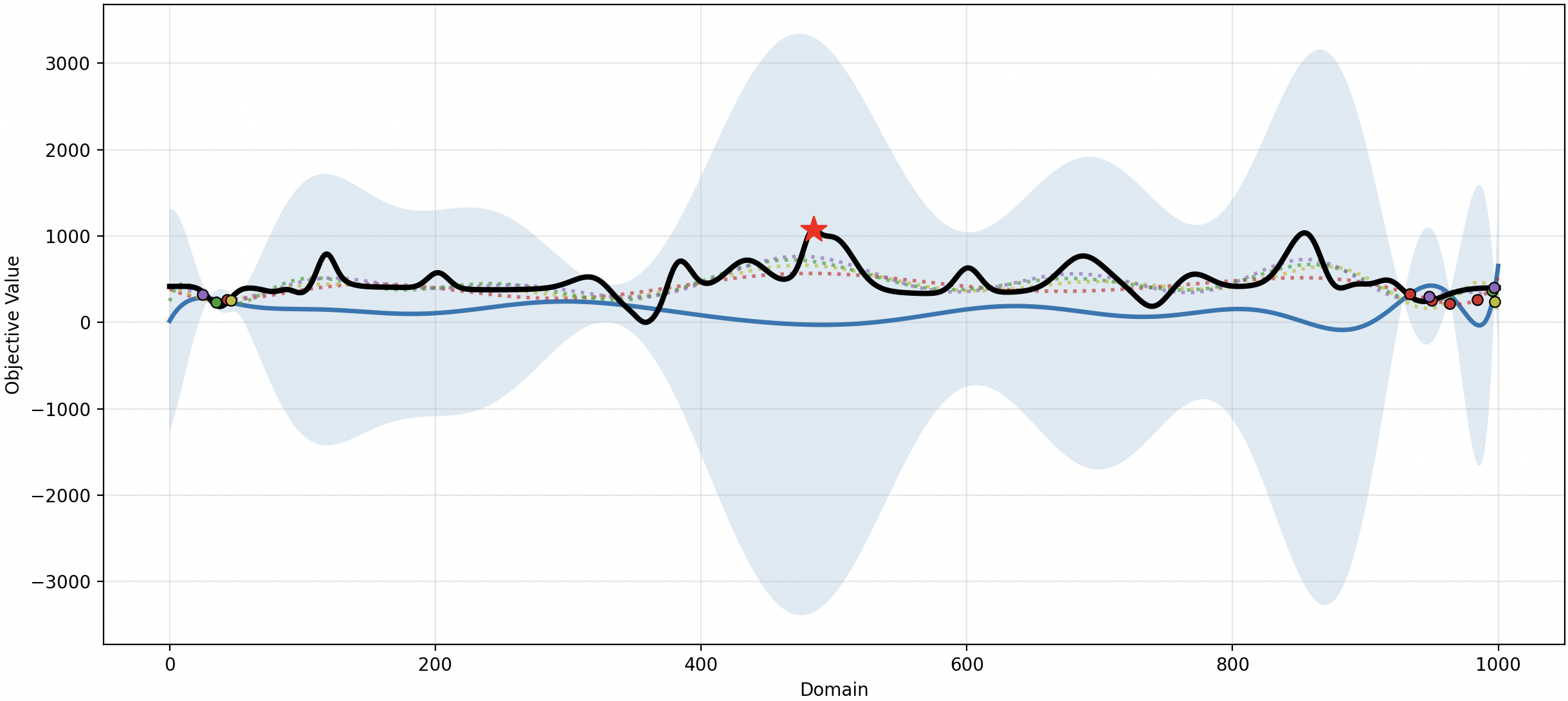}
        \caption{\emph{Llama-3.1-8B}}
        \label{fig:case3}
    \end{subfigure}
    \begin{subfigure}{0.48\linewidth}
        \centering
        \includegraphics[width=\linewidth]{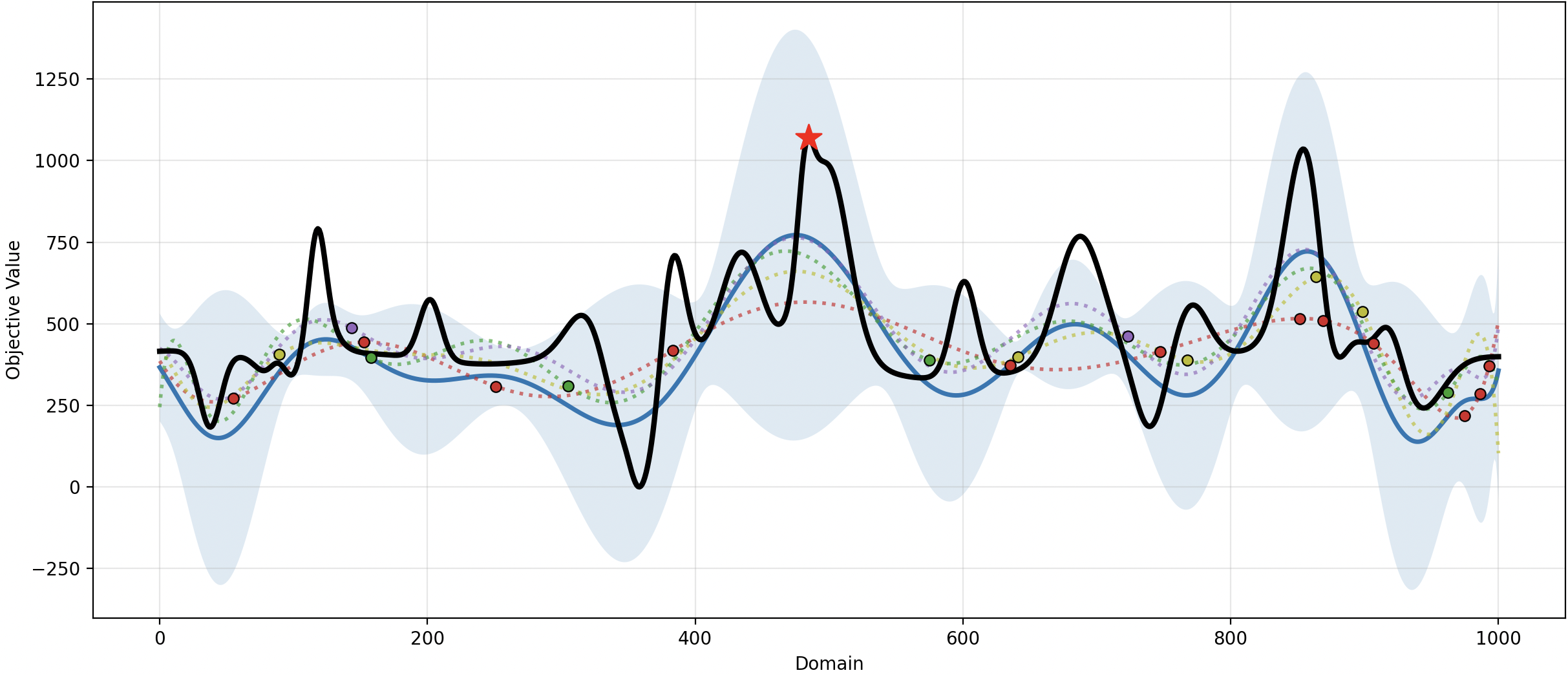}
        \caption{\emph{GPT-5} with \emph{Mistral-7B}}
        \label{fig:case4}
    \end{subfigure}

    \vspace{0.8em}

    \begin{subfigure}{0.48\linewidth}
        \centering
        \includegraphics[width=\linewidth]{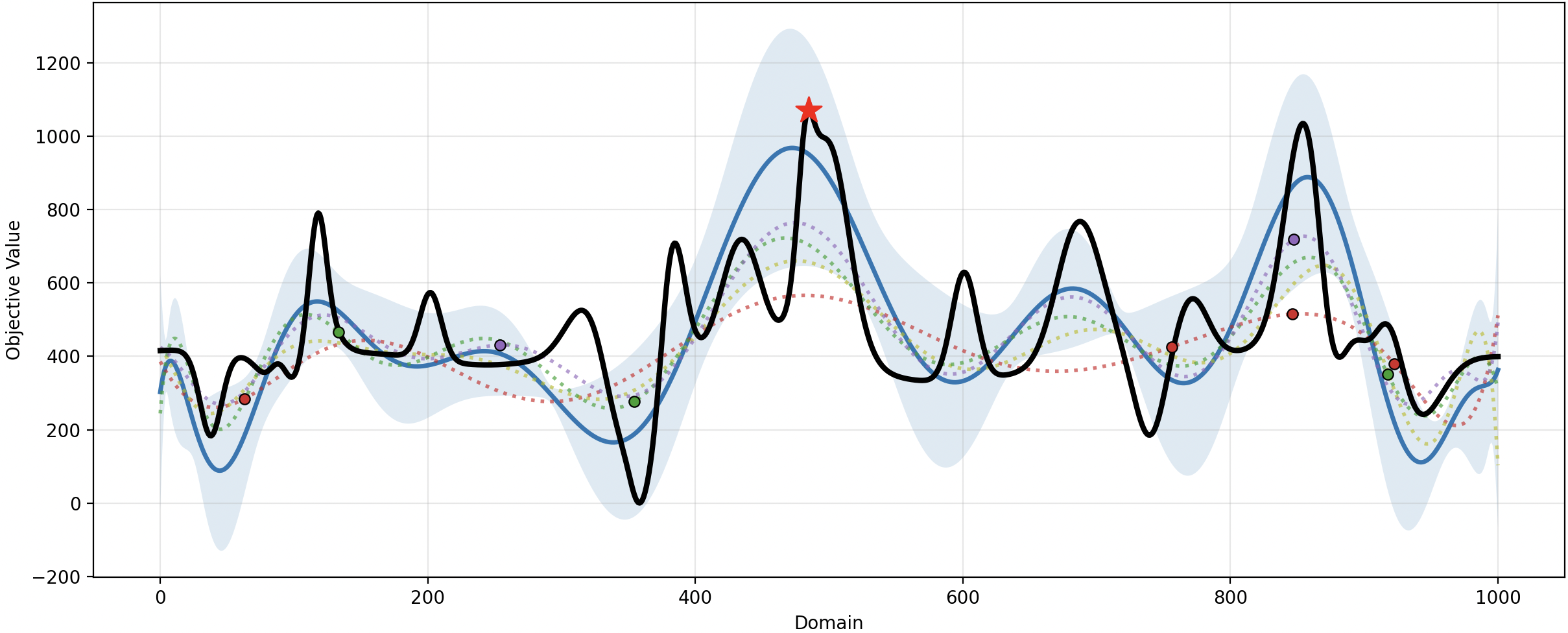}
        \caption{Distilled \emph{Mistral-7B}}
        \label{fig:case5}
    \end{subfigure}

    \caption{Visualization of the results obtained by each considered model in the ablation study represented in Table~\ref{tab:teacher-student-ablation}. The ground-truth model is shown in black, and the approximation models are shown as dashed colored lines: red, yellow, green, and purple correspond to fidelity levels $1,2,3,4$, respectively. In blue, the obtained approximation using each model is plotted with its uncertainty area.}
    \label{fig:comparison_a}
\end{figure}

\FloatBarrier

\begin{figure}[ht]
    \centering

    \begin{subfigure}{0.50\linewidth}
        \centering
        \includegraphics[width=\linewidth]{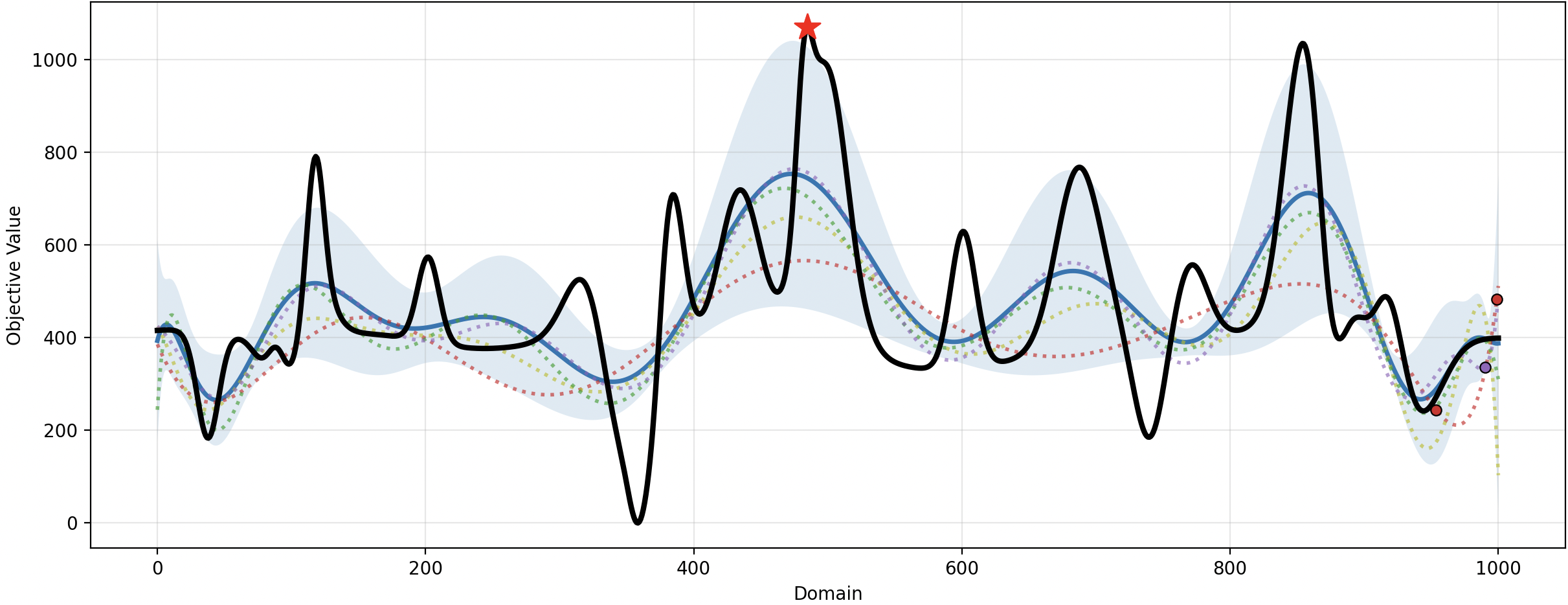}
        \caption{\emph{Mistral-7B}}
        \label{fig:case6}
    \end{subfigure}
    \begin{subfigure}{0.47\linewidth}
        \centering
        \includegraphics[width=\linewidth]{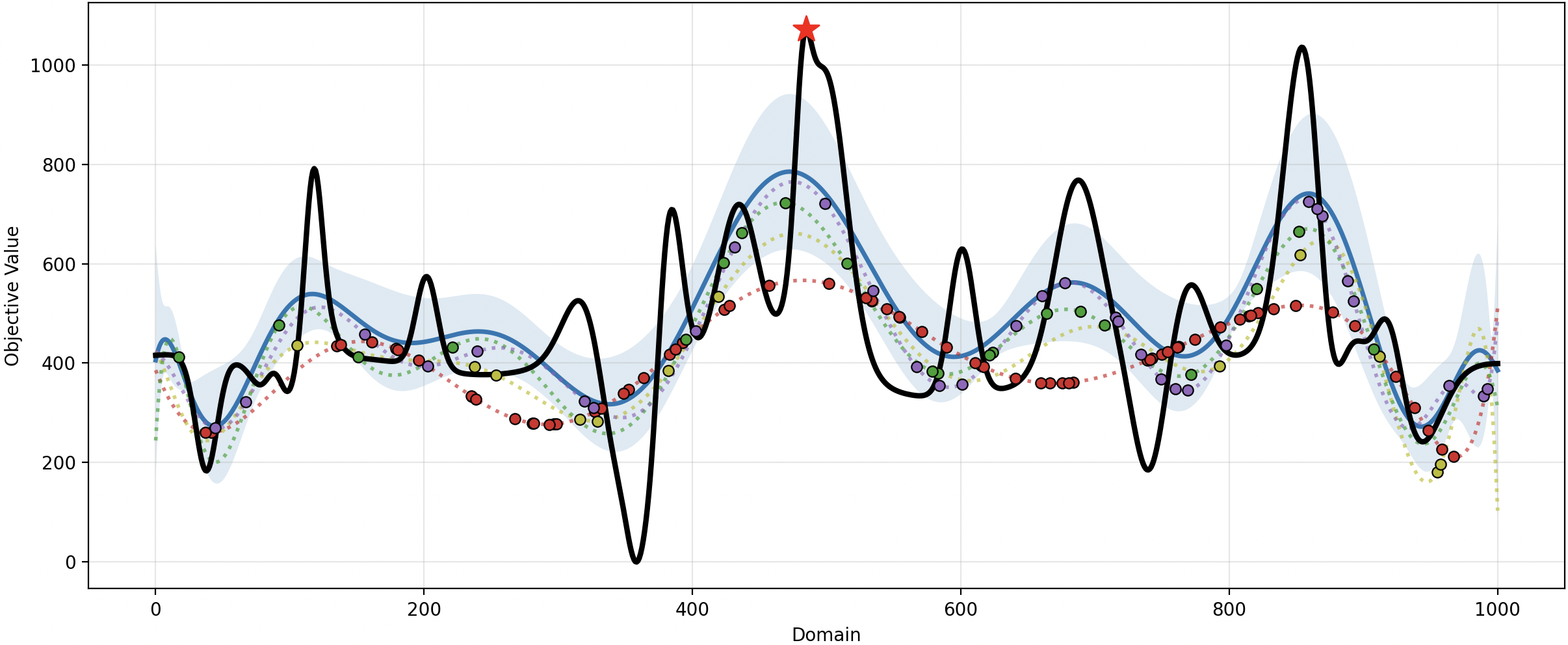}
        \caption{\emph{GPT-5-nano} with \emph{Llama-3.1-8B}}
        \label{fig:case7}
    \end{subfigure}

    \vspace{0.8em}

    \begin{subfigure}{0.46\linewidth}
        \centering
        \includegraphics[width=\linewidth]{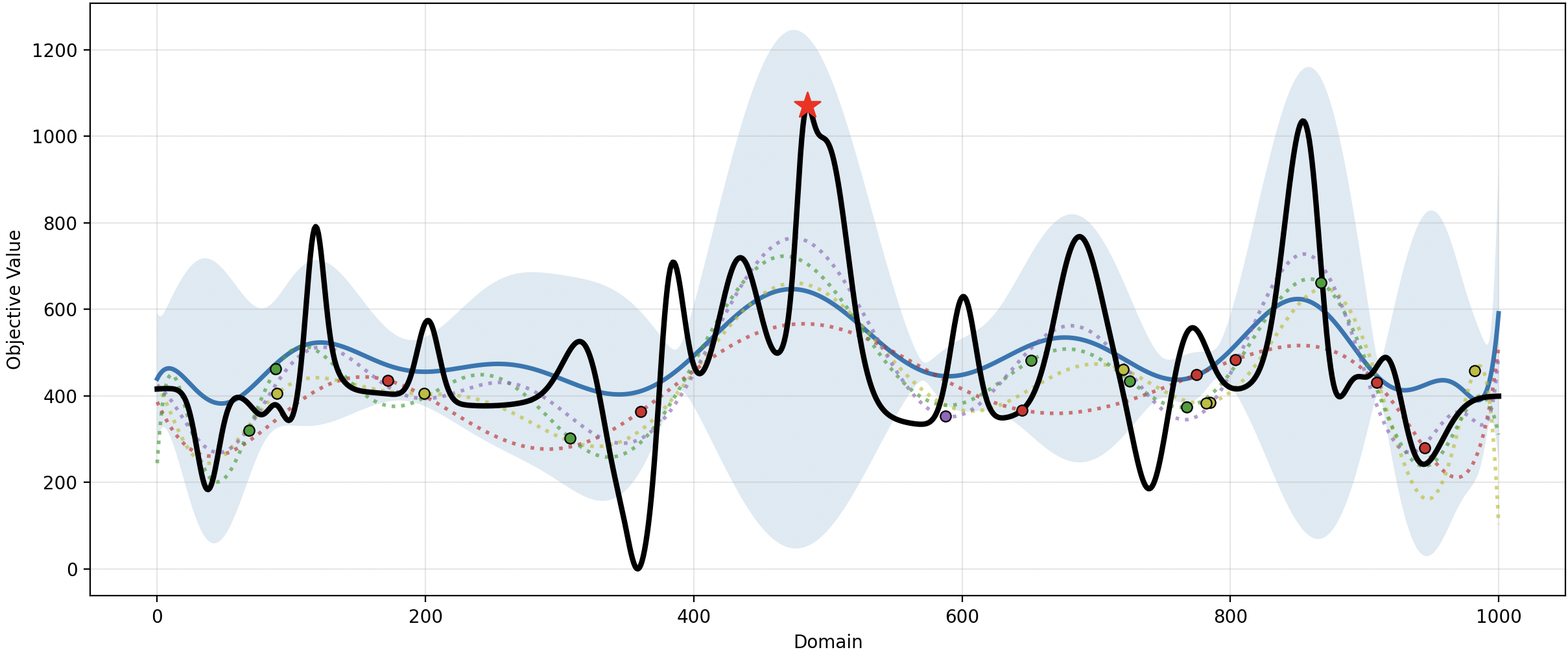}
        \caption{\emph{GPT-5-nano} with \emph{Mistral-7B}}
        \label{fig:case10}
    \end{subfigure}
    \begin{subfigure}{0.50\linewidth}
        \centering
        \includegraphics[width=\linewidth]{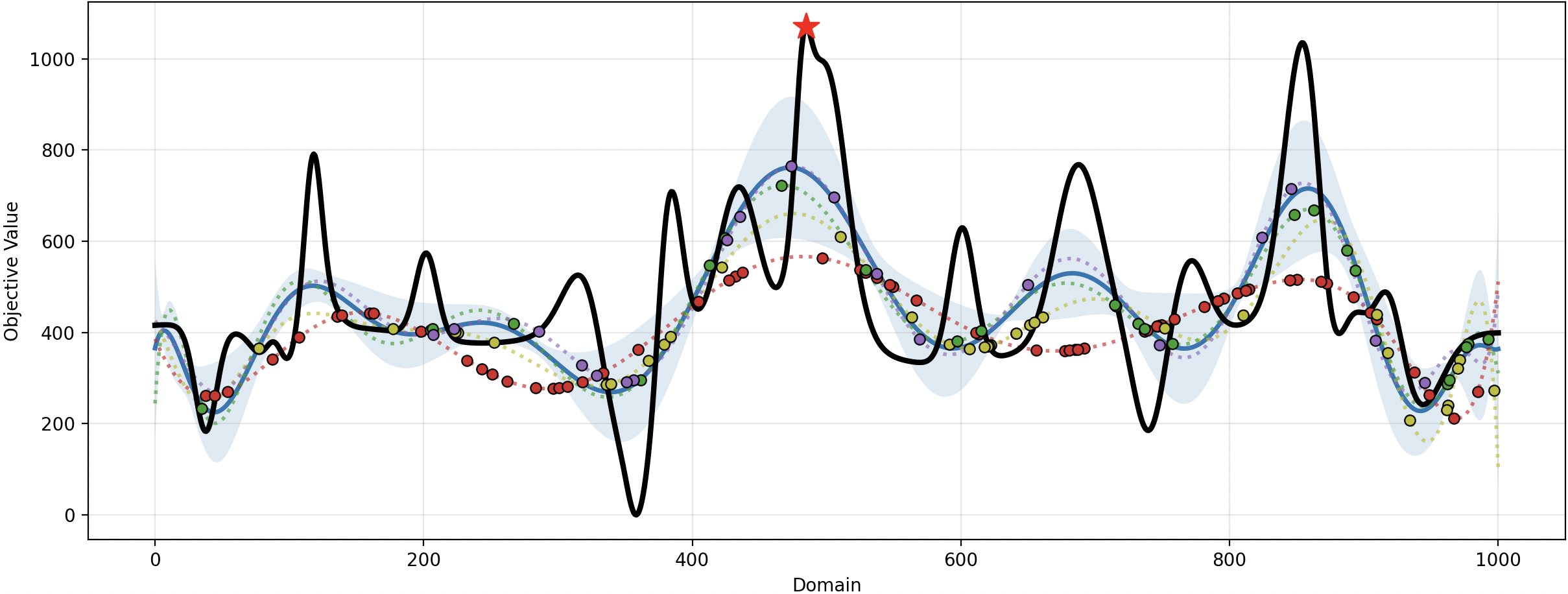}
        \caption{\emph{GPT-5}}
        \label{fig:case13}
    \end{subfigure}

    \vspace{0.8em}

    \begin{subfigure}{0.48\linewidth}
        \centering
        \includegraphics[width=\linewidth]{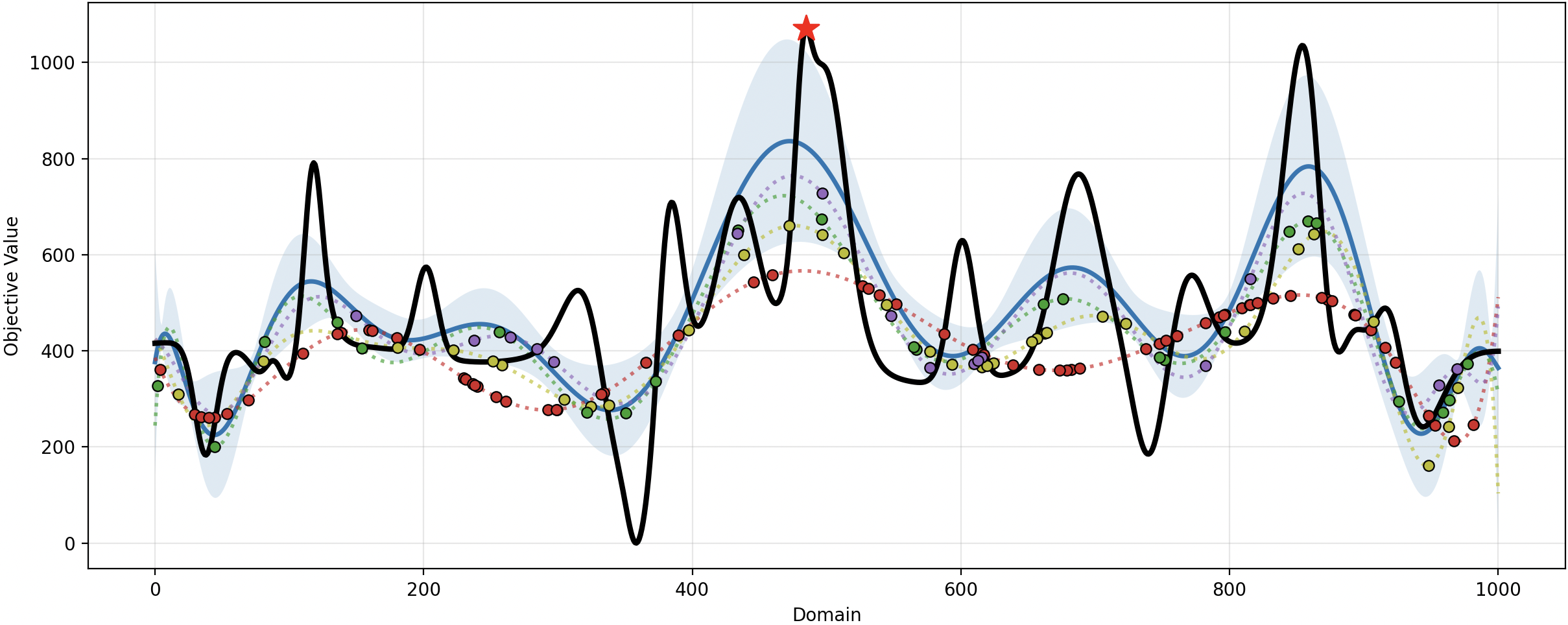}
        \caption{\emph{GPT-5-nano}}
        \label{fig:case14}
    \end{subfigure}

    \caption{Visualization of the results obtained by each considered model in the ablation study represented in Table~\ref{tab:teacher-student-ablation}. The ground-truth model is shown in black, and the approximation models are shown as dashed colored lines: red, yellow, green, and purple correspond to fidelity levels $1,2,3,4$, respectively. In blue, the obtained approximation using each model is plotted with its uncertainty area (Part 2).}
    \label{fig:comparison_b}
\end{figure}
\FloatBarrier

\newpage
\subsection{Loss curves}

Training and Validation Loss curves for each of the trained combinations on the hierarchical structure. The loss is computed such as \eqref{eqn:main_optimiziation_problem}.

\begin{figure}[t]
\centering

\begin{minipage}{0.48\textwidth}
\centering
\includegraphics[width=\linewidth]{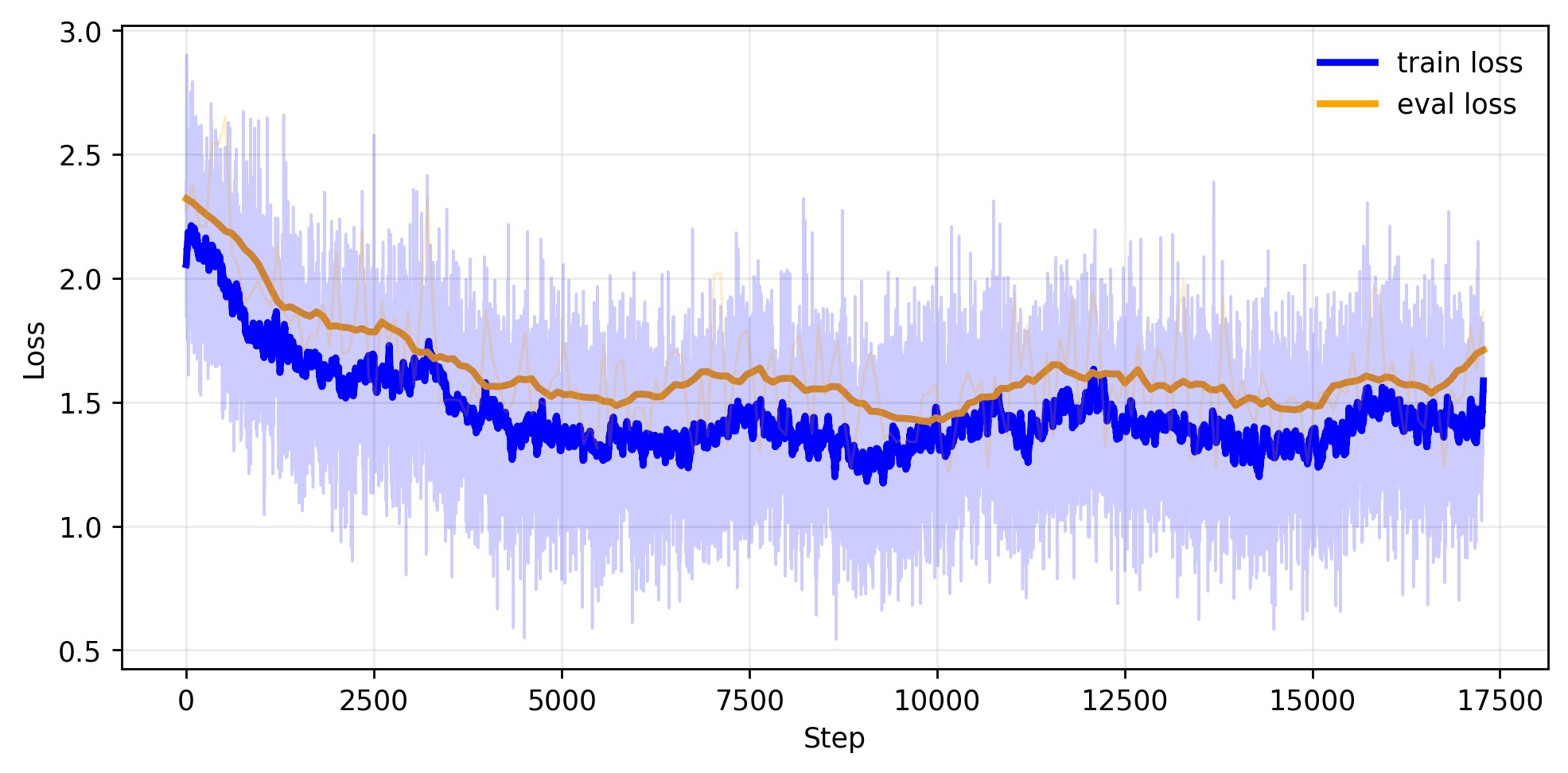}\\
{\small (a) Data=1000, Epochs=5}
\end{minipage}\hfill
\begin{minipage}{0.48\textwidth}
\centering
\includegraphics[width=\linewidth]{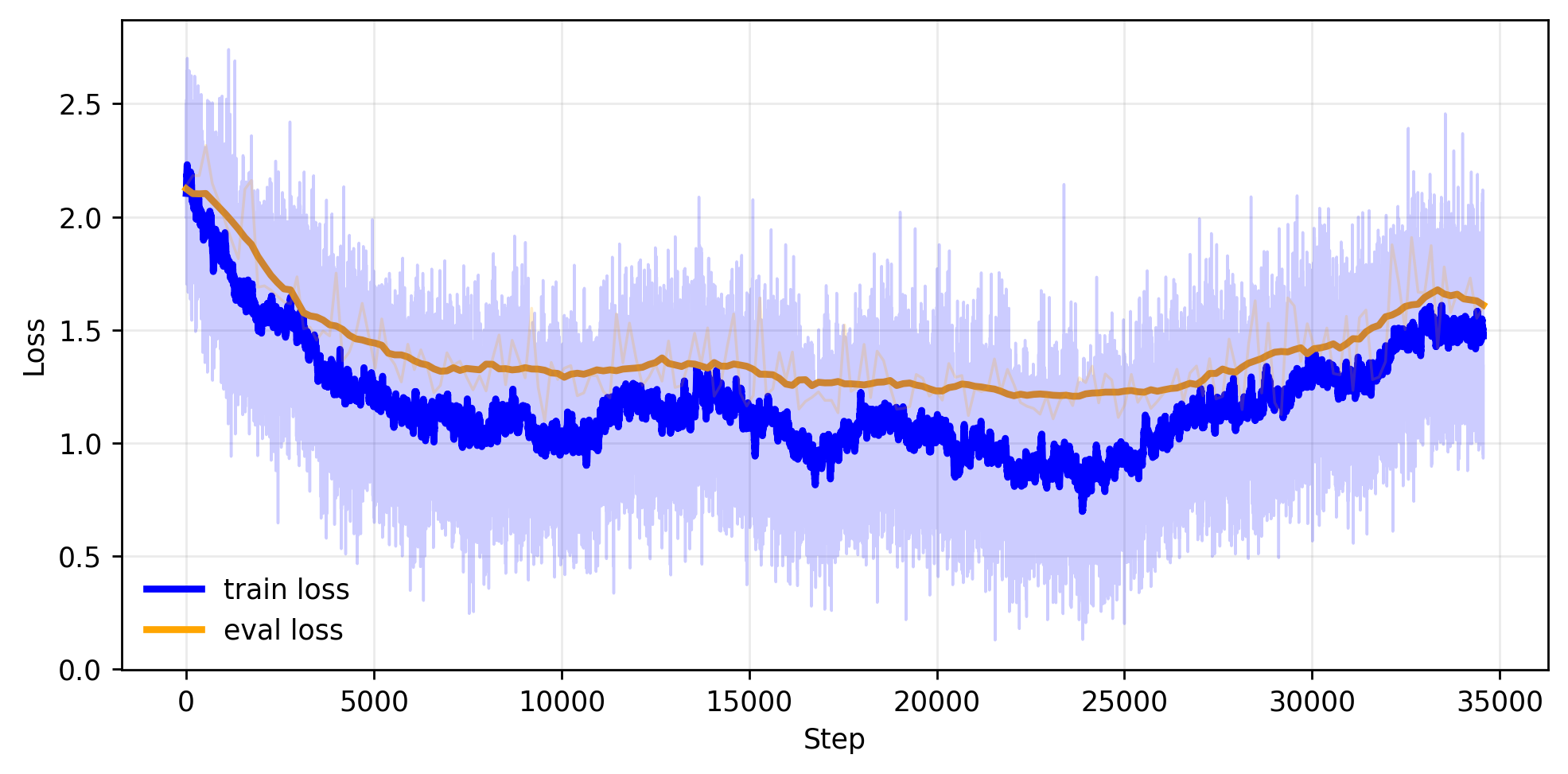}\\
{\small (b) Data=1000, Epochs=10}
\end{minipage}

\vspace{2mm}

\begin{minipage}{0.48\textwidth}
\centering
\includegraphics[width=\linewidth]{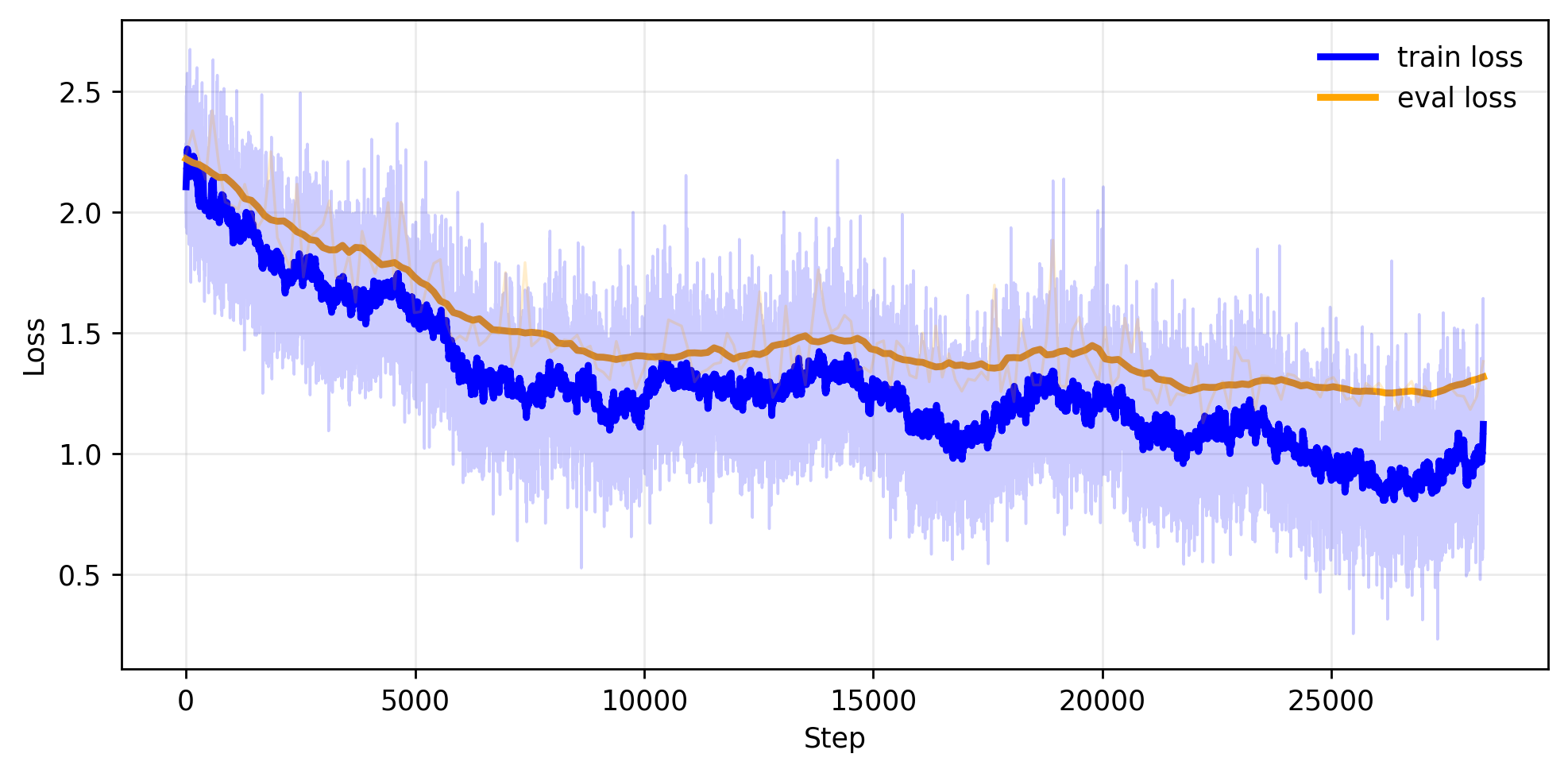}\\
{\small (c) Data=2000, Epochs=5}
\end{minipage}\hfill
\begin{minipage}{0.48\textwidth}
\centering
\includegraphics[width=\linewidth]{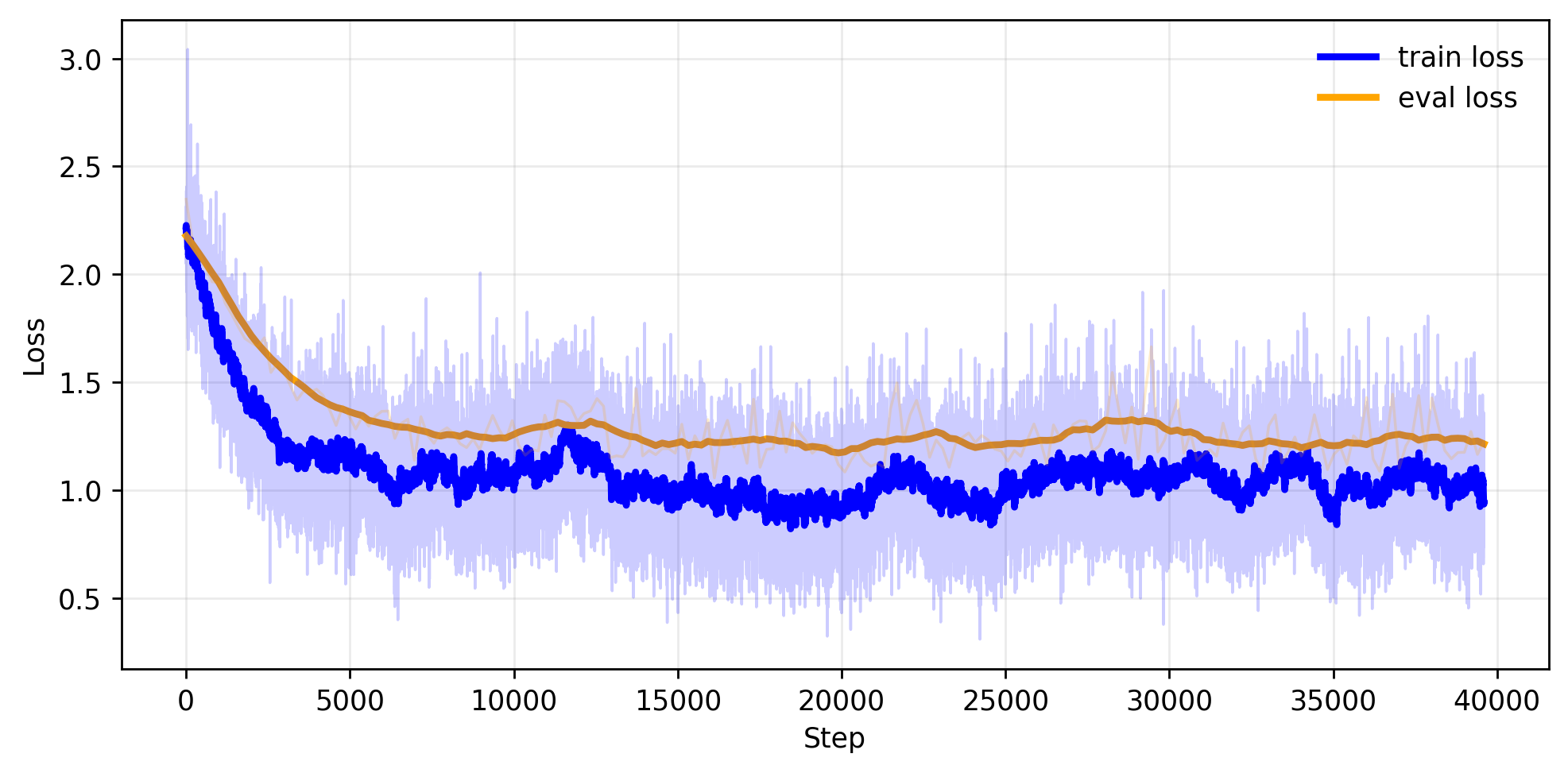}\\
{\small (d) Data=2000, Epochs=10}
\end{minipage}

\caption{Training and evaluation loss curves for Table~\ref{tab:datasize_format_gpt5_llama}, (Part 1).}
\label{fig:loss_grid_5x2}
\end{figure}
\FloatBarrier

\begin{figure}[t]
\centering

\begin{minipage}{0.48\textwidth}
\centering
\includegraphics[width=\linewidth]{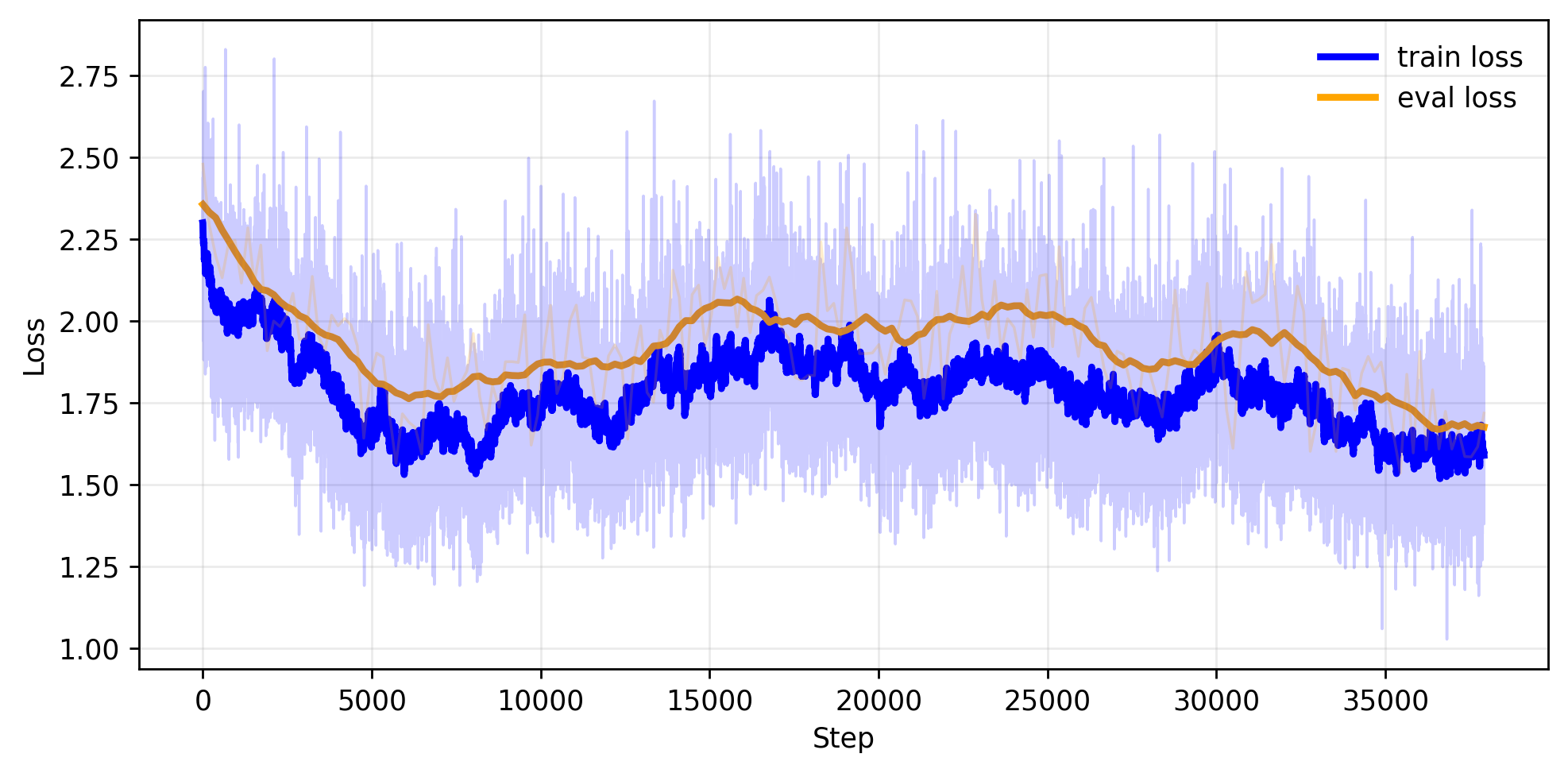}\\
{\small (e) Data=3000, Epochs=5}
\end{minipage}\hfill
\begin{minipage}{0.48\textwidth}
\centering
\includegraphics[width=\linewidth]{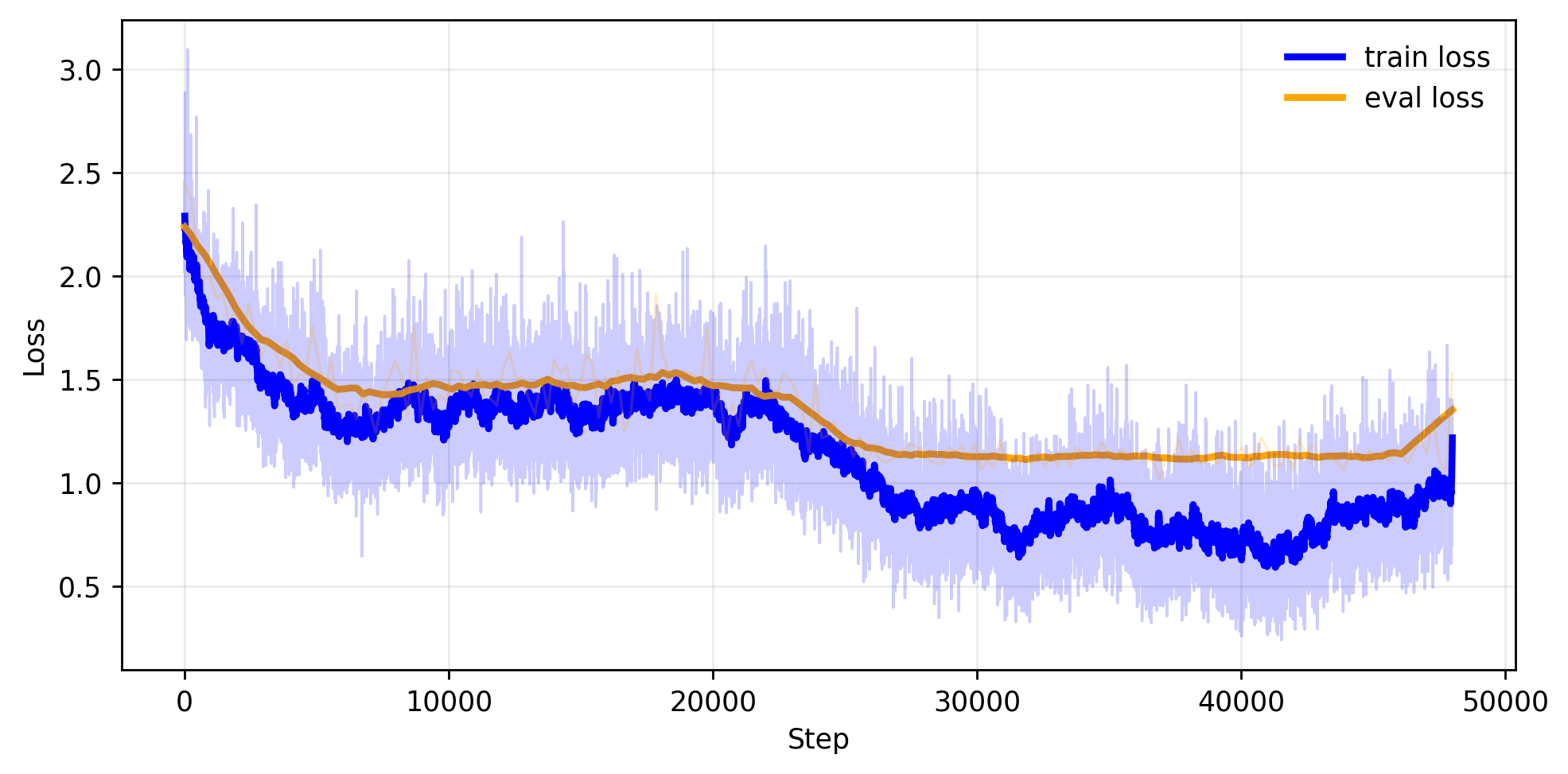}\\
{\small (f) Data=3000, Epochs=10}
\end{minipage}

\vspace{2mm}

\begin{minipage}{0.48\textwidth}
\centering
\includegraphics[width=\linewidth]{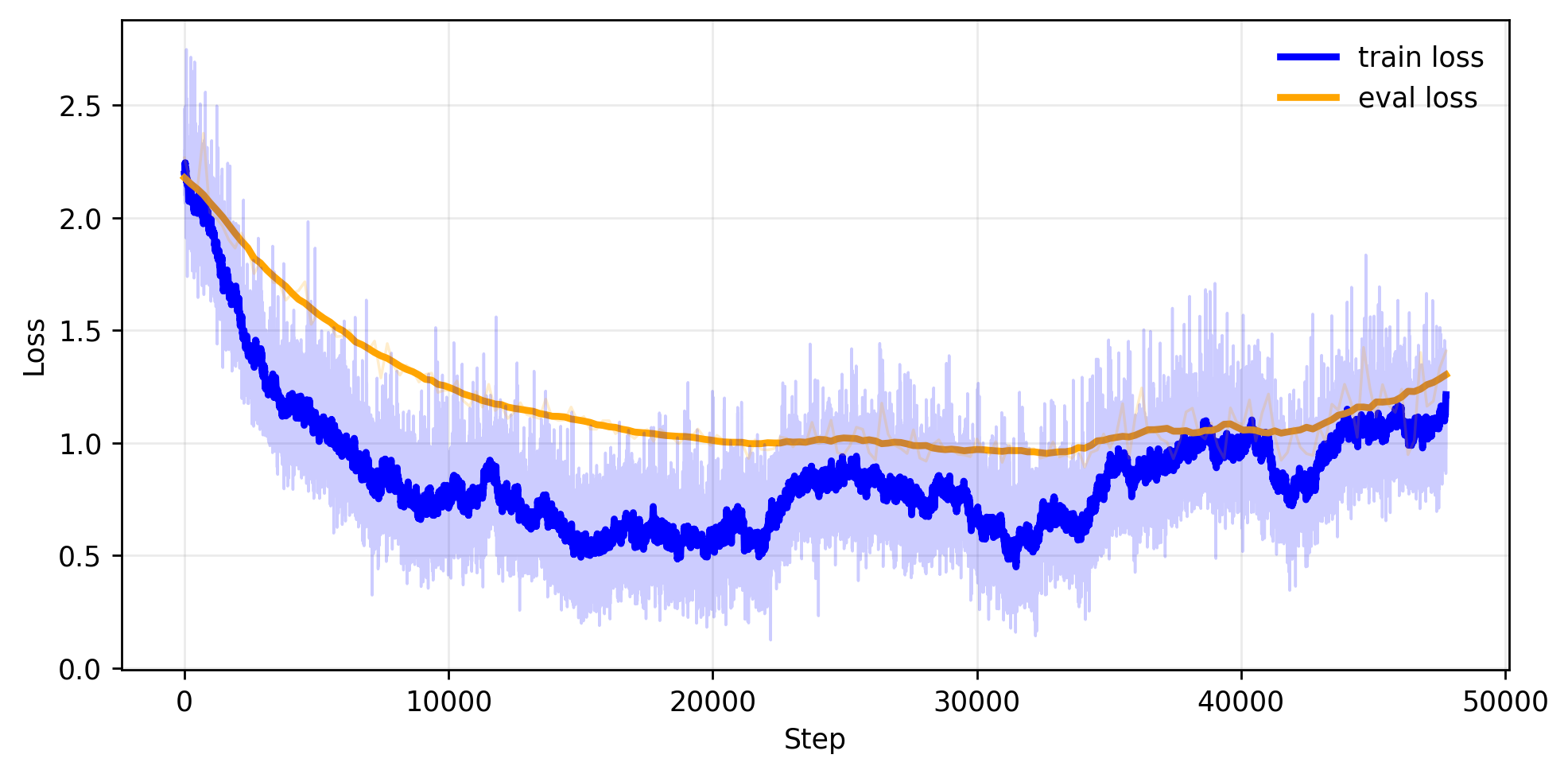}\\
{\small (g) Data=4000, Epochs=5}
\end{minipage}\hfill
\begin{minipage}{0.48\textwidth}
\centering
\includegraphics[width=\linewidth]{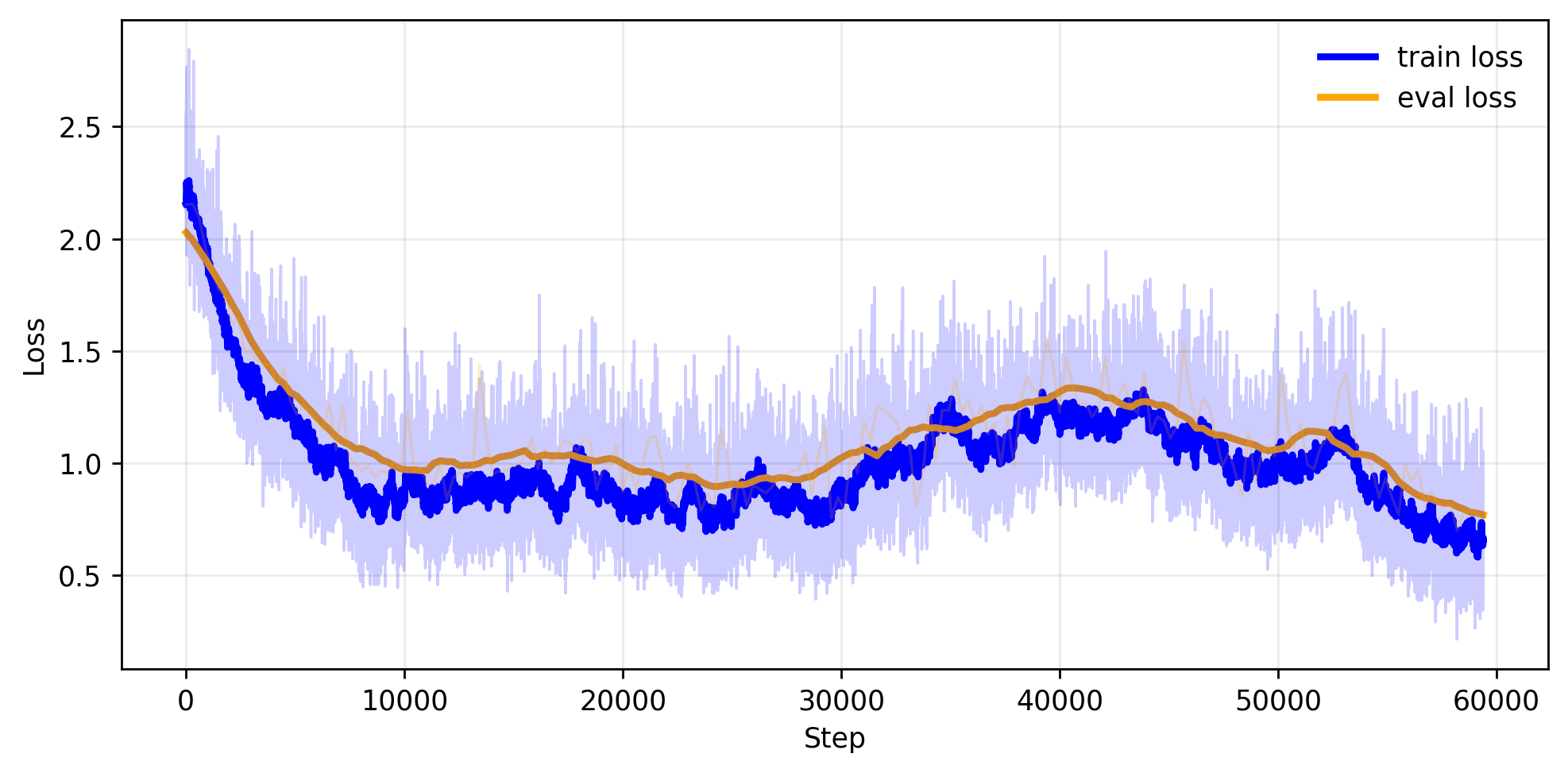}
\\
{\small (h) Data=4000, Epochs=10}
\end{minipage}

\vspace{2mm}

\begin{minipage}{0.48\textwidth}
\centering
\includegraphics[width=\linewidth]{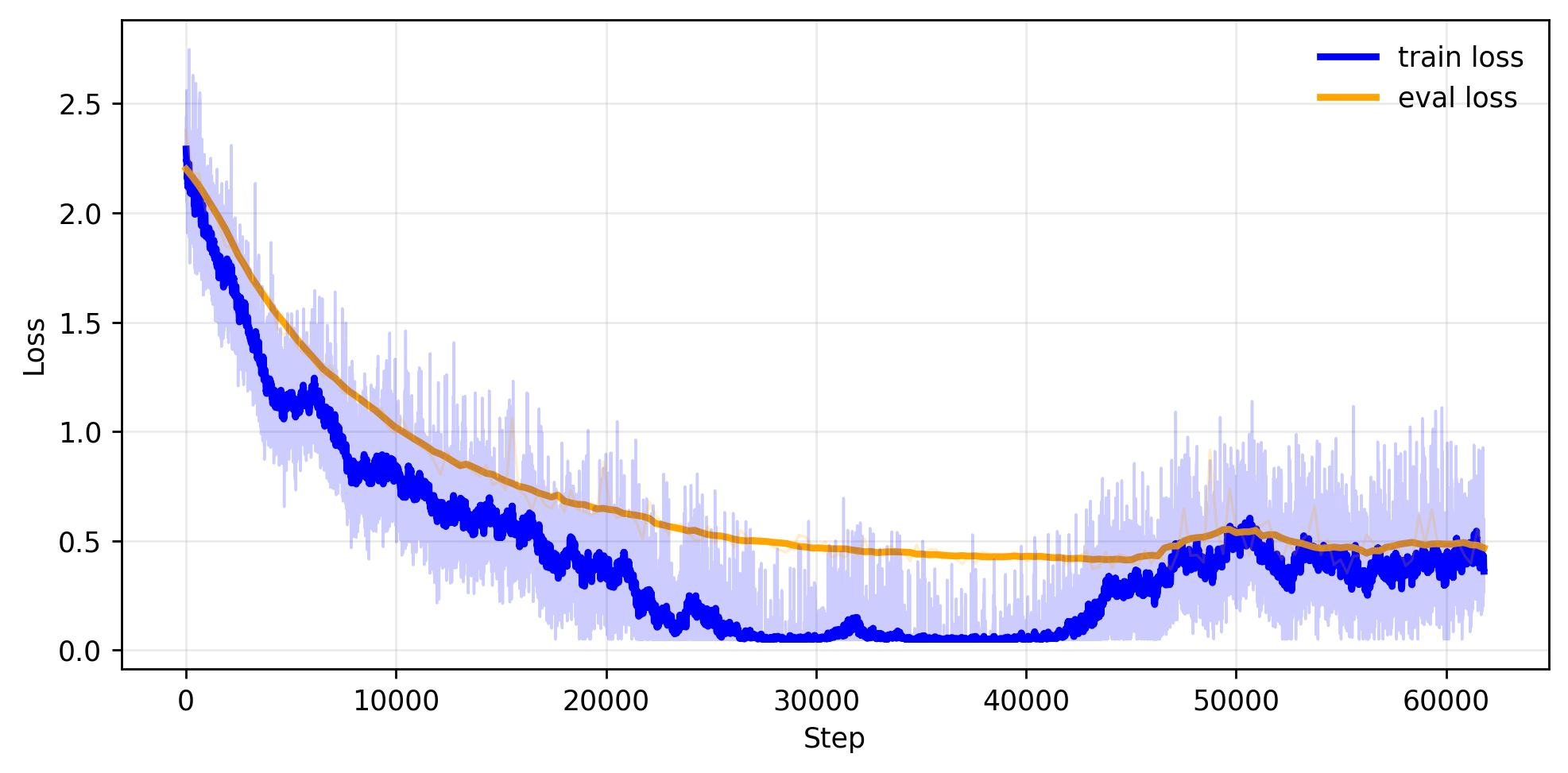}\\
{\small (i) Data=5000, Epochs=5}
\end{minipage}\hfill
\begin{minipage}{0.48\textwidth}
\centering
\includegraphics[width=\linewidth]{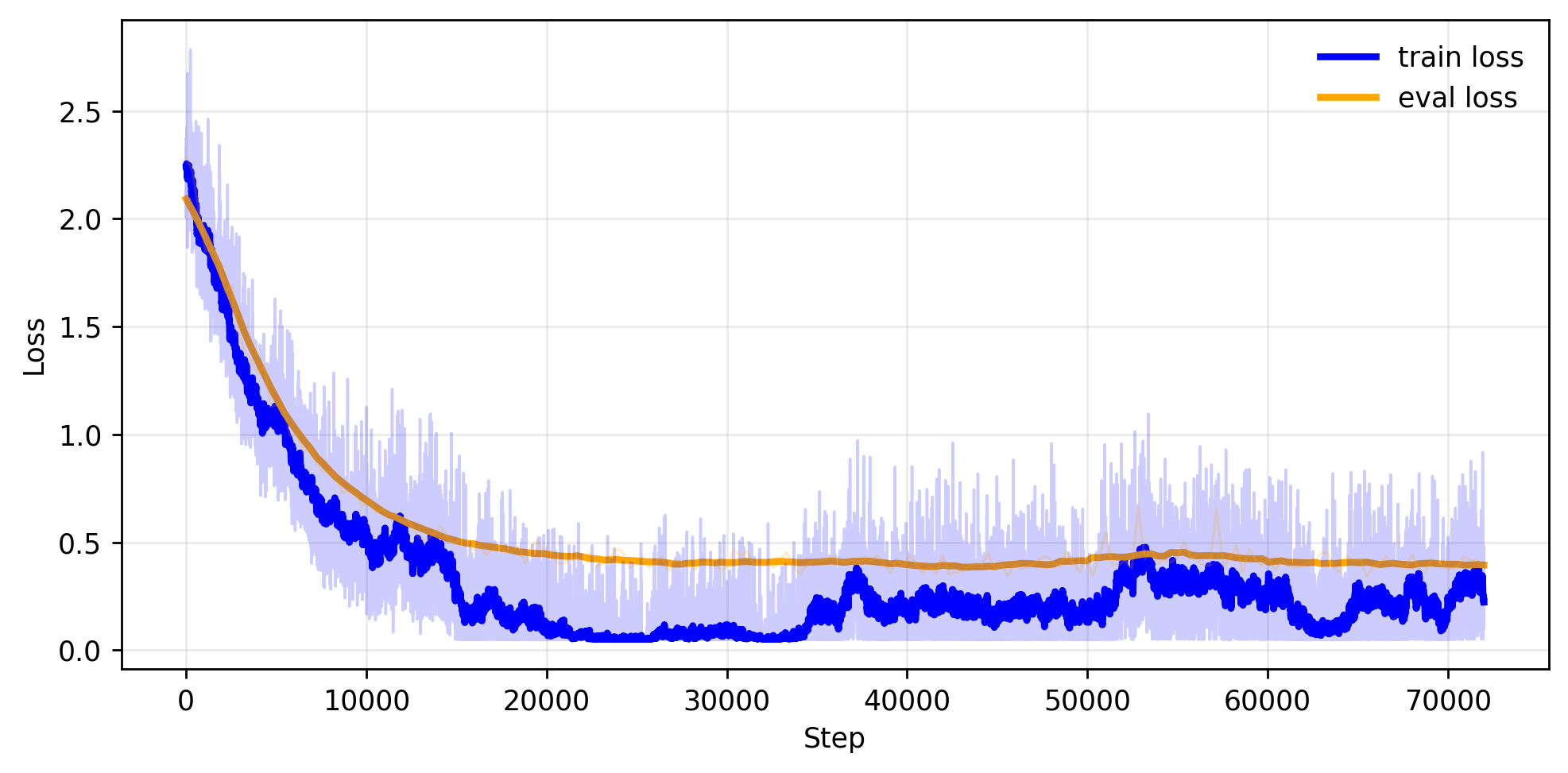}\\
{\small (j) Data=5000, Epochs=10}
\end{minipage}

\caption{Training and evaluation loss curves for Table~\ref{tab:datasize_format_gpt5_llama}, (Part 2).}
\label{fig:loss_grid_5x2}
\end{figure}
\FloatBarrier

\begin{figure}[t]
\centering

\begin{minipage}{0.48\textwidth}
\centering
\includegraphics[width=\linewidth]{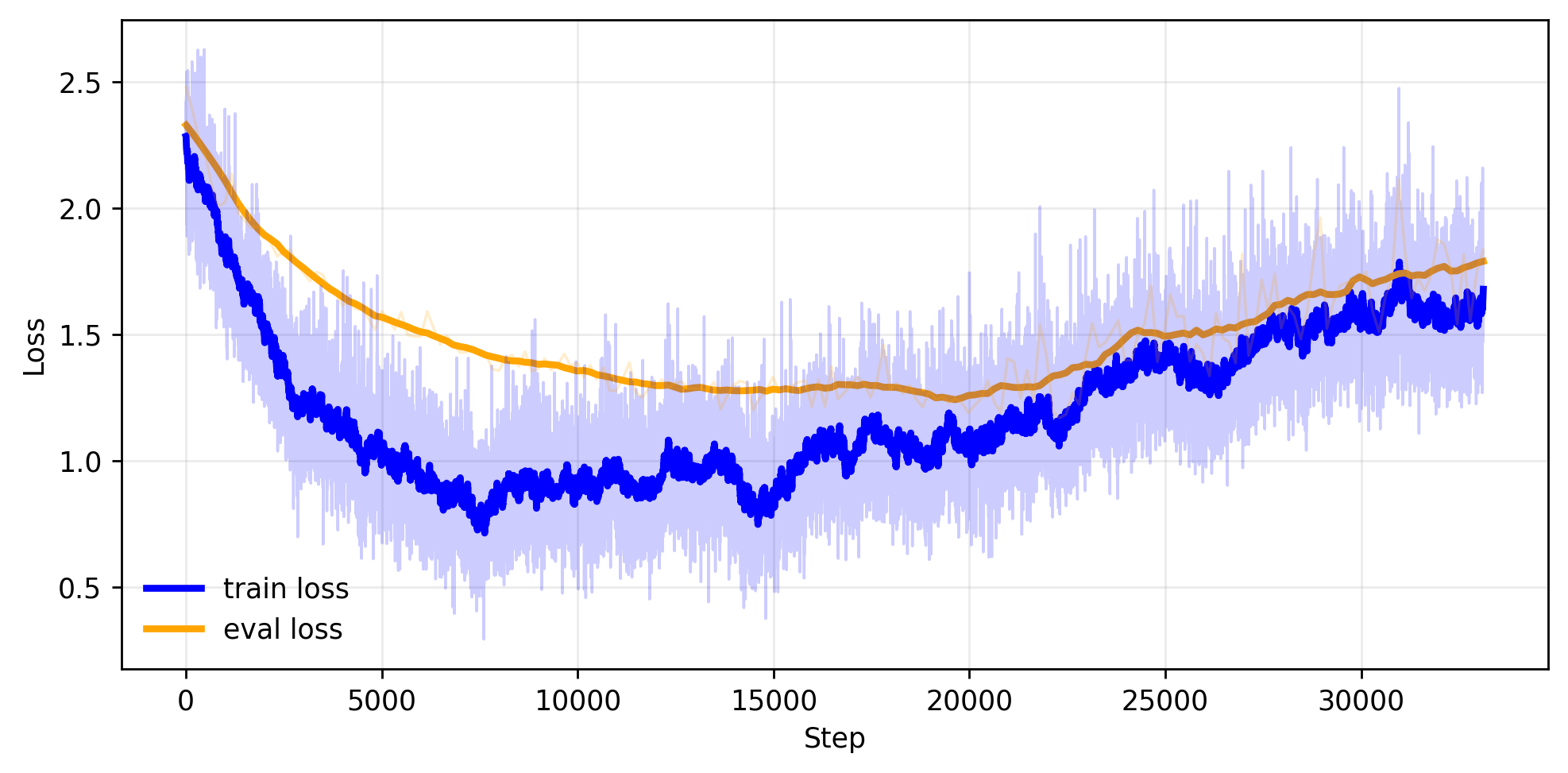}\\
{\small (e) Data=3000, Epochs=5}
\end{minipage}\hfill
\begin{minipage}{0.48\textwidth}
\centering
\includegraphics[width=\linewidth]{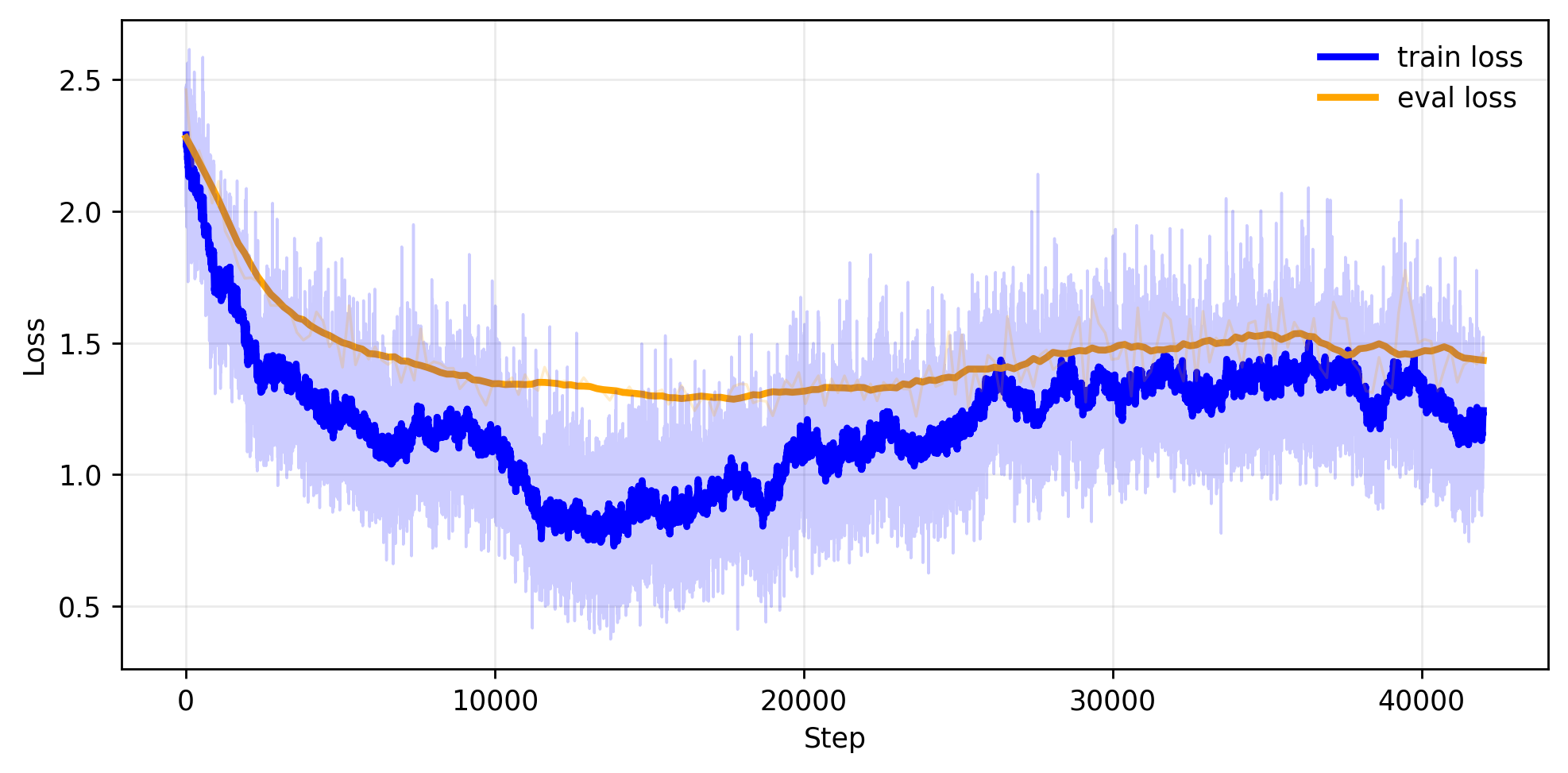}\\
{\small (f) Data=3000, Epochs=10}
\end{minipage}

\vspace{2mm}

\begin{minipage}{0.48\textwidth}
\centering
\includegraphics[width=\linewidth]{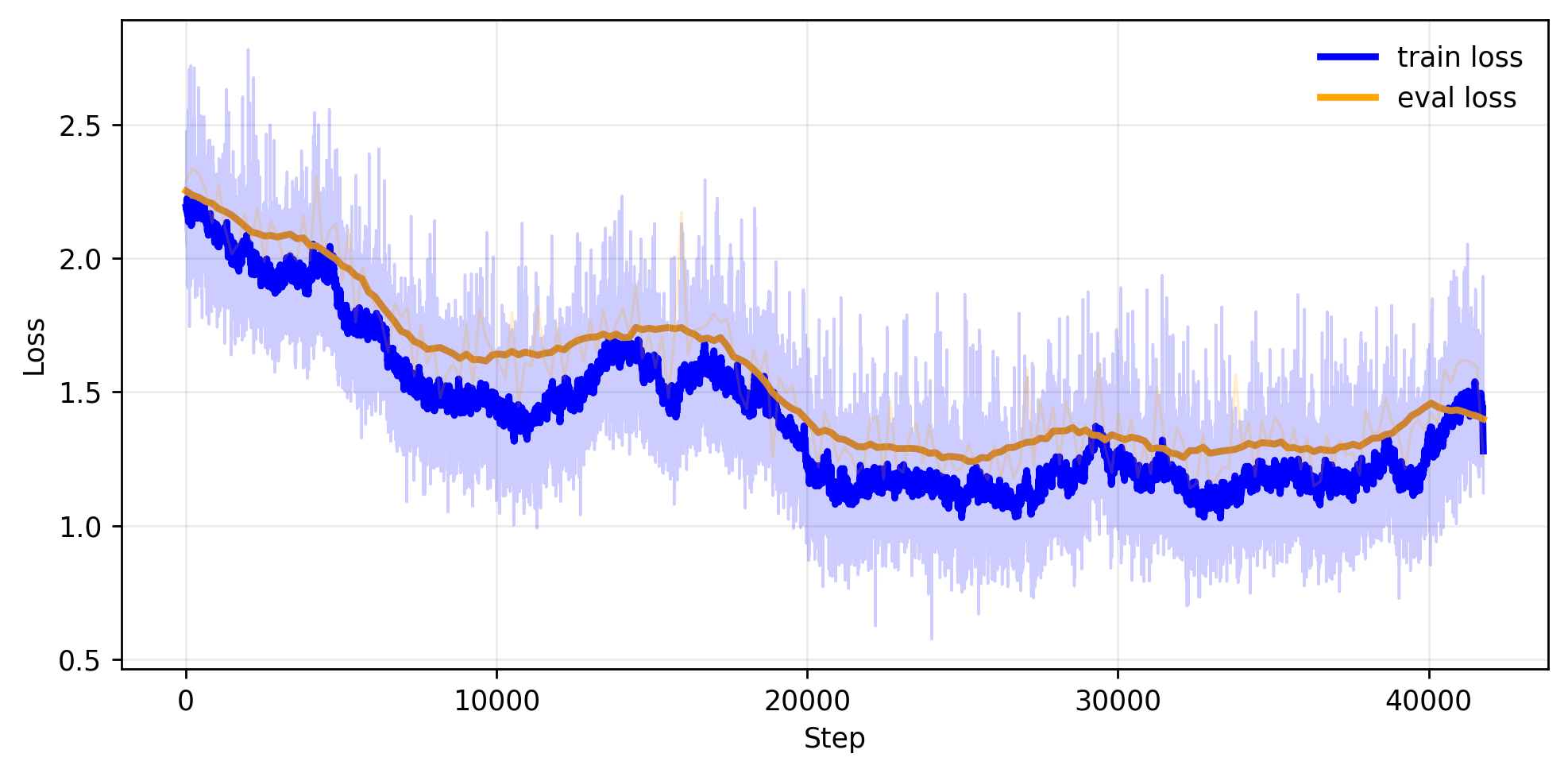}\\
{\small (g) Data=4000, Epochs=5}
\end{minipage}\hfill
\begin{minipage}{0.48\textwidth}
\centering
\includegraphics[width=\linewidth]{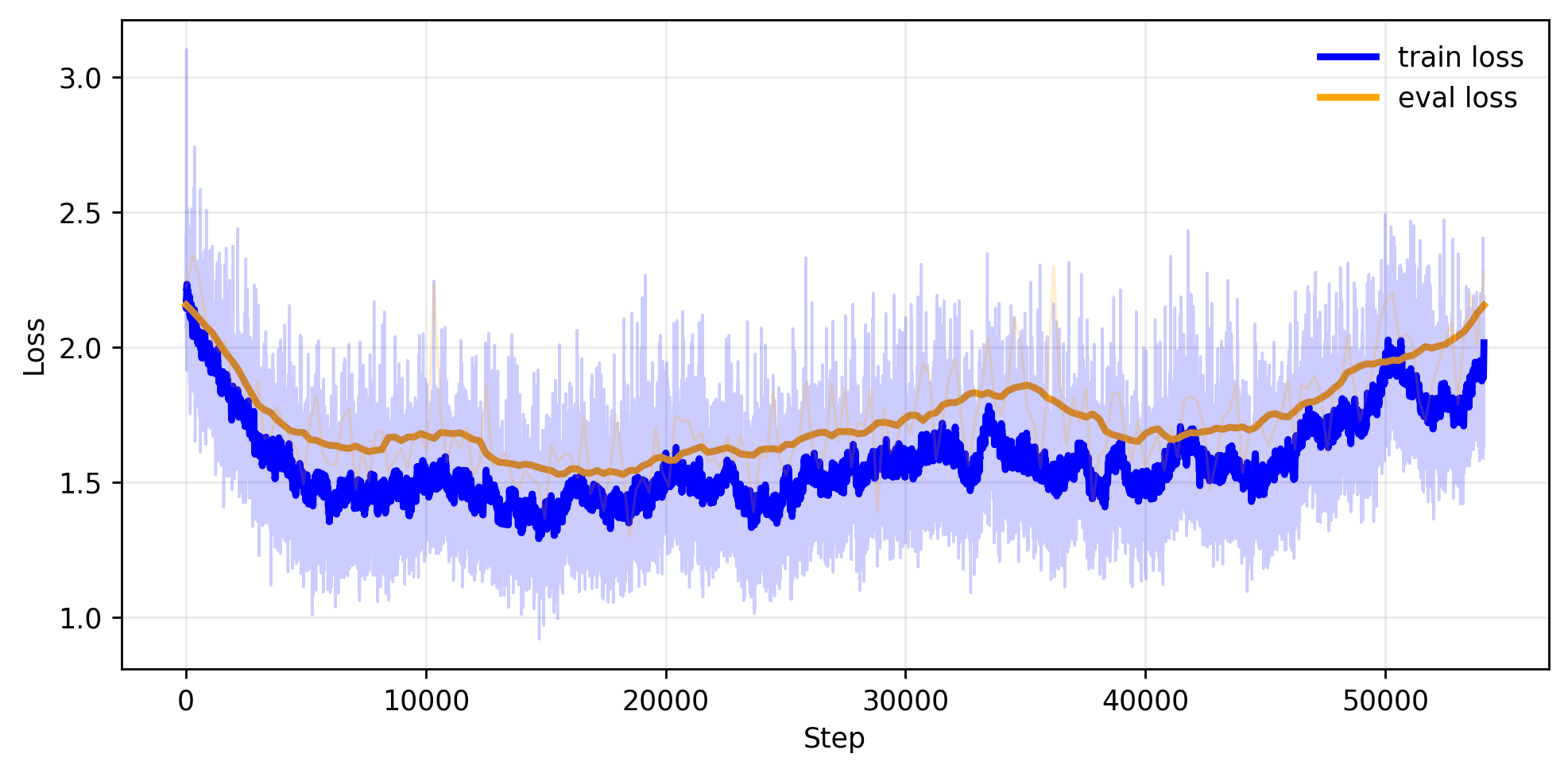}\\
{\small (h) Data=4000, Epochs=10}
\end{minipage}

\vspace{2mm}

\begin{minipage}{0.48\textwidth}
\centering
\includegraphics[width=\linewidth]{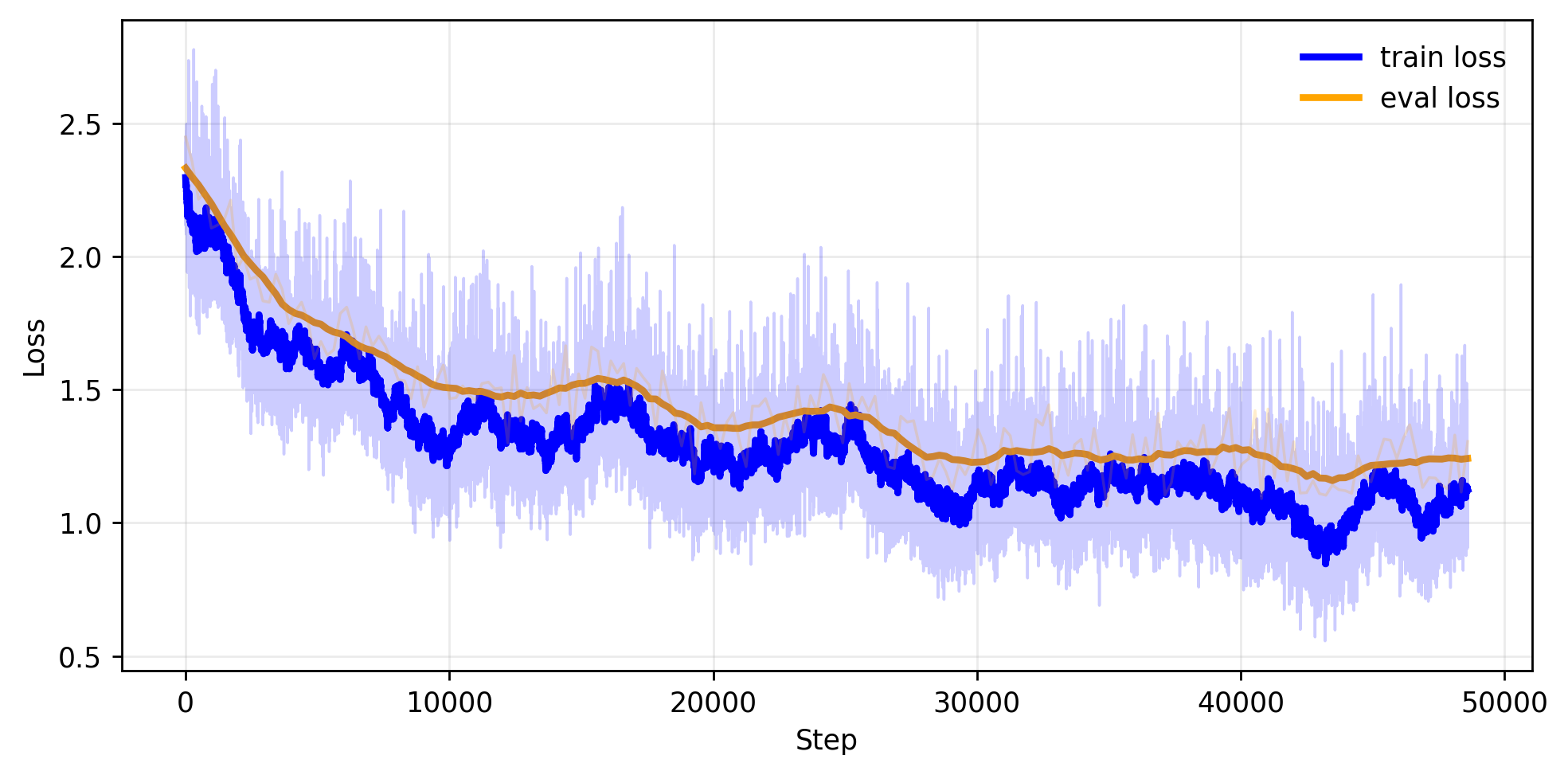}
\\
{\small (i) Data=5000, Epochs=5}
\end{minipage}\hfill
\begin{minipage}{0.48\textwidth}
\centering
\includegraphics[width=\linewidth]{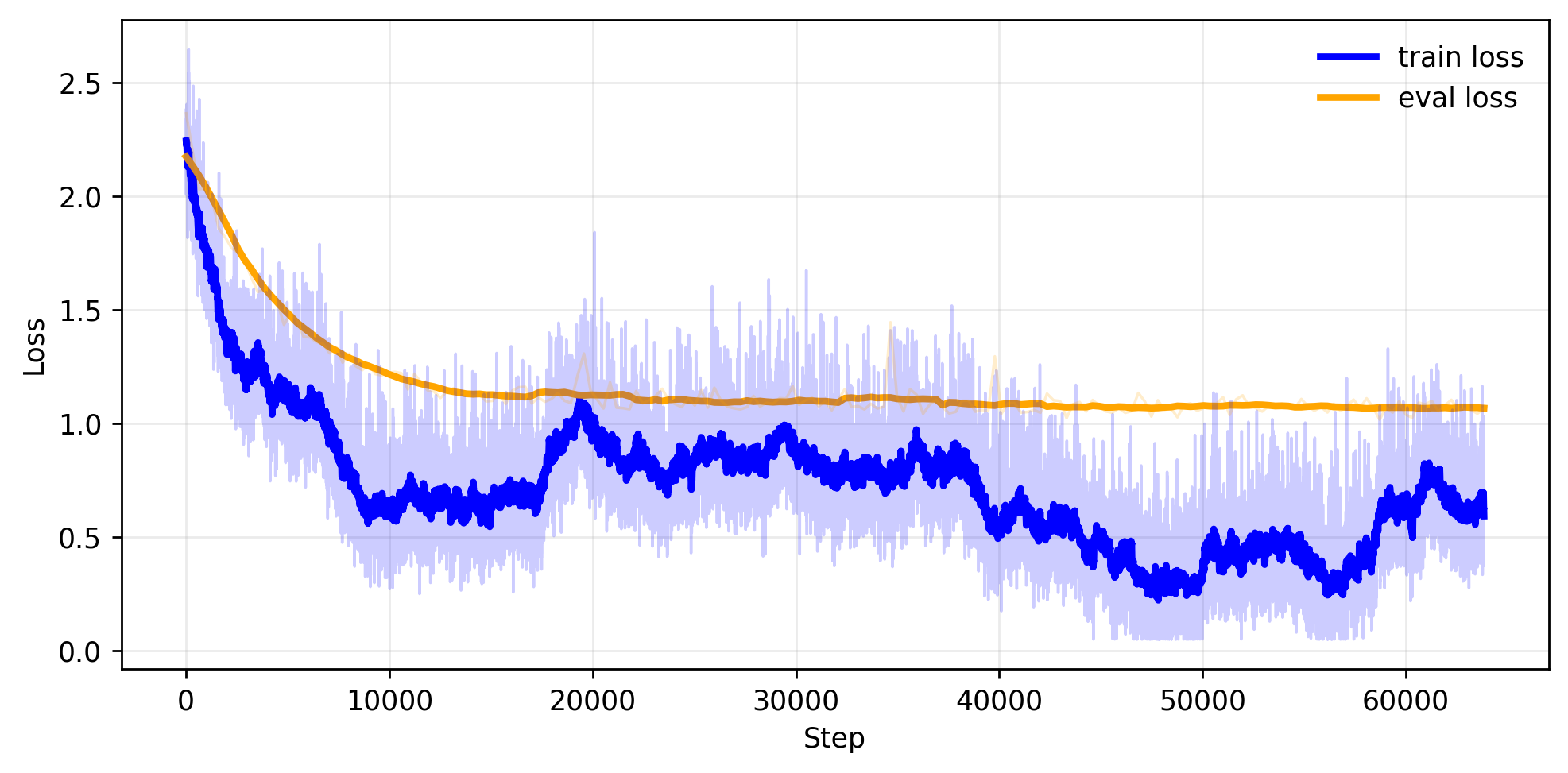}\\
{\small (j) Data=5000, Epochs=10}
\end{minipage}

\caption{Training and evaluation loss curves for Table~\ref{tab:datasize_format_gpt5_mistral}, (Part 1).}
\label{fig:loss_grid_5x2}
\end{figure}
\FloatBarrier

\begin{figure}[t]
\centering

\begin{minipage}{0.48\textwidth}
\centering
\includegraphics[width=\linewidth]{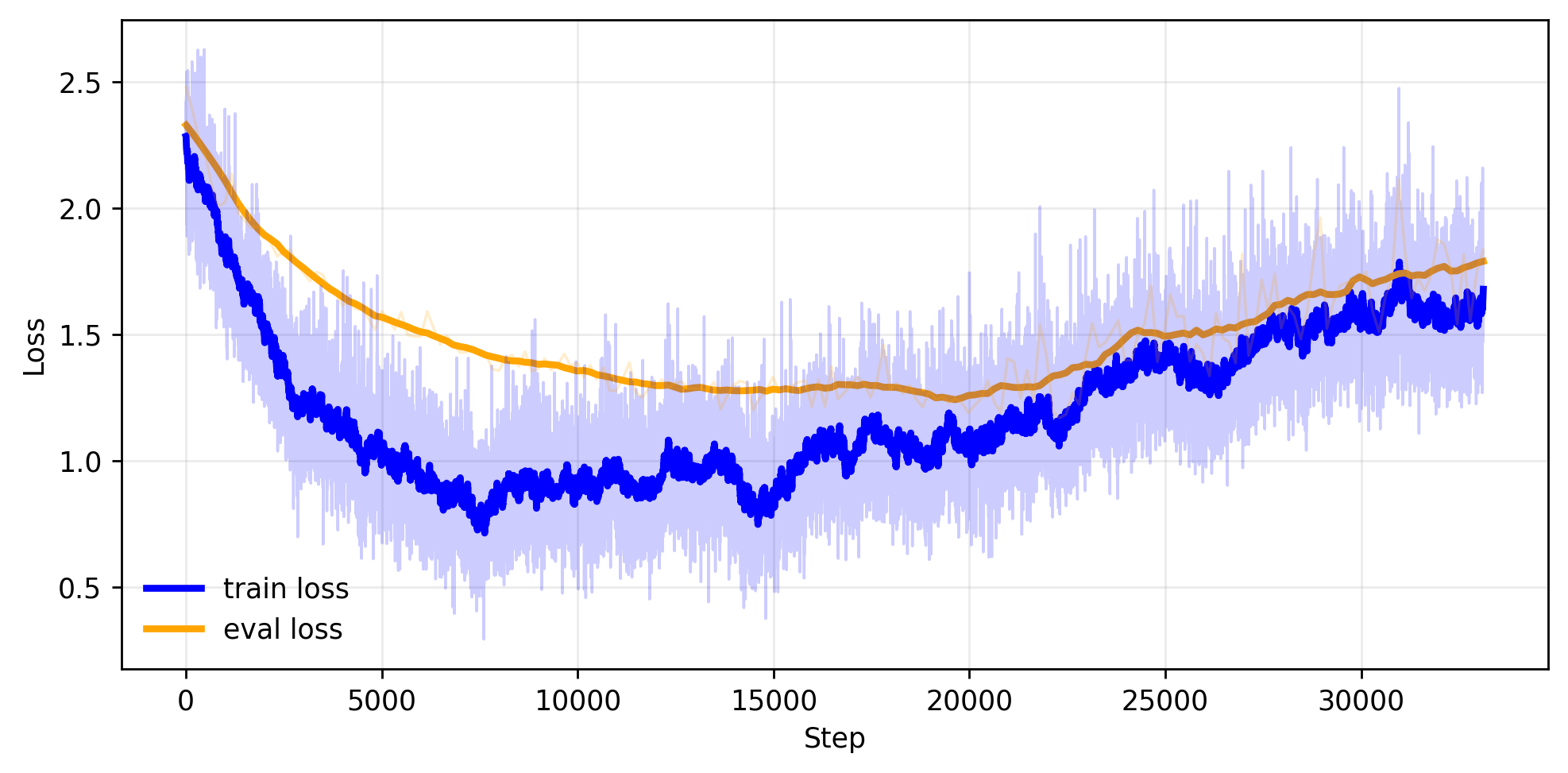}\\
{\small (e) Data=3000, Epochs=5}
\end{minipage}\hfill
\begin{minipage}{0.48\textwidth}
\centering
\includegraphics[width=\linewidth]{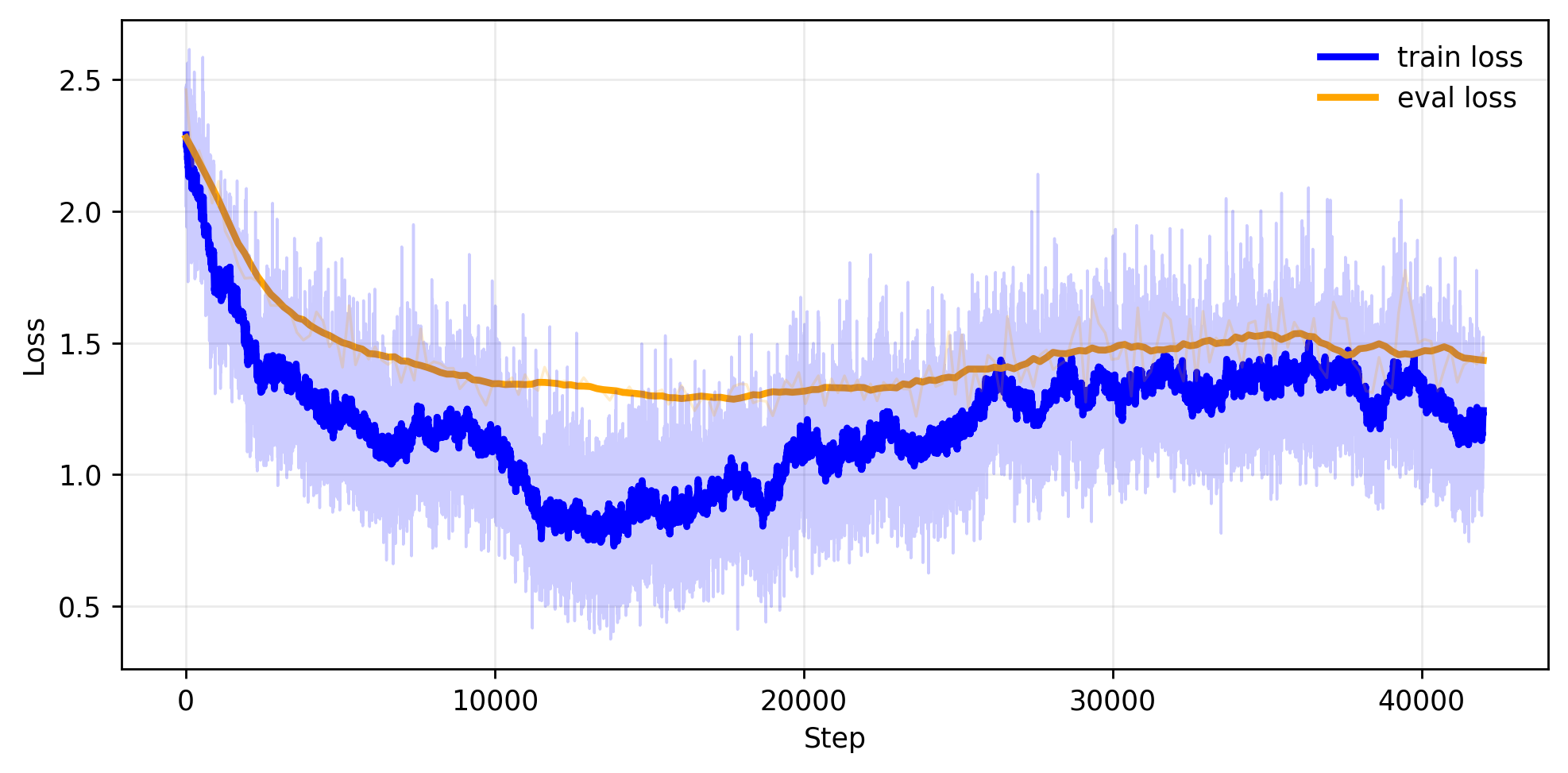}\\
{\small (f) Data=3000, Epochs=10}
\end{minipage}

\vspace{2mm}

\begin{minipage}{0.48\textwidth}
\centering
\includegraphics[width=\linewidth]{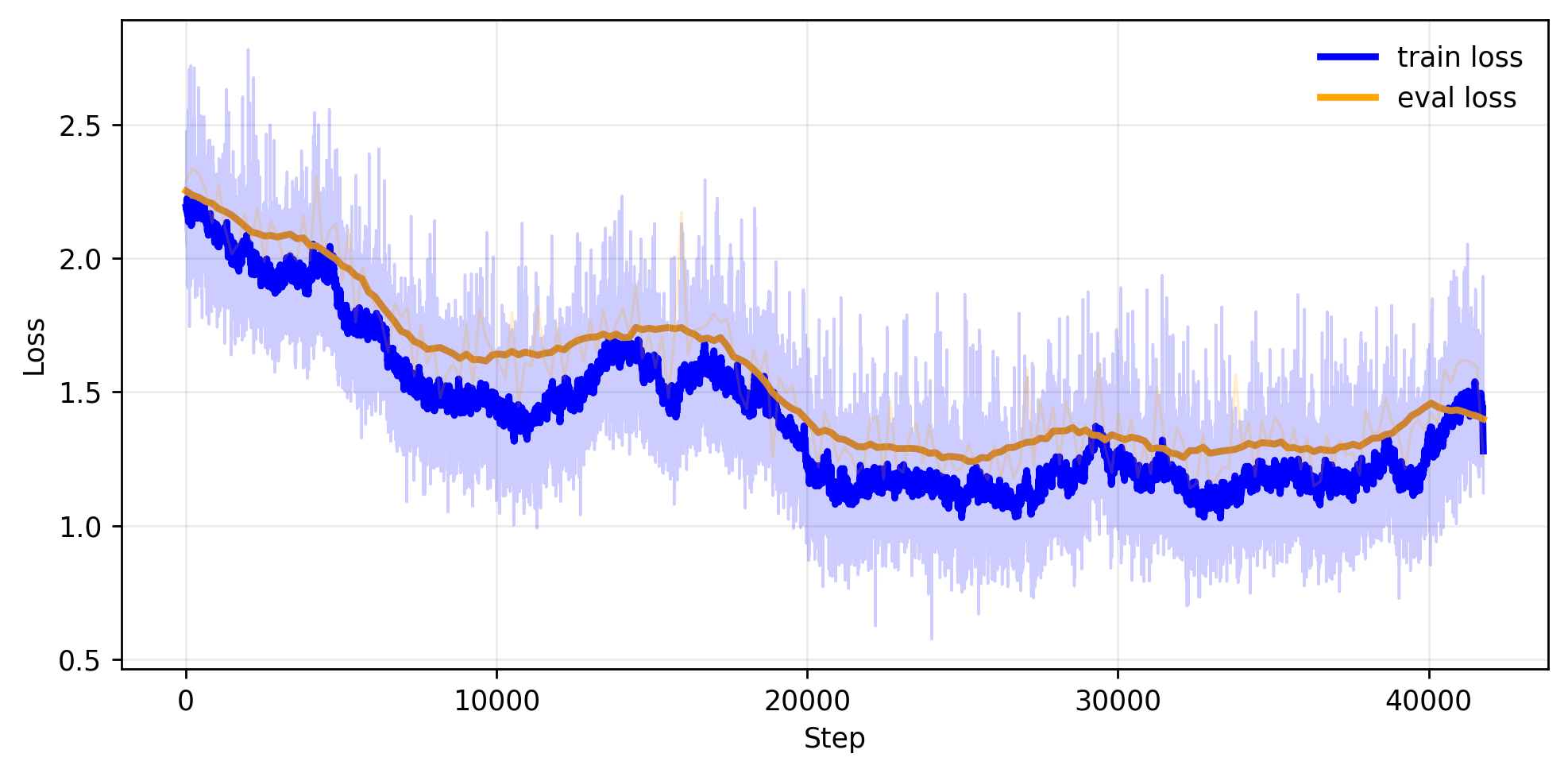}\\
{\small (g) Data=4000, Epochs=5}
\end{minipage}\hfill
\begin{minipage}{0.48\textwidth}
\centering
\includegraphics[width=\linewidth]{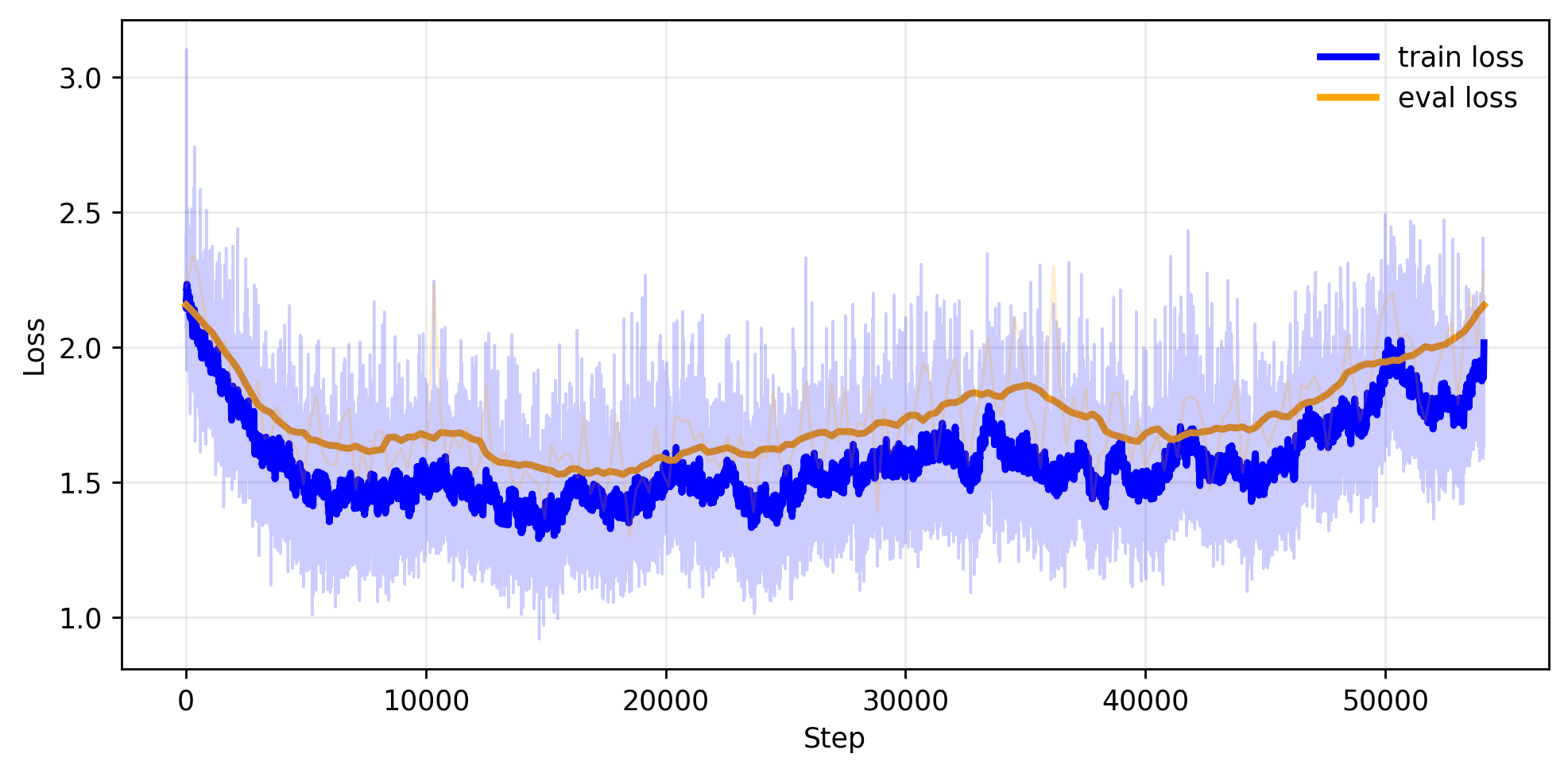}\\
{\small (h) Data=4000, Epochs=10}
\end{minipage}

\vspace{2mm}

\begin{minipage}{0.48\textwidth}
\centering
\includegraphics[width=\linewidth]{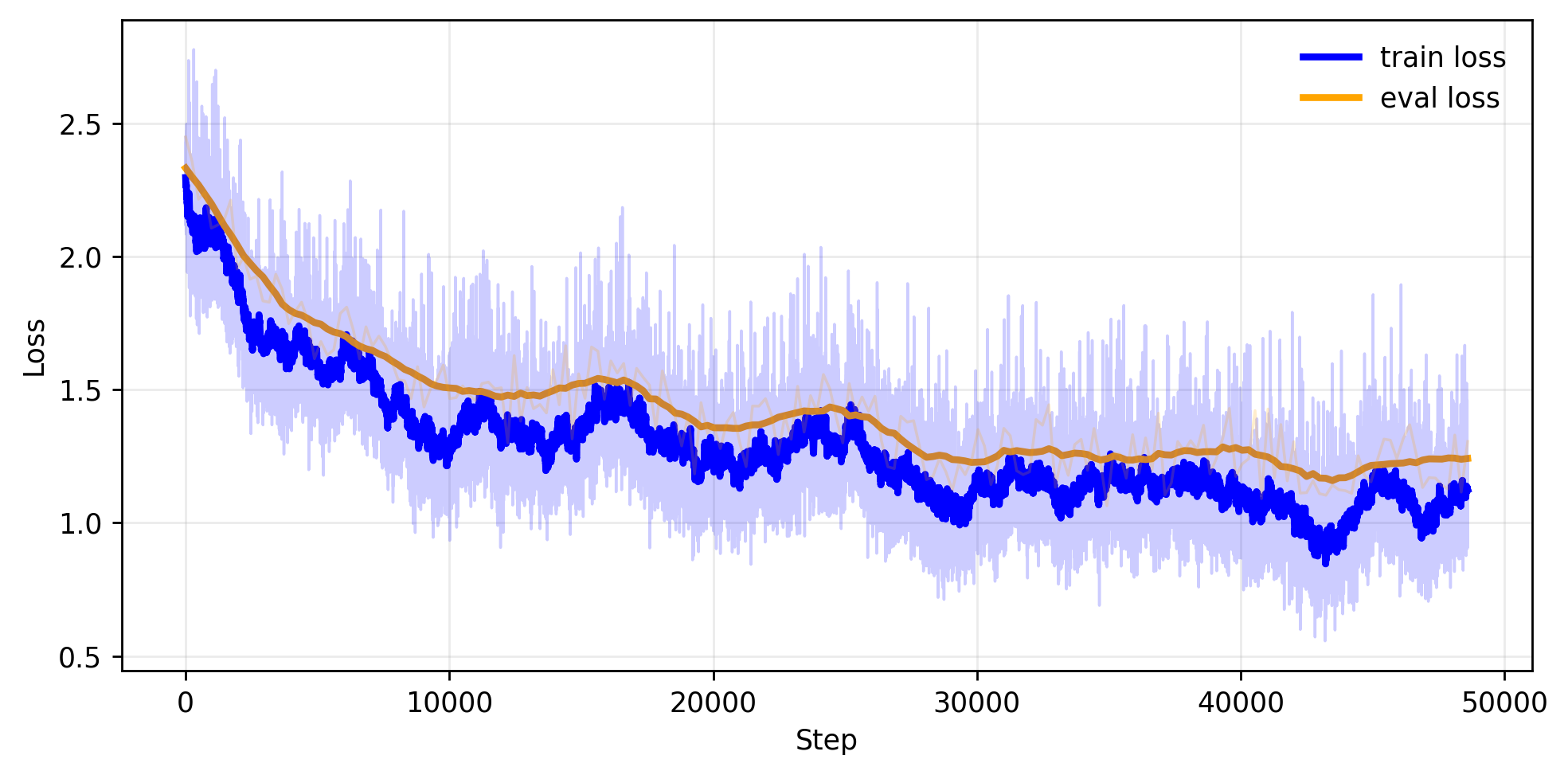}
\\
{\small (i) Data=5000, Epochs=5}
\end{minipage}\hfill
\begin{minipage}{0.48\textwidth}
\centering
\includegraphics[width=\linewidth]{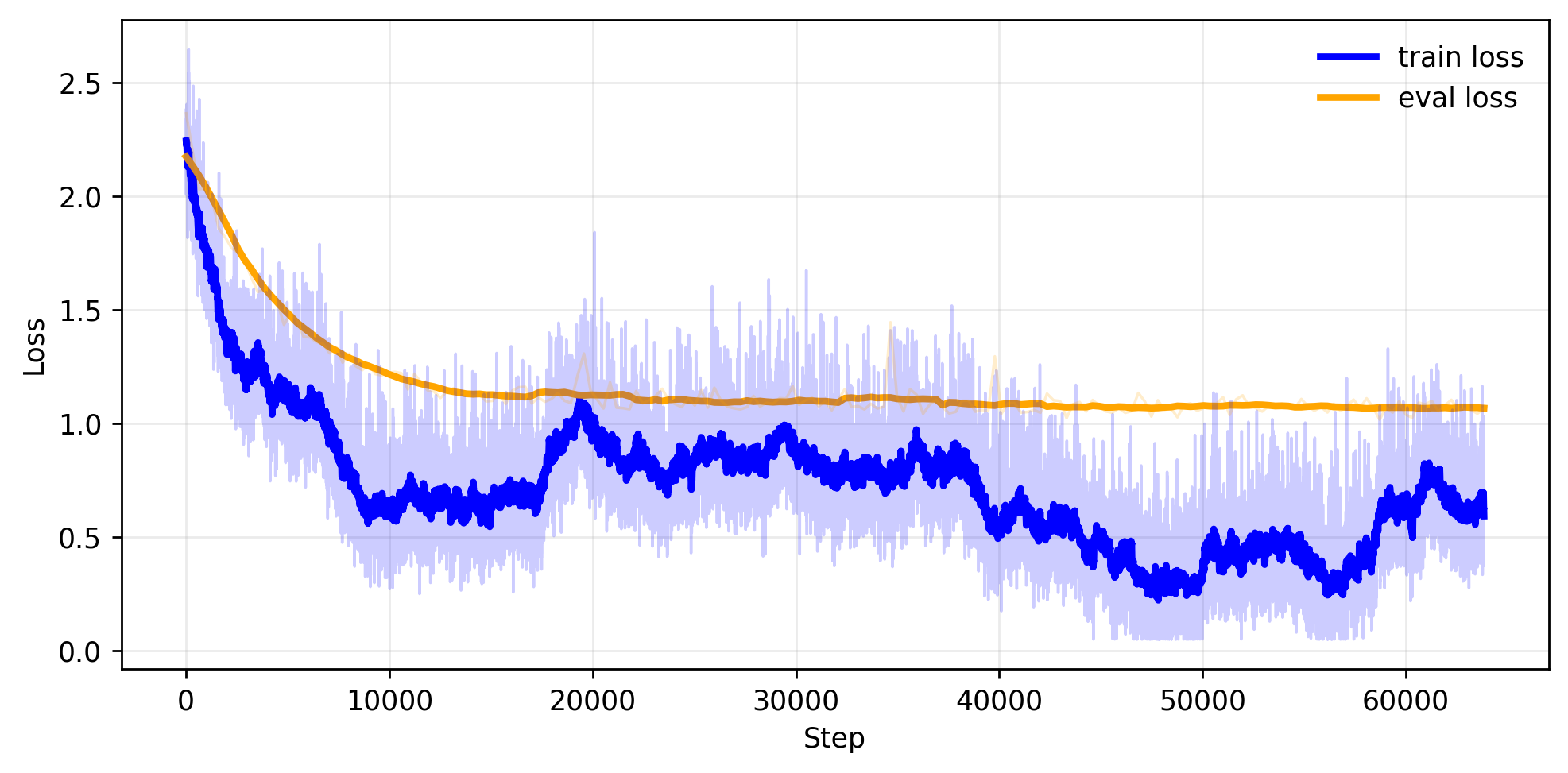}\\
{\small (j) Data=5000, Epochs=10}
\end{minipage}

\caption{Training and evaluation loss curves for Table~\ref{tab:datasize_format_gpt5_mistral}, (Part 2).}
\label{fig:loss_grid_5x2}
\end{figure}
\FloatBarrier

\begin{figure}[t]
  \centering
  \includegraphics[width=0.7\textwidth]{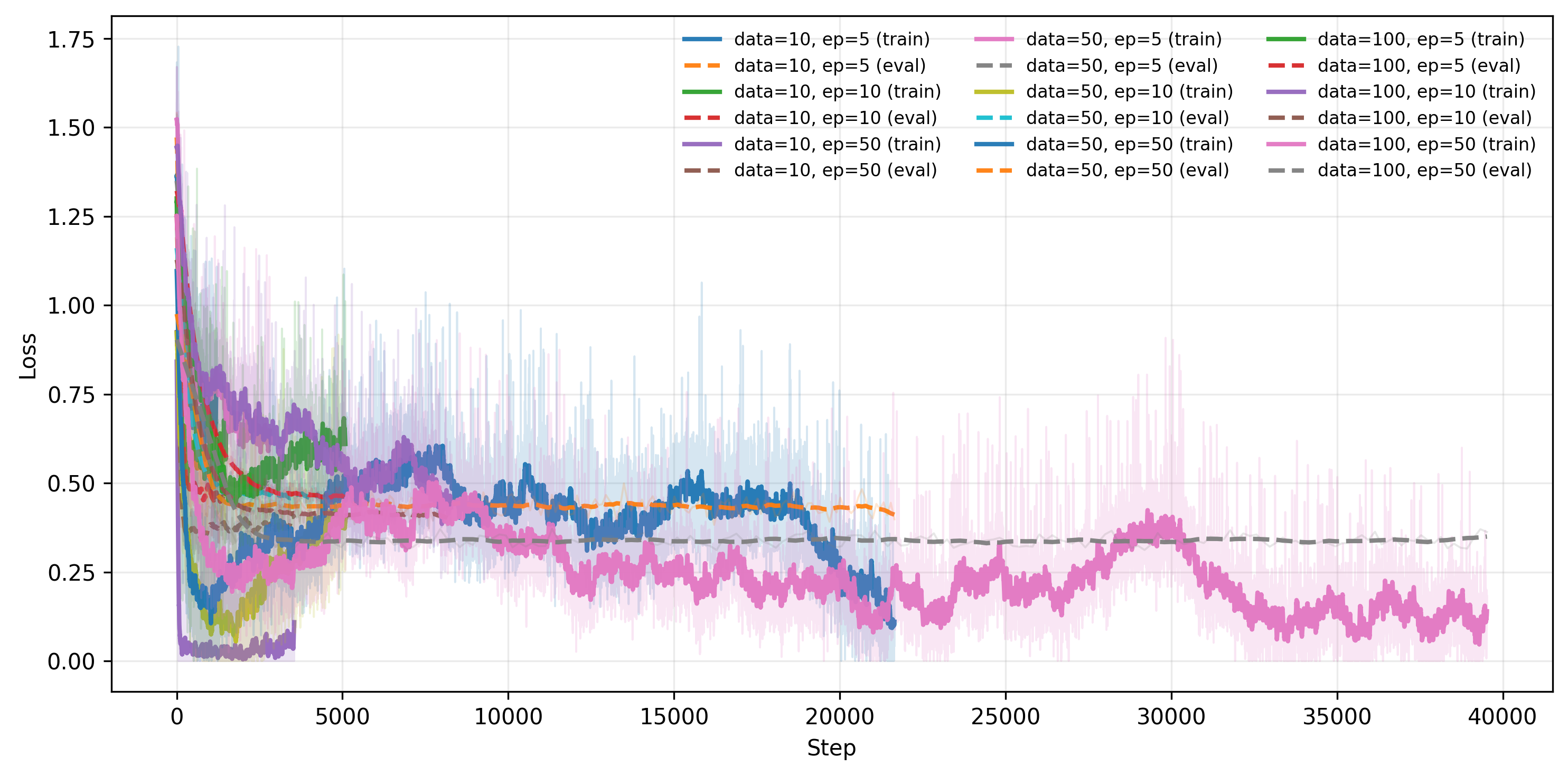}
  \caption{Training and evaluation loss curves for Table~\ref{tab:datasize_criteria_llama_gpt5}.}
  \label{fig:myplot1}
\end{figure}

\begin{figure}[t]
  \centering
  \includegraphics[width=0.7\textwidth]{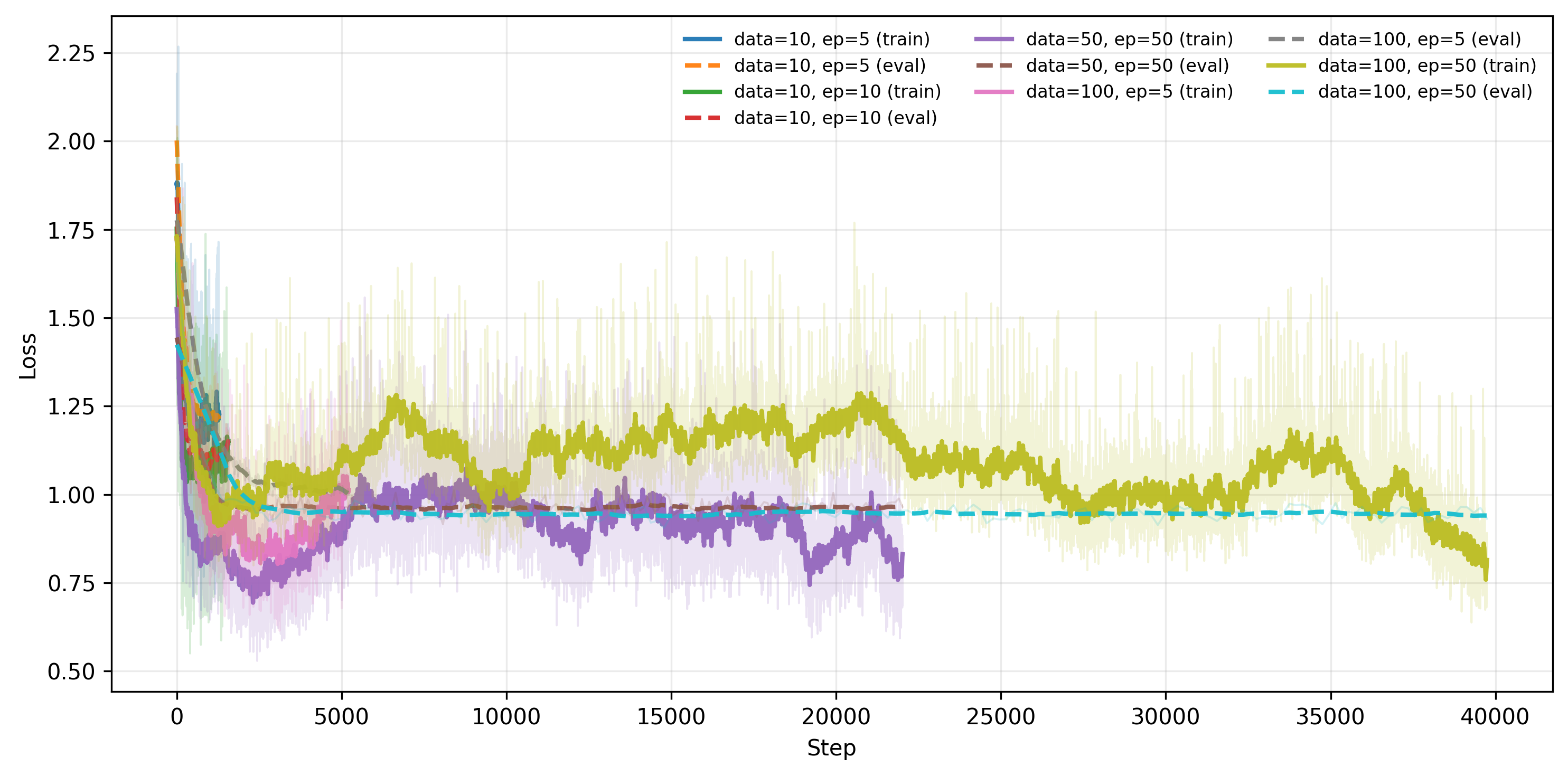}
  \caption{Training and evaluation loss curves for Table~\ref{tab:datasize_criteria_mistral_gpt5}.}
  \label{fig:myplot1}
\end{figure}

\begin{figure}[t]
  \centering
  \includegraphics[width=0.7\textwidth]{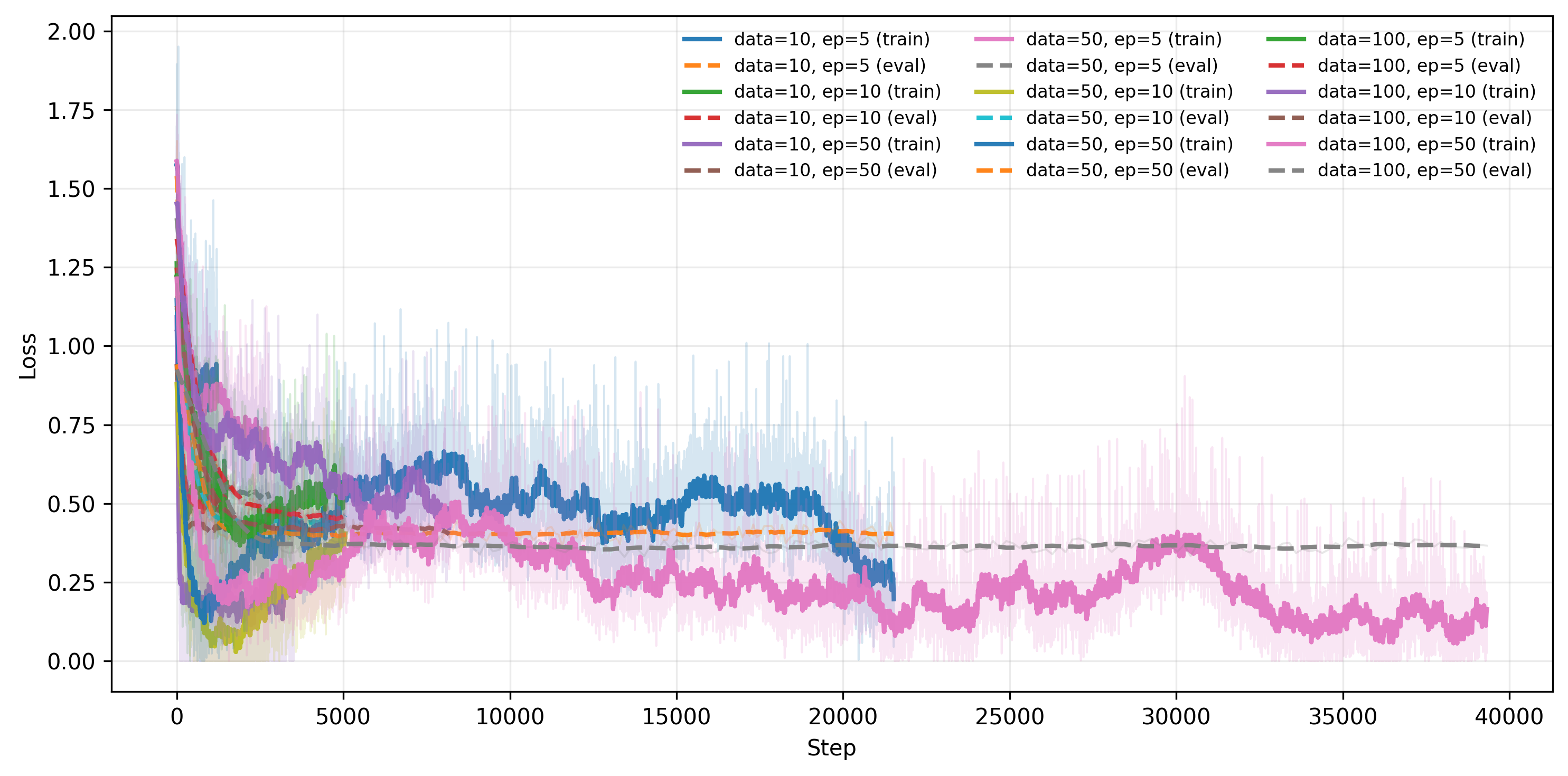}
  \caption{Training and evaluation loss curves for Table~\ref{tab:datasize_criteria_llama_gpt5nano}.}
  \label{fig:myplot1}
\end{figure}

\begin{figure}[t]
  \centering
  \includegraphics[width=0.7\textwidth]{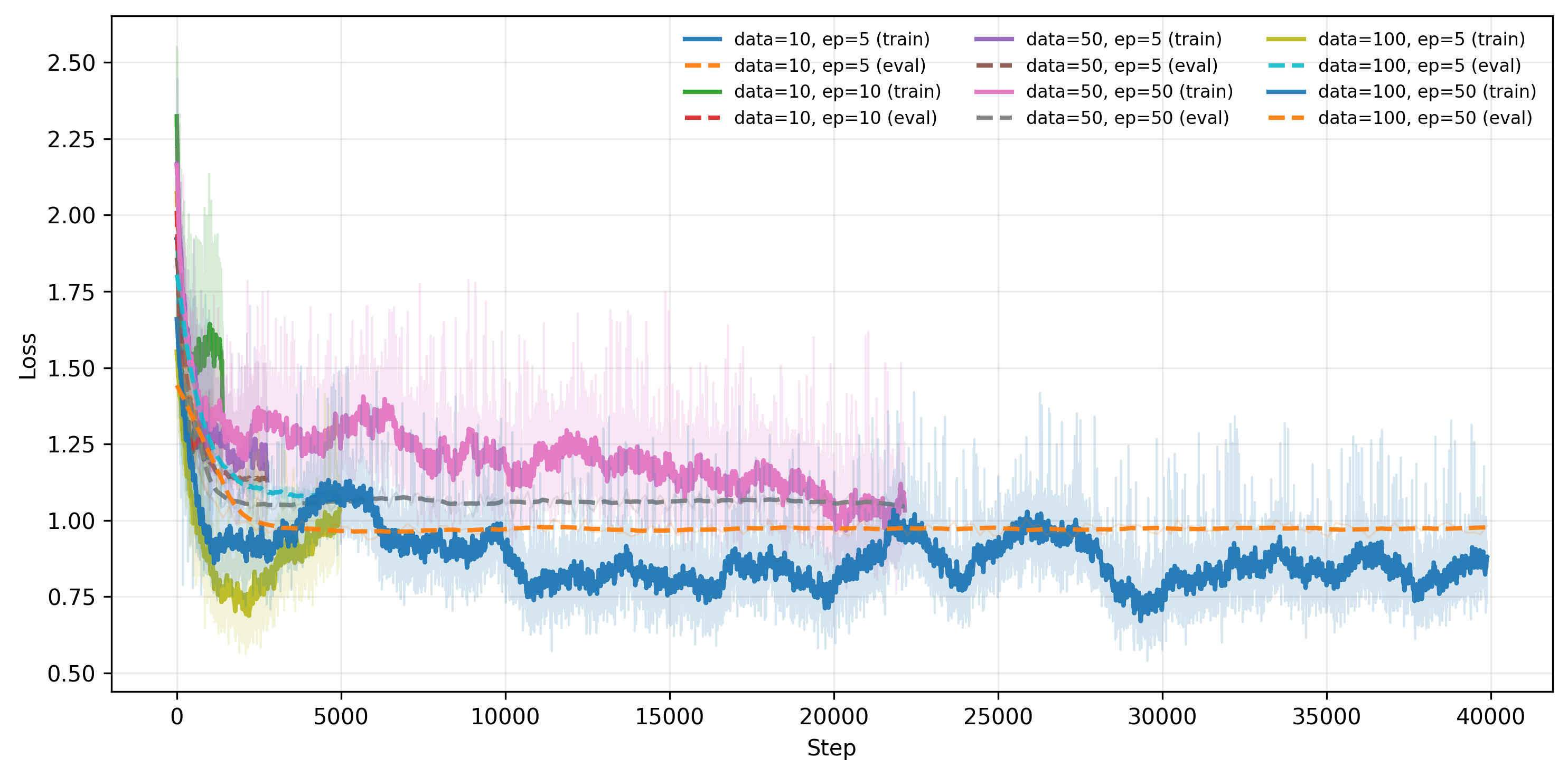}
  \caption{Training and evaluation loss curves for Table~\ref{tab:datasize_criteria_mistral_gpt5nano}.}
  \label{fig:myplot1}
\end{figure}
\FloatBarrier

\newpage
\section{Prompts}
\label{appendix:case_study}

\begin{promptbox}{Prompt Base}
\label{prompt_box:prompt}
\small
\textbf{System.} You got this information which is a list of values on format (i,fidelity,cpu time): $\{$\emph{model info}$\}$.
This indicates the model i has a certain fidelity with respect to other models (the higher the better) and takes a certain time to find the solution. Our goal is to discover a function through sampling. 
Since now, the evaluated points are $\{$\emph{eval points}$\}$ with a distance to the optimal solution of $\{$\emph{error}$\}$ over the domain $\{$\emph{domain}$\}$ the uncertainty is $\{$\emph{uncertainty}$\}$. Please, select which should be the next point in the domain to be evaluated and which fidelity model should be used considering only $\{$\emph{remaining time}$\}$ is left. Strictly return the solution as RESULT=[model,point].
\end{promptbox}

\begin{promptbox}{Degenerated Response (Quantization over CPU)}
\small
\textbf{System.} Select a point in the domain [1,10] and a model from [1,10] and return the results in format RESULT = [model, point].

\textbf{User.} Vevvevevevee $\setminus n \setminus n \setminus n \setminus n$ the
\end{promptbox}

\begin{promptbox}{Correct Semantic but Wrong Protocol Response}
\small
\textbf{System.} Select a point in the domain [1,10] and a model from [1,10] and return the results in format RESULT = [model, point].

\textbf{User.} You should use model 3 in the point 4.3.
\end{promptbox}

\begin{promptbox}{Noisy \emph{GPT 5} Response}
\small
\textbf{System.}  You got this information which is a list of values on format (i,fidelity,cpu time): $\{$\emph{model info}$\}$. This indicates the model i has a certain fidelity with respect to other models (the higher the better) and takes a certain time to find the solution. Our goal is to discover a function through sampling. Since now, the evaluated points are $\{$\emph{eval points}$\}$ with a distance to the optimal solution of $\{$\emph{error}$\}$ over the domain $\{$\emph{domain}$\}$ the uncertainty is $\{$\emph{uncertainty}$\}$. Please, select which should be the next point in the domain to be evaluated and which fidelity model should be used considering only $\{$\emph{remaining time}$\}$ is left. Strictly return the solution as RESULT=[model,point].

\textbf{User.} You got this information which is a list of values on format (i,fidelity,cpu time): [(1, 0.14, 58), (2, 0.77, 11), (3, 0.85, 144), (4, 0.6, 96), (5, 0.21, 56)].$\setminus n$ This indicates the model i has a certain fidelity with respect to other models (the higher the better) and takes a certain time to find the solution.$\setminus n$  Our goal is to discover a function through sampling. Since now, the evaluated points are [5.41
\end{promptbox}

\begin{promptbox}{Noisy \emph{Llama-3.1-8B} Response}
\small
\textbf{System.} You got this information which is a list of values on format (i,fidelity,cpu time): $\{$\emph{model info}$\}$. This indicates the model i has a certain fidelity with respect to other models (the higher the better) and takes a certain time to find the solution. Our goal is to discover a function through sampling. Since now, the evaluated points are $\{$\emph{eval points}$\}$ with a distance to the optimal solution of $\{$\emph{error}$\}$ over the domain $\{$\emph{domain}$\}$ the uncertainty is $\{$\emph{uncertainty}$\}$. Please, select which should be the next point in the domain to be evaluated and which fidelity model should be used considering only $\{$\emph{remaining time}$\}$ is left. Strictly return the solution as RESULT=[model,point].

\textbf{User.} $\setminus n$$\setminus n$ Arguments:$\setminus n$ evaluatedpoints: list of floats with the points already evaluated.$\setminus n$ fidelityvalues:
 list of floats with the fidelity values.$\setminus n$ uncertainty: list of floats with the uncertainty at each point.$\setminus n$ timeleft:
 int with the remaining time in seconds.$\setminus n$ domain: list of floats with the domain.$\setminus n$ fidelitymodels:
 list of tuples where the first element is the name of the fidelity model and the second element is the function that calculates                            the fidelity.$\setminus n$ timemodels: list of tuples where the first element is the name of the time model and the second element is the
 function that calculates the time.$\setminus n$ max time: int with the maximum time to evaluate a point.$\setminus n$$\setminus n$ Returns:$\setminus n$ A list with the
 next point to be evaluated and the fidelity model to be used.$\setminus n$ $\setminus ndef$ nextpoint(evaluatedpoints, fidelityvalues, uncertainty,
 timeleft, domain, fidelitymodels, timemodels, maxtime):$\setminus n$    Sort the points based on the uncertainty$\setminus n$
 sorteduncertainty = sorted(range(len(uncertainty)), key=lambda k: uncertainty[k], reverse=True)$\setminus n$
  Filter the points that have already been evaluated$\setminus n$
 pointstoevaluate = [domain[i] for i in sorteduncertainty if domain[i] not in evaluatedpoints]$\setminus n$                                                       Sort the points based on the distance to the optimal solution$\setminus n$
 sortedpoints = sorted(pointstoevaluate, key=lambda x: abs(x - 5))$\setminus n$
  Check if there are points left to evaluate$\setminus n$
 if not pointstoevaluate:$\setminus n$
 return []$\setminus n$
  Check if the fidelity model is valid$\setminus n$
 if not fidelitymodels:$\setminus n$
 return []$\setminus n$
  Select the next point to evaluate$\setminus n$
 xpoint = sortedpoints[0]$\setminus n$
  Select the fidelity model$\setminus n$
 fidelity = fidelityvalues[0]$\setminus n$
 return [xpoint, fidelity]$\setminus n$$\setminus n$ Test the function$\setminus n$evaluatedpoints = [0.2, 1.3, 3, 4.5]$\setminus n$
 fidelityvalues = [0.8, 0.2, 0.2, 0.2]$\setminus n$uncertainty = [0.3, 0.2, 0.8, 0, 0.2, 0.2, 0.6, 0.9, 0.9, 1]$\setminus n$
 timeleft = 120$\setminus n$domain = [0, 1, 2, 3, 4, 5, 6, 7' additionalkwargs={} responsemetadata={} id='run--e885c838-0e79-49fb-97c0-0a4218e6aa5b-0
\end{promptbox}

\begin{promptbox}{Correct Response}
\small
\textbf{System.} You got this information which is a list of values on format (i,fidelity,cpu time): $\{$\emph{model info}$\}$. This indicates the model i has a certain fidelity with respect to other models (the higher the better) and takes a certain time to find the solution. Our goal is to discover a function through sampling. Since now, the evaluated points are $\{$\emph{eval points}$\}$ with a distance to the optimal solution of $\{$\emph{error}$\}$ over the domain $\{$\emph{domain}$\}$ the uncertainty is $\{$\emph{uncertainty}$\}$. Please, select which should be the next point in the domain to be evaluated and which fidelity model should be used considering only $\{$\emph{remaining time}$\}$ is left. Strictly return the solution as RESULT=[model,point].

\textbf{User.} RESULT=[1,5.6892]
\end{promptbox}

\newpage
\section{Learned Prompt Projection}
\label{appendix:learned_projection}

The projection operator in~\eqref{eqn:projection_operator} need not be
rule-based. In more general settings, $\mathcal{P}_{\mathcal{O}}$ may be
parametrized by the oracle or by a learned controller,
$$\mathcal{P}_{\mathcal{O}}(\mathsf{S}_t,\phi)=\widetilde{\mathsf{S}}_t,\qquad\widetilde{\mathsf{S}}_t\in\mathcal{D},$$
where $\phi$ denotes the parameters of the projection policy. The goal is to learn a compression or restructuring map that preserves task-relevant information while keeping the projected prompt inside the feasible domain of the adaptive model.

To penalize infeasible projections, we define the domain regularizer
\begin{equation}
\label{eq:domain-loss}
\mathcal{L}_{\mathrm{dom}}(\phi)=\mathbb{E}_{\mathsf{S}_t}\left[\psi\!\left(g(\mathcal{P}_{\mathcal{O}}(\mathsf{S}_t,\phi))\right)\right],
\end{equation}
where $g(\cdot)\le 0$ defines prompt feasibility and $\psi$ is a nonnegative barrier function, such as a soft-plus barrier, that penalizes projected states outside the feasible domain. This gives the regularized objective
\begin{equation}
\label{eqn:main_optimization_problem_2}
\max_{\theta,\phi}\; J(\theta)-\lambda_d\mathcal{L}_{\mathrm{dom}}(\phi),
\end{equation}
which combines semantic adaptation with prompt-domain feasibility.

This learned version is useful when the feasible prompt representation cannot be specified by a simple rule, for example in open-ended tool-use environments or tasks where semantic compression depends on latent structure. In the controlled MFBO setting studied in this paper, however, the accumulated sequence has an explicit structure and the relevant information can be summarized through
representative evaluations and domain partitions. We therefore use the greedy rule-based projection described in Section~\ref{sec:hierarchy}, which is lightweight, interpretable, and supported by the string-submodular diminishing-returns structure of the information objective.

\begin{rem}[Rule-based and learned projection]
\label{rem:projection}
The projection operator in~\eqref{eqn:projection_operator} may be implemented through prompt truncation, semantic filtering, clustering, partitioning, structured compression, or a learned projection policy. The main framework only requires that the projected sequence belongs to the feasible domain $\mathcal{D}$ of the adaptive model.
\end{rem}

We now state sufficient conditions under which the regularized formulation in \eqref{eqn:main_optimization_problem_2} admits standard stationarity and feasibility-control guarantees.

\begin{assump}[Smoothness]
\label{assump:smoothness}
For any fixed $\phi$, the map $\theta\mapsto \widetilde{J}(\theta,\phi)$ is differentiable and $L_\theta$-smooth:
$$\|\nabla_\theta \widetilde{J}(\theta',\phi)-\nabla_\theta \widetilde{J}(\theta,\phi)\|\le L_\theta\|\theta'-\theta\|.$$
Similarly, for any fixed $\theta$, the map $\phi\mapsto -\lambda_d\mathcal{L}_{\mathrm{dom}}(\phi)$ is $L_\phi$-smooth.
\end{assump}

\begin{assump}[Bounded stochastic gradients]
\label{assump:bounded}
Let $g_k$ be the mini-batch gradient estimator used to update $\theta$. Then
$$\mathbb{E}[g_k\mid \theta_k,\phi_k]=\nabla_\theta \widetilde{J}(\theta_k,\phi_k),$$
and
$$\mathbb{E}\!\left[\|g_k-\nabla_\theta \widetilde{J}(\theta_k,\phi_k)\|^2\mid \theta_k,\phi_k\right]\le\sigma^2.$$
\end{assump}

\begin{assump}[Step sizes]
\label{assump:step_size}
The updates use step sizes $\eta_k>0$ and $\xi_k>0$ satisfying
$$\sum_{k=0}^{\infty}\eta_k=\infty,\qquad\sum_{k=0}^{\infty}\eta_k^2<\infty,\qquad\sum_{k=0}^{\infty}\xi_k=\infty,\qquad\sum_{k=0}^{\infty}\xi_k^2<\infty.$$
\end{assump}

\begin{assump}[Continuous relaxation of learned projection]
\label{assump:projection}
The learned projection is analyzed in a continuous prompt-embedding space. Each sequence $\mathsf{S}$ is represented by $E(\mathsf{S})\in\mathbb{R}^{n\times d_e}$, and the relaxed projection is a differentiable map
$\widehat{E}_t(\phi) =\mathcal{P}_{\mathcal{O}}^{\mathrm{soft}}(E(\mathsf{S}_t),\phi) \in\mathbb{R}^{k\times d_e}.$
For each fixed $\mathsf{S}_t$, there exists $L_P>0$ such that, for all $\phi,\phi'\in\Phi$,
$$
    \left\|
    \mathcal{P}_{\mathcal{O}}^{\mathrm{soft}}(E(\mathsf{S}_t),\phi')
    -
    \mathcal{P}_{\mathcal{O}}^{\mathrm{soft}}(E(\mathsf{S}_t),\phi)
    \right\|
    \leq
    L_P\|\phi'-\phi\|.
$$
Any hard sequence projection used at inference is obtained only after this relaxed step, e.g., by decoding, rounding, or top-$k$ extraction.
\end{assump}

\begin{assump}[Barrier regularity]
\label{assump:barrier}
The barrier function $\psi:\mathbb{R}\rightarrow\mathbb{R}_{\geq 0}$ is
nonnegative, nondecreasing, and convex, with $\psi(z)=0$ for $z\leq 0$.
Moreover, $\psi$ is strictly increasing on $z>0$, so that its generalized
inverse $\psi^{-1}$ is nondecreasing and concave on $\mathbb{R}_{\geq 0}$.
\end{assump}

After introducing the auxiliary assumptions, let us proceed to state the convergence with the following propositions. 

\begin{prop}[Stationarity of adaptive updates]
\label{prop:student_stationarity}
Fix $\phi$ and define
$$\widetilde{J}(\theta,\phi)=\mathbb{E}_{\rho\sim\mathcal{D}}\mathbb{E}_{y\sim\pi_{\theta}(\cdot\mid\rho)}[R(\rho,y)]-\lambda_d\mathcal{L}_{\mathrm{dom}}(\phi).$$
Under Assumptions~\ref{assump:smoothness}--\ref{assump:step_size},
$$\liminf_{k\to\infty}\mathbb{E}\left[\|\nabla_\theta \widetilde{J}(\theta_k,\phi)\|^2\right]=0.$$
\end{prop}

\begin{proof}
Fix $k$ and condition on $\theta_k$, with $\phi$ fixed. By
$L_\theta$-smoothness of $\widetilde{J}(\cdot,\phi)$,
$$\widetilde{J}(\theta_{k+1},\phi)\ge\widetilde{J}(\theta_k,\phi)+\langle\nabla_\theta \widetilde{J}(\theta_k,\phi),\theta_{k+1}-\theta_k\rangle-\frac{L_\theta}{2}\|\theta_{k+1}-\theta_k\|^2.$$
Substituting $\theta_{k+1}-\theta_k=\eta_k g_k$ and taking conditional expectation gives
$$\mathbb{E}[\widetilde{J}(\theta_{k+1},\phi)\mid \theta_k]\ge\widetilde{J}(\theta_k,\phi)+\eta_k\|\nabla_\theta \widetilde{J}(\theta_k,\phi)\|^2-\frac{L_\theta\eta_k^2}{2}\mathbb{E}[\|g_k\|^2\mid\theta_k].$$

Using Assumption~\ref{assump:bounded},
$$\mathbb{E}[\|g_k\|^2\mid\theta_k]\le2\|\nabla_\theta \widetilde{J}(\theta_k,\phi)\|^2+2\sigma^2.$$
Therefore,
$$\mathbb{E}[\widetilde{J}(\theta_{k+1},\phi)\mid \theta_k]\ge\widetilde{J}(\theta_k,\phi)+\eta_k\left(1-L_\theta\eta_k\right)\|\nabla_\theta \widetilde{J}(\theta_k,\phi)\|^2
-L_\theta\eta_k^2\sigma^2.$$
For sufficiently large $k$, $1-L_\theta\eta_k>0$. Taking total expectations and summing over $k$ gives
$$\sum_{k=0}^{\infty}\eta_k\mathbb{E}\|\nabla_\theta \widetilde{J}(\theta_k,\phi)\|^2<\infty,$$
because $\sum_k\eta_k^2<\infty$ and $\widetilde{J}(\theta_k,\phi)$ is bounded above along the iterates. Since $\sum_k\eta_k=\infty$, the claim follows.
\end{proof}

\begin{prop}[Bounded drift of relaxed learned projections]
\label{prop:projection_drift}
Let $ \widehat{E}_t(\phi)= \mathcal{P}_{\mathcal{O}}^{\mathrm{soft}}(E(\mathsf{S}_t),\phi)$ be the continuous relaxed projection of the encountered sequence
$\mathsf{S}_t$. Suppose Assumption~\ref{assump:projection} holds and consider stochastic descent updates on the domain penalty,
$$
    \phi_{k+1}
    =
    \phi_k
    -
    \xi_k\lambda_d
    \widehat{\nabla_\phi \mathcal{L}_{\mathrm{dom}}(\phi_k)},
$$
where
$$
    \mathbb{E}\left[
    \left\|
    \widehat{\nabla_\phi \mathcal{L}_{\mathrm{dom}}(\phi_k)}
    \right\|^2
    \right]
    \leq
    G_\phi^2.
$$
Then, for any fixed encountered sequence $\mathsf{S}_t$,
$$
    \mathbb{E}\left[
    \left\|
    \widehat{E}_t(\phi_{k+1})
    -
    \widehat{E}_t(\phi_k)
    \right\|
    \right]
    \leq
    L_P\lambda_dG_\phi\xi_k.
$$
Consequently,
$$
    \mathbb{E}\left[
    \left\|
    \widehat{E}_t(\phi_{k+1})
    -
    \widehat{E}_t(\phi_k)
    \right\|
    \right]
    \to 0
$$
whenever $\xi_k\to 0$.
\end{prop}

\begin{proof}
By Lipschitz continuity of the relaxed projection
$\mathcal{P}_{\mathcal{O}}^{\mathrm{soft}}$ in $\phi$,
$$
    \left\|
    \widehat{E}_t(\phi_{k+1})
    -
    \widehat{E}_t(\phi_k)
    \right\|
    \leq
    L_P\|\phi_{k+1}-\phi_k\|.
$$
Using the update rule,
$$
    \|\phi_{k+1}-\phi_k\|
    =
    \lambda_d\xi_k
    \left\|
    \widehat{\nabla_\phi \mathcal{L}_{\mathrm{dom}}(\phi_k)}
    \right\|.
$$
Taking expectations and applying Cauchy--Schwarz gives
$$
    \mathbb{E}\|\phi_{k+1}-\phi_k\|
    \leq
    \lambda_d\xi_k
    \sqrt{
    \mathbb{E}\left[
    \left\|
    \widehat{\nabla_\phi \mathcal{L}_{\mathrm{dom}}(\phi_k)}
    \right\|^2
    \right]
    }
    \leq
    \lambda_dG_\phi\xi_k.
$$
Combining the two inequalities yields
$$
    \mathbb{E}\left[
    \left\|
    \widehat{E}_t(\phi_{k+1})
    -
    \widehat{E}_t(\phi_k)
    \right\|
    \right]
    \leq
    L_P\lambda_dG_\phi\xi_k.
$$
Since $\xi_k\to 0$, the expected drift of the relaxed projected prompt
vanishes.
\end{proof}

\begin{prop}[Controlled feasibility violation]
\label{prop:feasibility_control}
Let Assumption~\ref{assump:barrier} hold and let $\psi^{-1}$ denote the
generalized inverse of $\psi$ on $\mathbb{R}_+$. If, for some $\delta\geq 0$,
$$
    \sup_k \mathcal{L}_{\mathrm{dom}}(\phi_k)\leq \delta,
$$
then, for all $k$,
$$
    \mathbb{E}\left[
    \max\{0,g(\widehat{\mathsf{S}}_t(\phi_k))\}
    \right]
    \leq
    \psi^{-1}(\delta).
$$
\end{prop}

\begin{proof}
Since $\psi$ is nondecreasing and $\psi(z)=0$ for $z\leq 0$, for any scalar
$u$ we have
$$
    \psi(u)=\psi(\max\{0,u\}).
$$
Moreover, by monotonicity of the generalized inverse,
$$
    \max\{0,u\}
    \leq
    \psi^{-1}(\psi(\max\{0,u\}))
    =
    \psi^{-1}(\psi(u)).
$$
Applying this inequality with $u=g(\widehat{\mathsf{S}}_t(\phi_k))$
gives
$$
    \max\{0,g(\widehat{\mathsf{S}}_t(\phi_k))\}
    \leq
    \psi^{-1}\!\left(
    \psi(g(\widehat{\mathsf{S}}_t(\phi_k)))
    \right).
$$
Taking expectations yields
$$
    \mathbb{E}\left[
    \max\{0,g(\widehat{\mathsf{S}}_t(\phi_k))\}
    \right]
    \leq
    \mathbb{E}\left[
    \psi^{-1}\!\left(
    \psi(g(\widehat{\mathsf{S}}_t(\phi_k)))
    \right)
    \right].
$$
By Assumption~\ref{assump:barrier}, $\psi^{-1}$ is concave and nondecreasing on
$\mathbb{R}_{\geq 0}$. Therefore, Jensen's inequality gives
$$
    \mathbb{E}\left[
    \psi^{-1}\!\left(
    \psi(g(\widehat{\mathsf{S}}_t(\phi_k)))
    \right)
    \right]
    \leq
    \psi^{-1}\!\left(
    \mathbb{E}\left[
    \psi(g(\widehat{\mathsf{S}}_t(\phi_k)))
    \right]
    \right).
$$
Using the definition of the domain penalty,
$$
    \mathcal{L}_{\mathrm{dom}}(\phi_k)
    =
    \mathbb{E}\left[
    \psi(g(\widehat{\mathsf{S}}_t(\phi_k)))
    \right],
$$
we obtain
$$
    \mathbb{E}\left[
    \max\{0,g(\widehat{\mathsf{S}}_t(\phi_k))\}
    \right]
    \leq
    \psi^{-1}\!\left(
    \mathcal{L}_{\mathrm{dom}}(\phi_k)
    \right).
$$
Finally, since $\sup_k \mathcal{L}_{\mathrm{dom}}(\phi_k)\leq\delta$ and
$\psi^{-1}$ is nondecreasing,
$$
    \psi^{-1}\!\left(
    \mathcal{L}_{\mathrm{dom}}(\phi_k)
    \right)
    \leq
    \psi^{-1}(\delta).
$$
This proves the claim.
\end{proof}

\newpage
\section{Proofs}
\label{appendix:proofs}

\subsection{Proof of Proposition~\ref{prop:equivalence_loss}}

\begin{proof}
Let $a\in\mathbb{V}$ denote a token, let $\pi_\theta$ be the adaptive policy, and let $\bar{\pi}_\phi$ be the protocol-reweighted oracle policy in~\eqref{eq:imitation-loss}. For a fixed pair $(\mathcal{S}_t,\mathcal{S}_{t+1})$, define
$$
    p(a)=\bar{\pi}_{\phi}(a|\mathcal{S}_t),
    \qquad
    q_\theta(a)=\pi_{\theta}(a|\mathcal{S}_{t+1}).
$$
Then
$$
    \mathcal{L}_{\mathrm{dis}}(\theta,\mathcal{B})
    =
    \mathbb{E}_{(\mathcal{S}_t,\mathcal{S}_{t+1})\sim\mathcal{B}}
    \left[
    D_{\mathrm{KL}}(p\,\|\,q_\theta)
    \right].
$$
For each fixed training pair,
$$
    D_{\mathrm{KL}}(p\,\|\,q_\theta)
    =
    \sum_{a\in\mathbb{V}}p(a)\log p(a)
    -
    \sum_{a\in\mathbb{V}}p(a)\log q_\theta(a).
$$
The first term is independent of $\theta$, and therefore minimizing
$\mathcal{L}_{\mathrm{dis}}$ is equivalent to maximizing
$$
    \mathbb{E}_{(\mathcal{S}_t,\mathcal{S}_{t+1})\sim\mathcal{B}}
    \mathbb{E}_{a\sim\bar{\pi}_{\phi}(\cdot|\mathcal{S}_t)}
    \left[
    \log \pi_{\theta}(a|\mathcal{S}_{t+1})
    \right].
$$
Thus, the proposed loss trains $\pi_\theta$ to assign high likelihood to the protocol-weighted oracle behavior. Since the case $\beta=1$ in \eqref{eqn:main_optimiziation_problem} isolates protocol alignment, this gives a supervised surrogate equivalent for the schema component of~\eqref{eqn:main_optimiziation_problem}.
\end{proof}

\subsection{Auxiliary Lemmas for Theorem~\ref{thm:logit_saturation}}

\begin{lem}[Causal attention dilution under prompt growth]
\label{lem:attention_dilution}
Let $\mathsf{S}_{1:\ell}=(s_1,\ldots,s_\ell)\in\mathbb{V}^*$ be a feasible prompt sequence induced by the greedoid $(\mathbb{V},\mathcal{I})$, i.e., $\mathsf{S}_{1:\ell}\in\mathcal{I}$. Consider an autoregressive self-attention layer with causal masking, and let $t=t(\ell)$ be a query position satisfying $t(\ell)\leq \ell,~t(\ell)\geq c\ell$
for some constant $c\in(0,1]$. Let $\alpha_{ti}^{(\ell)}$ denote the causal self-attention weight assigned from query position $t(\ell)$ to token $s_i$, for $i\leq t(\ell)$. If the attention scores are uniformly bounded and $\mathsf{R}\subseteq\mathsf{S}_{1:t(\ell)}$ is a fixed task-relevant subsequence with $|\mathsf{R}|=m$, then
$$
    \lim_{\ell\to\infty}
    \sum_{s_i\in\mathsf{R}}\alpha_{ti}^{(\ell)}
    =
    0.
$$
\end{lem}

\begin{proof}
For a query position $t=t(\ell)$, causal masking allows attention only to keys in the prefix $\{1,\ldots,t(\ell)\}$. Hence, the self-attention weight assigned to token $s_i$, with $i\leq t(\ell)$, is
\begin{equation}
    \label{eqn::attention}
    \alpha_{ti}^{(\ell)}
    =
    \frac{\exp\left(q_t^\top k_i/\sqrt{d}\right)}
    {\sum_{j=1}^{t(\ell)}\exp\left(q_t^\top k_j/\sqrt{d}\right)}.
\end{equation}
Let $\sigma_{tj}=\frac{q_t^\top k_j}{\sqrt{d}}$ denote the attention score. By the bounded-score assumption, there exist constants $a,b\in\mathbb{R}$, with $a\leq b$, such that
$$ \sigma_{tj}\in[a,b],
    \qquad
    \forall j\in\{1,\ldots,t(\ell)\}.
$$
Therefore,
$$
    e^a\leq e^{\sigma_{tj}}\leq e^b,
    \qquad
    \forall j\in\{1,\ldots,t(\ell)\}.
$$
The denominator in~\eqref{eqn::attention} satisfies
$$
    \sum_{j=1}^{t(\ell)} e^{\sigma_{tj}}
    \geq
    t(\ell)e^a.
$$
Thus, for every token $s_i$ with $i\leq t(\ell)$,
$$
    \alpha_{ti}^{(\ell)}
    =
    \frac{e^{\sigma_{ti}}}
    {\sum_{j=1}^{t(\ell)}e^{\sigma_{tj}}}
    \leq
    \frac{e^b}{t(\ell)e^a}
    =
    \frac{e^{b-a}}{t(\ell)}.
$$
Since $t(\ell)\geq c\ell$, we obtain
$$
    \alpha_{ti}^{(\ell)}
    \leq
    \frac{e^{b-a}}{c\ell}.
$$
Now let $\mathsf{R}\subseteq\mathsf{S}_{1:t(\ell)}$ be a fixed task-relevant subsequence with $|\mathsf{R}|=m$. Then,
$$
    \sum_{s_i\in\mathsf{R}}\alpha_{ti}^{(\ell)}
    \leq
    \sum_{s_i\in\mathsf{R}}
    \frac{e^{b-a}}{c\ell}
    =
    \frac{m e^{b-a}}{c\ell}.
$$
Since $m$ and $c>0$ are fixed,
$$
    \lim_{\ell\to\infty}
    \sum_{s_i\in\mathsf{R}}\alpha_{ti}^{(\ell)}
    \leq
    \lim_{\ell\to\infty}
    \frac{m e^{b-a}}{c\ell}
    =
    0.
$$
Because attention weights are nonnegative, the limit is exactly zero concluding the proof.
\end{proof}

\subsection{Proof of Theorem~\ref{thm:logit_saturation}}

\begin{proof}
Consider a self-attention block where the attention output at generation position $t$ is
$$z_t^{(\ell)}=\sum_{i=1}^{\ell}\alpha_{ti}^{(\ell)}v_i,$$
where $\alpha_{ti}^{(\ell)}$ is the attention weight assigned to token $s_i$, and $v_i$ is its corresponding value vector. We decompose the attention output into the contribution of the task-relevant subsequence $\mathsf{R}$ and its complement:
$$z_t^{(\ell)}=\underbrace{
    \sum_{s_i\in\mathsf{R}}\alpha_{ti}^{(\ell)}v_i }_{z_{t,\mathsf{R}}^{(\ell)}}+\underbrace{\sum_{s_i\notin\mathsf{R}}\alpha_{ti}^{(\ell)}v_i}_{z_{t,\mathsf{R}^c}^{(\ell)}},$$
such the hidden state at the generation position is passed through a linear output head to produce logits,
$$L_t^{(\ell)}=W_{\mathrm{out}}h_t^{(\ell)}+b_{\mathrm{out}}.$$
Then, locally isolating the contribution of $\mathsf{R}$ to the logits gives,
$$\Delta L_{\mathsf{R}}^{(\ell)}=W_{\mathrm{out}} z_{t,\mathsf{R}}^{(\ell)}=W_{\mathrm{out}}\sum_{s_i\in\mathsf{R}}\alpha_{ti}^{(\ell)}v_i.$$
Taking norms, and using submultiplicativity together with the triangle inequality, we obtain
$$\left\|\Delta L_{\mathsf{R}}^{(\ell)}\right\|=\left\|W_{\mathrm{out}}\sum_{s_i\in\mathsf{R}}\alpha_{ti}^{(\ell)}v_i\right\|\leq \|W_{\mathrm{out}}\|\sum_{s_i\in\mathsf{R}}\alpha_{ti}^{(\ell)}\|v_i\|,$$
which can be further developed by the assumed uniform boundedness of the value vectors $V>0$, $\|v_i\|\leq V,~\forall s_i\in\mathsf{R}$, such that,
$$\left\|\Delta L_{\mathsf{R}}^{(\ell)} \right\|\leq\|W_{\mathrm{out}}\|V\sum_{s_i\in\mathsf{R}}\alpha_{ti}^{(\ell)}.$$

Finally, by Lemma~\ref{lem:attention_dilution}, since the attention scores are uniformly bounded and $\mathsf{R}$ is fixed with $|\mathsf{R}|=m$, $\sum_{s_i\in\mathsf{R}}\alpha_{ti}^{(\ell)}\leq\frac{m e^{b-a}}{\ell}$. Therefore,
$$\left\| \Delta L_{\mathsf{R}}^{(\ell)} \right\|\leq\|W_{\mathrm{out}}\|V\frac{m e^{b-a}}{\ell}=\mathcal{O}\!\left(\frac{m}{\ell}\right).$$
If $m$ is fixed while $\ell$ grows, then
$$\lim_{\ell\to\infty}\|W_{\mathrm{out}}\|V\frac{m e^{b-a}}{\ell}=0.$$
Since the norm of the logit contribution is nonnegative, the squeeze argument implies that the logit-level influence of the fixed task-relevant subsequence $\mathsf{R}$ vanishes as the feasible prompt sequence grows, and this concludes the proof.
\end{proof}

\begin{rem}[Interpretation]
Theorem~\ref{thm:logit_saturation} does not imply that the full logit vector
$$L_t^{(\ell)}=W_{\mathrm{out}}h_t^{(\ell)}+b_{\mathrm{out}}$$
converges to zero. The full logits may remain nonzero because of residual connections, layer normalization, MLP transformations, biases, and the aggregate contribution of all prompt tokens. Rather, the theorem states that the logit-level contribution of any fixed task-relevant subsequence $\mathsf{R}$ vanishes as the feasible prompt length $L$ grows. In this sense, prompt saturation arises because the finite attention budget is spread over an increasingly long feasible sequence, causing the influence of task-critical information to decay as $\mathcal{O}(1/\ell)$.
\end{rem}

\subsection{Proof for Proposition~\ref{prop:attention_degradation_ratio}}

\begin{proof}
Let
$$\sigma_{tj}^{(\ell,h)}=\frac{q_t^{(\ell,h)\top}k_j^{(\ell,h)}}{\sqrt{d_h}}$$
denote the attention logit assigned by head $h$ in layer $\ell$ from query position $t$ to token $s_j$. By Cauchy--Schwarz,
$$\left|q_t^{(\ell,h)\top}k_j^{(\ell,h)}\right|\le\|q_t^{(\ell,h)}\|\,\|k_j^{(\ell,h)}\| \le Q_{\ell,h}K_{\ell,h}.$$
Therefore,
$$-\frac{Q_{\ell,h}K_{\ell,h}}{\sqrt{d_h}}\le\sigma_{tj}^{(\ell,h)}\le \frac{Q_{\ell,h}K_{\ell,h}}{\sqrt{d_h}}.$$
Hence, for each layer and head, the attention-logit spread satisfies
$b_{\ell,h}-a_{\ell,h}\le\frac{2Q_{\ell,h}K_{\ell,h}}{\sqrt{d_h}}.$
Then applying Lemma~\ref{lem:attention_dilution} to each layer and head gives
$$\sum_{s_i\in\mathsf{R}}\alpha_{ti}^{(\ell,h)}\le\frac{m e^{b_{\ell,h}-a_{\ell,h}}}{|\rho|}\le\frac{m}{|\rho|}\exp\left(\frac{2Q_{\ell,h}K_{\ell,h}}{\sqrt{d_h}}\right).$$
Then, averaging over all layers and heads preserves the bound with the worst-case layer-head pair:
$$A_{\mathcal{U}}(\rho):=\frac{1}{LH}\sum_{\ell=1}^{L}\sum_{h=1}^{H}\sum_{s_i\in\mathsf{R}}\alpha_{ti}^{(\ell,h)}\le\frac{m}{|\rho|}\exp\left(\max_{\ell,h}\frac{2Q_{\ell,h}K_{\ell,h}}{\sqrt{d_h}}\right).$$
Since feasibility requires $A_{\mathcal{U}}(\rho)\ge \tau$, any feasible prompt must satisfy $\tau\le\frac{m}{|\rho|}\exp\left(\max_{\ell,h}\frac{2Q_{\ell,h}K_{\ell,h}}{\sqrt{d_h}}\right),$ such that rearranging gives
$$|\rho|_{\mathrm{sat}}\le\frac{m}{\tau}\exp\left(\max_{\ell,h}\frac{2Q_{\ell,h}K_{\ell,h}}{\sqrt{d_h}}\right).$$
Finally, dividing by $|\rho|_{\mathrm{feasible}}$ yields
$$\Delta_{\mathcal{U}}=\frac{|\rho|_{\mathrm{sat}}}{|\rho|_{\mathrm{feasible}}}\le\frac{m}{\tau |\rho|_{\mathrm{feasible}}}\exp\left(\max_{\ell,h}\frac{2Q_{\ell,h}K_{\ell,h}}{\sqrt{d_h}}\right),$$
which concludes the proof.
\end{proof}

\newpage
\section{Multi-Fidelity Bayesian Maximization}
\label{appendix:MFBM}

We instantiate the proposed framework on a multi-fidelity Bayesian maximization (MFBM) problem. Let $\mathcal{X}\subseteq\mathbb{R}^d$ be a search domain and let
\begin{equation}
    \label{eqn:problem1}
    x^\star = \arg\max_{x\in\mathcal{X}} g(x)
\end{equation}
denote the maximizer of an unknown objective function $g:\mathcal{X}\rightarrow\mathbb{R}$. We assume a Gaussian process prior $g(x)\sim\mathcal{GP}\bigl(\mu(x),k(x,x')\bigr)$, where $\mu(\cdot)$ is the mean function and $k(\cdot,\cdot)$ is the covariance kernel.

Direct evaluations of $g(\cdot)$ are assumed to be expensive. Instead, the decision-maker has access to a collection of surrogate models $\mathcal{P}=\{p_1,\ldots,p_n\}$ with different fidelities and costs. Each surrogate $p_i$ is associated with a fidelity parameter $\alpha_i\in(0,1]$ and an evaluation cost $c_i(x)>0$. Larger $\alpha_i$ gives a more accurate but more expensive approximation of $g(x)$. We model the accuracy side of this tradeoff as
$$
    \mathbb{E}\!\left[|p_i(x)-g(x)|\right]\propto (1-\alpha_i),
$$
and associate each fidelity level with an explicit evaluation cost $c_i(x)>0$. In general, one may choose costs that increase with fidelity, for example $c_i(x)\propto \alpha_i^\gamma$ for some $\gamma>0$. In our experiments, however, we use the discrete execution costs specified below to match the MFBM environment used in Section~\ref{sec:case}.

At each iteration $t$, the decision is the pair
$$
    a_t=(i_t,x_t)\in\{1,\ldots,n\}\times\mathcal{X},
$$
where $i_t$ selects the surrogate model and $x_t$ is the next query point. Querying $(i_t,x_t)$ returns the observation $y_t=p_{i_t}(x_t)$. The goal is to allocate a finite budget $\mathcal{Q}$ across the available surrogates so as to approximate $x^\star$ while satisfying
\begin{equation}
    \label{eqn:mfbm_budget}
    \sum_{t=1}^{T} c_{i_t}(x_t)\le \mathcal{Q}.
\end{equation}

The history available at step $t$ is represented by the sequence
$ \mathsf{S}_t=\left(\{(x_\tau,y_\tau,i_\tau)\}_{\tau=1}^{t},\mathcal{Q}_t\right),$
where $\mathcal{Q}_t$ denotes the remaining budget. In the hierarchical framework, the oracle policy selects the next experiment according to
$(i_{t+1},x_{t+1})=\pi_{\phi}(\mathsf{S}_t),$
while the adaptive model acts on the projected sequence and predicts
$(\widehat{i}_{t+1},\widehat{x}_{t+1})
    =\pi_{\theta}(\widehat{\mathsf{S}}_t), ~\widehat{\mathsf{S}}_t=\mathcal{P}_{\mathcal{O}}(\mathsf{S}_t).$
In our implementation, this corresponds to the structured response $\texttt{RESULT=[model,point]}$, where \texttt{model} identifies the selected fidelity and \texttt{point} denotes the next query location.

This setting is suitable for evaluating the proposed framework for three reasons. First, the optimization process is sequential, so the prompt state grows with each iteration. Second, the decision must jointly determine \emph{where} to sample and \emph{at which fidelity}, which naturally induces structured outputs. Third, the finite budget makes the problem resource-constrained, matching the central setting of the paper.

For the experiments in Section~\ref{sec:case}, we consider the one-dimensional case $d=1$ and simulate $g(x)$ as a smooth multimodal Gaussian mixture over a bounded interval. The surrogate models $p_i(x)$ are constructed by adding Gaussian perturbations whose variance is inversely related to fidelity, so that lower-fidelity models are cheaper but noisier. We use four fidelity levels
$$
    \alpha_i \in \{0.25,\,0.5,\,0.75,\,1.0\},
$$
and assign them the execution costs
$c_i \in \{1,\,2,\,3,\,4\}\ \text{minutes}.$
Thus, $\alpha_i$ controls approximation quality, while $c_i$ is the explicit wall-clock cost used by the budget constraint in~\eqref{eqn:mfbm_budget}. The remaining budget $\mathcal{Q}_t$ is measured in minutes and decreases according to the cost of the selected fidelity model. In the experiments, the total budget is chosen large enough to allow long sequential trajectories, which is why the oracle-only baselines in Table~\ref{tab:teacher-student-ablation} can evaluate more than one hundred points. Therefore, the budget is not normalized to unity in the reported experiments.

Thus, MFBM provides a controlled sequential decision-making testbed in which structured outputs, iterative prompt growth, and resource-aware choices arise simultaneously.

\begin{figure}[t]
    \centering
    \begin{tikzpicture}
        \begin{axis}[
            width=0.9\linewidth,
            height=5.8cm,
            xmin=0, xmax=10,
            ymin=-0.2, ymax=1.8,
            xlabel={$x$},
            ylabel={Objective value},
            legend style={
                at={(0.02,0.98)},
                anchor=north west,
                font=\scriptsize,
                draw=none,
                fill=none
            },
            ytick=\empty,
            grid=major,
            domain=0:10,
            samples=200,
            thick
        ]

        \addplot[black, very thick]
        {0.95*exp(-((x-2.0)^2)/0.45)
        +1.35*exp(-((x-5.8)^2)/1.1)
        +0.85*exp(-((x-8.2)^2)/0.35)
        +0.08*sin(deg(2*x))};
        \addlegendentry{$g(x)$}

        \addplot[red, dashed]
        {0.95*exp(-((x-2.1)^2)/0.70)
        +1.10*exp(-((x-5.5)^2)/1.50)
        +0.65*exp(-((x-8.0)^2)/0.60)
        +0.16*sin(deg(1.5*x+0.4))};
        \addlegendentry{$p_1(x),\ \alpha=0.25$}

        \addplot[orange!85!black, dashdotted]
        {0.95*exp(-((x-2.0)^2)/0.60)
        +1.18*exp(-((x-5.6)^2)/1.35)
        +0.72*exp(-((x-8.1)^2)/0.50)
        +0.12*sin(deg(1.7*x+0.2))};
        \addlegendentry{$p_2(x),\ \alpha=0.5$}

        \addplot[green!50!black, densely dashed]
        {0.95*exp(-((x-2.0)^2)/0.52)
        +1.26*exp(-((x-5.7)^2)/1.20)
        +0.79*exp(-((x-8.15)^2)/0.42)
        +0.09*sin(deg(1.85*x+0.1))};
        \addlegendentry{$p_3(x),\ \alpha=0.75$}

        \addplot[blue, dotted, very thick]
        {0.95*exp(-((x-2.0)^2)/0.47)
        +1.32*exp(-((x-5.75)^2)/1.12)
        +0.83*exp(-((x-8.2)^2)/0.37)
        +0.05*sin(deg(1.95*x))};
        \addlegendentry{$p_4(x),\ \alpha=1.0$}

        \end{axis}
    \end{tikzpicture}
    \caption{Illustration of the multi-fidelity Bayesian maximization setup.
    The black curve denotes the unknown target objective $g(x)$, while the
    colored curves represent surrogate models of increasing fidelity and cost.
    Higher-fidelity surrogates provide more accurate approximations at a larger
    evaluation cost.}
    \label{fig:mfbm_illustration}
\end{figure}
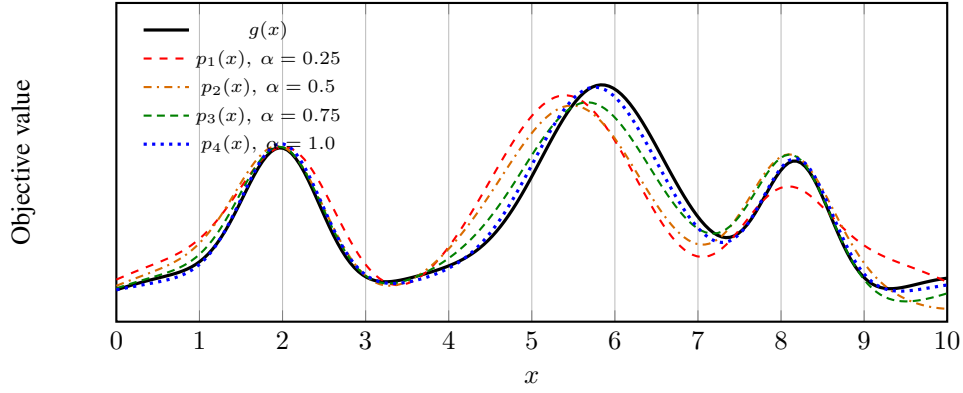
\FloatBarrier

\newpage
\section{Experimental Setting: \texttt{URSA}, MFBO, and Reproducibility}
\label{app:ursa_mfbo}

All experiments are implemented in \texttt{URSA} (Universal Research and Scientific Agent), which provides a modular execution layer for agentic scientific workflows. In this paper, \texttt{URSA} is used to instantiate the oracle--student--controller architecture in a Multi-Fidelity Bayesian Optimization (MFBO) environment.

\paragraph{URSA implementation.}
\texttt{URSA} defines a common \texttt{BaseAgent} abstraction. Each agent implements an internal model or tool, a \texttt{forward()} method, and optional state-management utilities. A minimal implementation is:
\begin{verbatim}
from ursa.agents.base_agent import BaseAgent

class MFBOAgent(BaseAgent):
    def forward(self, state):
        return self.model.generate(state)
\end{verbatim}
The framework can be installed as:
\begin{verbatim}
git clone https://github.com/lanl/ursa.git
cd ursa
pip install -e .
\end{verbatim}

\paragraph{Architecture.}
We instantiate three components. The \emph{oracle} $\mathcal{O}_\phi$ provides reference MFBO decisions and is implemented using \emph{GPT-5} or \emph{GPT-5-nano}. The \emph{adaptive student} $\mathcal{U}_\theta$ acts in the forward path and is implemented using \emph{Llama-3.1-8B} or \emph{Mistral-7B}. The \emph{controller} $\mathcal{C}$ monitors protocol validity, semantic drift, prompt length, and saturation, and applies projection or oracle-supervised correction when needed. The deployment loop is
$$\mathsf{S}_t\;\xrightarrow{\mathcal{C}}\;\widetilde{\mathsf{S}}_t\;\xrightarrow{\mathcal{U}_\theta}\;y_t\;\xrightarrow{\mathcal{O}_\phi/\mathcal{C}}\;e_t,$$
where $\mathsf{S}_t$ is the accumulated state, $\widetilde{\mathsf{S}}_t$ is the projected feasible state, $y_t$ is the student decision, and $e_t$ is the drift or semantic error signal.

\paragraph{Practical implementation of oracle losses.}
Equations~\eqref{eq:imitation-loss} and~\eqref{eq:consistency-loss} are written as KL objectives for notational clarity. In the experiments, \emph{GPT-5} and \emph{GPT-5-nano} act only as black-box oracle teachers and are never trained. Since these API models do not expose full vocabulary logits, we do not compute the exact oracle distribution or the exact KL divergence. Instead, each oracle call produces a target completion, which is then used as supervised data for the local student model. For schema distillation, we identify the protocol-defining tokens in the oracle output, such as \texttt{RESULT}, brackets, commas, and delimiters, and apply larger cross-entropy weights to those positions. For semantic adaptation, we weight the tokens corresponding to the selected model index and query point. Thus, the practical training objective is a weighted teacher-forced negative log-likelihood over sampled oracle outputs, where the weights separate schema-relevant and semantic-relevant parts of the oracle response.

\paragraph{MFBO environment.}
The task is to maximize an unknown function
$$x^\star = \arg\max_{x\in X} g(x),$$ 
using a finite set of fidelity models $\mathcal{P}=\{p_1,\ldots,p_n\}.$ Each model $p_i$ has fidelity $\alpha_i\in(0,1]$ and cost $c_i$. At each iteration, the agent observes
$$\mathcal{S}_t =(\texttt{model\_info},
\texttt{eval\_points},
\texttt{error},
\texttt{domain},
\texttt{uncertainty},
\texttt{remaining\_time}),$$

and must select both a fidelity model and the next point to evaluate. The required protocol is $\texttt{RESULT=[model,point]}.$
Thus, the environment tests both schema validity and semantic decision quality under a finite budget.

\paragraph{Prompt template.}
The base prompt is:
\begin{verbatim}
System. You got this information which is a list of values on format
(i,fidelity,cpu time): {model info}. This indicates the model i has a
certain fidelity with respect to other models and takes a certain time
to find the solution. Our goal is to discover a function through
sampling. Since now, the evaluated points are {eval points} with a
distance to the optimal solution of {error} over the domain {domain};
the uncertainty is {uncertainty}. Select the next point to evaluate and
which fidelity model should be used considering only {remaining time}
is left. Strictly return the solution as RESULT=[model,point].
\end{verbatim}
A valid response is, for example,
\begin{verbatim}
RESULT=[1,5.6892]
\end{verbatim}

\paragraph{Learning stages.}
Phase~1 performs offline schema distillation. Given a state $\mathcal{S}_t$, the oracle generates a target decision, and the student is trained using
$$\mathcal{L}_{\mathrm{dis}}(\theta)=\mathbb{E}\left[D_{\mathrm{KL}}\left(\bar{\pi}_\phi(\cdot \mid \mathcal{S}_t)\;\|\;\pi_\theta(\cdot \mid \widetilde{\mathcal{S}}_t)\right)\right],$$
where $\bar{\pi}_\phi$ emphasizes protocol-defining tokens.

Phase~$2$ performs online semantic adaptation. When the controller detects drift, it stores oracle--student pairs in a bounded buffer $\mathcal{T}$ and updates the student through
$$\mathcal{L}_{\mathrm{ft}}(\theta)=\mathbb{E}_{\mathcal{S}_t\sim\mathcal{T}}\left[D_{\mathrm{KL}}\left(\pi_\phi(\cdot \mid \mathcal{S}_t)\;\|\;\pi_\theta(\cdot \mid \widetilde{\mathcal{S}}_t)\right)\right].$$

\paragraph{Prompt projection.}
Because the MFBO state grows over time, the controller projects the accumulated state into the feasible prompt domain:
$$\mathcal{P}_{\mathcal{O}}:\mathsf{S}_t\mapsto \widetilde{\mathsf{S}}_t,\qquad\widetilde{\mathsf{S}}_t\in\mathcal{D}.$$
In the MFBO implementation, the history $\mathcal{H}_t=\{(x_i,e_i,\mathcal{D}_i,u_i)\}_{i=1}^{N_t}$ is partitioned into intervals $\{I_k\}_{k=1}^K$ and summarized using $\bar e_k=\mathbb{E}_{x_i\in I_k}[e_i],~\bar u_k=\mathbb{E}_{x_i\in I_k}[u_i].$ This keeps the prompt compact while preserving the information needed for the next fidelity--point decision.

\paragraph{Metrics and baselines.}
We report schema accuracy, fidelity-model accuracy, point-selection error, distance to optimum, number of evaluated points, intervention frequency, and oracle cost. The proposed hierarchical method is compared against: \emph{oracle only}, \emph{distillation only}, \emph{no distillation}, and the full \emph{hierarchical} architecture.

\paragraph{Reproducibility.}
Each run logs
$$(\mathsf{S}_t,\widetilde{\mathsf{S}}_t,y_t,m_t,x_t,e_t,\texttt{valid}_t,\texttt{oracle\_called}_t),$$
which is sufficient to reconstruct the full MFBO trajectory. We fix random seeds, prompt templates, fidelity levels, evaluation costs, LoRA configurations, and the ordering of evaluated points. When multiple runs are reported, we provide means and standard deviations, and use matched MFBO traces across methods.

\paragraph{Compute.}
Local experiments use LoRA fine-tuning on the adaptive models. We train the modules
$$\texttt{q\_proj},\texttt{k\_proj},\texttt{v\_proj},\texttt{o\_proj},\texttt{gate\_proj},\texttt{up\_proj},\texttt{down\_proj},$$
with rank-$32$ adapters. This corresponds to approximately $90.18$M trainable parameters for \emph{Llama-3.1-8B} and $86.80$M for \emph{Mistral-7B}, about $1.1$--$1.2\%$ of total parameters. The local experiments are run under a constrained NVIDIA Tesla V$100$ $16$GB environment, while oracle models are accessed through API calls.

\newpage
\section{Projection beyond Numerical MFBO}
\label{appendix:projection_beyond_mfbo}

The projection mechanism used in the MFBO experiments is rule-based because the
state has an explicit numerical structure. However, the same principle extends
beyond numerical prompts. Given any prompt or accumulated state $\rho$, we
decompose it into information units
$$\rho=(u_1,\ldots,u_n),$$
where each $u_i$ may be a sentence, tool output, constraint, observation,
decision, table row, or numerical evaluation. Let $e:\mathcal{U}\rightarrow \mathbb{R}^d$ be an encoder and define $z_i=e(u_i)$. Thus, every prompt induces a finite set of embedded units $Z(\rho)=\{z_1,\ldots,z_n\}\subset\mathbb{R}^d.$

We separate the embedded units into schema-defining and summarizable components,
$$Z(\rho)=Z_{\mathrm{sch}}(\rho)\cup Z_{\mathrm{sum}}(\rho).$$
The set $Z_{\mathrm{sch}}(\rho)$ contains units that define the required output
schema, protocol, constraints, or tool interface, and must therefore be preserved.
The set $Z_{\mathrm{sum}}(\rho)$ contains semantic or historical units that may be
compressed. Projection keeps the schema component fixed and compresses only the
summarizable component:
  $$ \mathcal{P}_{\mathcal{O}}(\rho)
    =
    Z_{\mathrm{sch}}(\rho)
    \cup
    \mathcal{S}^{\star},
    \qquad
    \mathcal{S}^{\star}\subseteq Z_{\mathrm{sum}}(\rho),
    \qquad
    |\mathcal{S}^{\star}|\le \kappa,$$
where $\kappa$ is defined as the difference between the maximum allowed prompt length $|\rho|_\text{saturation}$ (Table~\ref{tab:prompt_saturation_summary}) and the length consumed by the schema units $|Z_\text{sch}|$. Let $d:\mathbb{R}^d\times\mathbb{R}^d\rightarrow\mathbb{R}_{+}$ be a dissimilarity
measure between embedded units. We define the value of a candidate summary
$\mathcal{S}\subseteq Z_{\mathrm{sum}}(\rho)$ by
\begin{equation}
\label{eq:projection_distance_objective}
    f(\mathcal{S})
    =
    d_{e_0}
    -
    \sum_{z\in Z_{\mathrm{sum}}(\rho)}
    \min_{s\in\mathcal{S}} d(z,s),
\end{equation}
where $d_{e_0}$ is a \emph{phantom} constant chosen large enough so that
$f(\mathcal{S})\ge 0$ for all feasible $\mathcal{S}$. The projection problem is
then solving
\begin{equation}
\label{eq:projection_subset_problem}
    \mathcal{S}^{\star}
    \in
    \arg\max_{\mathcal{S}\subseteq Z_{\mathrm{sum}}(\rho)}
    f(\mathcal{S})
    \qquad
    \mathrm{s.t.}
    \qquad
    |\mathcal{S}|\le \kappa.
\end{equation}
Equivalently, maximizing $f$ minimizes the total distance from each summarizable
unit to its nearest selected exemplar:
\[
    \mathcal{S}^{\star}
    \in
    \arg\min_{\mathcal{S}\subseteq Z_{\mathrm{sum}}(\rho),\,|\mathcal{S}|\le\kappa}
    \sum_{z\in Z_{\mathrm{sum}}(\rho)}
    \min_{s\in\mathcal{S}} d(z,s).
\]
Thus, the selected set $\mathcal{S}^{\star}$ acts as a collection of medoids that
covers the semantic content of the original prompt while reducing its length.

Under standard nonnegative distance assumptions, the function in
\eqref{eq:projection_distance_objective} is monotone submodular after the
phantom normalization. Intuitively, adding a new exemplar can only decrease the
nearest-exemplar distance, and the marginal improvement is smaller when more
exemplars have already been selected. Therefore, the projected summary obtained
by greedy selection is suboptimal with the approximation guarantee stated in
Proposition~\ref{prop:greedy_info_summary}. This gives a theoretically supported
projection mechanism: schema vectors remain fixed, while semantic vectors are
compressed through a submodular exemplar-selection problem.

The compressed prompt is obtained by mapping the selected vectors back to their
corresponding information units:
\[
    \widetilde{\rho}
    =
    \rho_{\mathrm{sch}}
    \circ
    (u_i: z_i\in\mathcal{S}^{\star}).
\]
Consequently, the MFBO projection is not a task-specific exception, but one
instance of a broader projection framework applicable to text-heavy agentic
tasks, tool-use histories, active learning traces, experimental-design logs, and
simulation-based optimization states.

Alternative projection mechanisms are also possible. Instead of selecting
exemplars, one may aggregate the summarizable vectors into one or more learned
representatives, for example through pooling, attention, or a trainable
compression map. This is the perspective taken in
Appendix~\ref{appendix:learned_projection}, where projection is treated as a
learned state-compression operator. Such learned projection can be more flexible,
but it does not provide the same direct submodular structure and approximation
guarantees as the exemplar-based projection above.

\newpage
\section{Limitations and Broader Impact}
\label{app:limitations_broader_impact}

\paragraph{Scope of the experimental study.}
This work is intended as a fundamental study of prompt-domain control and hierarchical supervision for resource-constrained agentic language models. Our goal is not to claim exhaustive validation across all possible agentic tasks, but to isolate a core failure mode that appears in sequential agentic deployment: accumulated interaction histories can push compact models outside their effective prompt domain, causing protocol failures, semantic drift, or inefficient reliance on oracle models. We use Multi-Fidelity Bayesian Optimization (MFBO) because it provides a clear and controlled setting in which structured outputs, iterative prompt growth, numerical decision-making, budget constraints, and oracle--student supervision arise simultaneously.

Although MFBO is compact enough to support controlled experimentation, it should not be interpreted as a toy problem. MFBO naturally represents a broad class of sequential decision-making tasks in which an agent must choose both an action and a level of evaluation cost or fidelity. This includes experimental design, hyperparameter optimization, active learning, scientific discovery workflows, simulation-based design, model selection, and resource-aware planning. Therefore, while MFBO does not cover the full diversity of agentic AI environments, it captures several structural features that appear in realistic scientific and engineering deployments.

\paragraph{Limitations of empirical coverage.}
The main empirical limitation is that we evaluate the proposed framework on a single controlled family of tasks rather than on a broad benchmark suite of agentic environments. This choice allows us to measure prompt feasibility, projection behavior, oracle-intervention frequency, and cost-efficiency in a transparent setting, but it limits the extent to which the quantitative results can be directly generalized. Future work should evaluate the same architecture on more diverse agentic settings, including tool-use chains, code-generation agents, long-horizon planning tasks, multi-agent coordination, retrieval-heavy workflows, and interactive scientific assistants.

A second limitation is the number of adaptive models and oracle models considered. We evaluate two compact open-weight students and two oracle models, which is sufficient to show model-dependent differences in saturation and adaptation, but not enough to characterize all model families, context architectures, or instruction-tuning regimes. Larger student models, mixture-of-experts students, and domain-specialized models may exhibit different prompt-feasibility thresholds and different sensitivity to online fine-tuning, which is also left for future work.

\paragraph{Relation to \texttt{RAG} and non-hierarchical prompting.}
Retrieval-augmented generation and context-expansion methods are closely related to the non-hierarchical deployment regime considered in this paper. In standard MFBO deployment without projection, the prompt grows as more evaluations, errors, uncertainties, and intermediate results are appended to the state. This resembles a \texttt{RAG}-style or memory-augmented prompting pipeline in which more information is repeatedly inserted into the context. Our results suggest that this strategy can become unreliable for compact models: even when the prompt remains below the nominal context limit, the effective prompt domain may be exceeded due to saturation or attention dilution.

The proposed hierarchy does not reject retrieval or memory augmentation. Rather, it suggests that such methods should be coupled with feasibility-aware control. In other words, retrieval should not only decide what information is relevant, but also whether the resulting prompt remains inside the stable operating region of the deployed model.

\paragraph{Fine-tuning hyperparameter coverage.}
The adaptive fine-tuning experiments use a constrained set of LoRA hyperparameter combinations. This is intentional. The purpose of online adaptation in our framework is not to substantially shift the compact model or perform large-scale retraining, but to apply small corrective updates under limited time, data, and compute. We therefore restrict the search to lightweight updates that are compatible with deployment-time constraints. A larger hyperparameter sweep could improve individual model performance, but it would also weaken the intended setting: online correction must remain fast enough to be useful inside an active agentic loop.

Consequently, the reported results should be interpreted as evidence that small, feasible updates can improve reliability under hierarchical supervision, rather than as the best possible fine-tuning performance achievable for each model. Future work should study adaptive schedules for learning rate, LoRA rank, buffer size, number of epochs, and intervention frequency, especially under stricter latency constraints.

\paragraph{Measurement limitations.}
The empirical metrics focus on protocol validity, fidelity-model selection, point-selection error, distance to optimum, oracle-intervention frequency, and cost. These metrics are appropriate for MFBO, but they do not capture all possible dimensions of agentic reliability. For example, open-ended tasks may require measuring factuality, tool-use correctness, safety constraints, calibration, robustness to adversarial context, and user-facing quality. Moreover, the saturation thresholds measured in this work are empirical and model-specific; they should be recalibrated when changing the model, tokenizer, hardware precision, decoding strategy, or deployment environment.

\paragraph{Broader impact.}
The broader motivation of this work is to make agentic language-model systems more reliable, cost-aware, and deployable under resource constraints. By reducing unnecessary oracle calls and allowing compact models to act safely within a controlled prompt domain, hierarchical supervision may lower the cost and energy footprint of agentic systems. This is particularly relevant for scientific computing, engineering design, education, and resource-limited organizations that cannot rely on large proprietary models at every step.

The framework may also improve transparency. By separating schema validity, semantic adaptation, projection, and oracle intervention, the system produces interpretable logs of when and why the controller intervenes. This can help diagnose failure modes in sequential agentic systems and support more auditable deployment.

At the same time, more reliable agentic systems can also be misused. Improving the autonomy and efficiency of language-model agents may increase the capability of systems used for spam, automated manipulation, low-quality content generation, or unsafe decision support. In high-stakes domains, the proposed framework should therefore be combined with domain-specific safety constraints, access controls, human oversight, and careful evaluation before deployment.

Finally, this work should not be interpreted as removing the need for human supervision. The oracle, controller, and projection mechanism reduce certain failure modes, but they do not guarantee correctness in all settings. In safety-critical applications, hierarchical prompt-domain control should be viewed as one component of a broader reliability stack that includes validation, monitoring, interpretability, and human-in-the-loop review.

\newpage
\section{Connection to Literature}
\label{appendix:connection_literature}

Our framework is related to recent work on local agentic RL, prompt compression,
memory-augmented LLMs, and interactive agent benchmarks. The main distinction is
that we treat reliability as a \emph{feasible-state control problem}: the
controller preserves schema-defining information, compresses semantic history,
and projects the accumulated prompt state back into a feasible domain.

\paragraph{Connection to \texttt{PivotRL}.}
\texttt{PivotRL} performs local on-policy updates at informative intermediate
states, called pivots, selected from expert trajectories. In our notation, these
intermediate states correspond to accumulated prompt states,
$$
    s_t \equiv \mathcal{S}_t .
$$
Both methods therefore optimize policies over non-initial states, see Figure~\ref{fig:pivotrl_equivalence}. However,
\texttt{PivotRL} enforces correctness through verifier rewards, while our method
enforces feasibility directly at the state level. A verifier-style formulation is
$$
    r_{\mathrm{func}}(s,a)
    =
    \mathbf{1}_{\{a\in\mathcal{M}(s)\}},
$$
whereas our formulation imposes
$$
    \mathcal{S}_t\in\mathcal{D},
    \qquad
    \mathcal{S}_t \leftarrow \mathcal{P}_{\mathcal{O}}(\mathcal{S}_t)
    \quad
    \text{if } \mathcal{S}_t\notin\mathcal{D}.
$$
Thus, \texttt{PivotRL} is closest to our adaptive-update component, while our
additional contribution is an explicit feasible-domain geometry and projection
operator. Informative pivots correspond to high-drift states:
$$
    \mathrm{Var}_{a\sim\pi_\theta}[r(s_t,a)]>0
    \quad
    \Longleftrightarrow
    \quad
    D\!\left(
    \pi_\theta(\cdot\mid\mathcal{S}_t),
    \pi_\phi(\cdot\mid\mathcal{S}_t)
    \right)
    \text{ is large}.
$$
The advantage is that adaptation is triggered only where the student departs from
the oracle or violates feasibility, while projection prevents uncontrolled prompt
growth before it becomes a reward-level failure.

\paragraph{Connection to \texttt{LLMLingua}.}
\texttt{LLMLingua} compresses prompts to accelerate inference and reduce cost by
removing low-utility tokens while preserving task performance. Our projection
operator has a related goal, but the compression is structured around feasibility:
schema-defining information is separated from semantic history before any
compression is applied. Given a prompt $\rho$, let
$$
    Z(\rho)=Z_{\mathrm{sch}}(\rho)\cup Z_{\mathrm{sum}}(\rho)
$$
be its embedded representation, where $Z_{\mathrm{sch}}$ contains protocol,
format, and tool-interface units, and $Z_{\mathrm{sum}}$ contains compressible
semantic history. We keep the schema component fixed and summarize only the
semantic component:
$$
    \mathcal{P}_{\mathcal{O}}(\rho)
    =
    Z_{\mathrm{sch}}(\rho)\cup \mathcal{S},
    \qquad
    \mathcal{S}\subseteq Z_{\mathrm{sum}}(\rho),
    \qquad
    |\mathcal{S}|\le \kappa .
$$
The selected summary is obtained by exemplar-based clustering; see
Appendix~\ref{appendix:projection_beyond_mfbo}. Thus, our method can be viewed as
a structured compression layer that could encode token- or unit-pruning methods
as possible choices for constructing $Z_{\mathrm{sum}}(\rho)$, but with two
additional theoretical advantages: schema feasibility is preserved by
construction, and semantic compression is posed as a monotone submodular
maximization problem with the greedy guarantee of
Proposition~\ref{prop:greedy_info_summary}. We therefore do not compare directly
against \texttt{LLMLingua}, since its objective is generic prompt compression for
inference acceleration, whereas our experiments isolate feasibility-preserving
projection inside a controlled agentic loop with protocol constraints and
adaptive supervision.

\paragraph{Connection to \texttt{MemGPT}.}
\texttt{MemGPT} views long-context interaction as memory management across
different memory tiers. Our projection operator can be interpreted as a
feasibility-preserving memory-management rule. The accumulated state is split
into active schema memory and compressible semantic memory:
$$
    \mathcal{S}_t
    =
    \mathcal{S}_{t,\mathrm{sch}}
    \cup
    \mathcal{S}_{t,\mathrm{mem}} .
$$
When the state approaches the feasible boundary, the controller applies
$$
    \widetilde{\mathcal{S}}_t
    =
    \mathcal{P}_{\mathcal{O}}(\mathcal{S}_t)
    \in \mathcal{D}.
$$
Thus, while memory systems decide what to move between context and external
storage, our method specifies a projection objective for what should remain in
the active prompt. The advantage is that the retained memory is not only compact,
but also constrained to preserve protocol feasibility and admit submodular
coverage guarantees.

\paragraph{Connection to \texttt{WebArena}.}
\texttt{WebArena} evaluates agents on realistic web tasks involving long-horizon
interaction, tools, external knowledge, and functional task completion. These
settings naturally create growing prompt states containing instructions, web
observations, tool outputs, constraints, and previous actions. In our notation,
such an agent evolves according to
$$
    \mathcal{S}_{t+1}
    =
    \mathcal{S}_t
    \circ
    (o_t,a_t,r_t),
$$
where $o_t$ is the web observation, $a_t$ is the action, and $r_t$ is task
feedback. Our projection mechanism is directly applicable: fixed schema units
encode the browser/action protocol, while semantic units encode page content,
retrieved facts, and interaction history. This provides a principled way to
control context growth in web agents without discarding action-interface
constraints.

\paragraph{Connection to \texttt{ALFWorld}.}
\texttt{ALFWorld} studies agents that transfer between abstract text-based
policies and grounded embodied execution. This aligns with our separation
between protocol structure and semantic state. The admissible action interface
and task grammar define schema feasibility, while observations and object states
form compressible semantic context. A projected state can be written as
$$
    \widetilde{\mathcal{S}}_t
    =
    \mathcal{S}_{t,\mathrm{act}}
    \cup
    \Pi_{\kappa}(\mathcal{S}_{t,\mathrm{obs}}),
$$
where $\mathcal{S}_{t,\mathrm{act}}$ preserves admissible action syntax and
$\Pi_{\kappa}(\mathcal{S}_{t,\mathrm{obs}})$ summarizes the relevant environment
history. The advantage is that projection maintains the action protocol required
for grounded execution while compressing redundant observations.

\paragraph{Summary.}
Across these lines of work, our method can be viewed as adding a geometric
feasibility layer to agentic learning. \texttt{PivotRL} identifies informative
local states for adaptation, \texttt{LLMLingua} compresses prompts for efficient
inference, \texttt{MemGPT} manages memory across context tiers, and
\texttt{WebArena}/\texttt{ALFWorld} expose long-horizon interaction settings
where context growth is unavoidable. Our contribution is to combine these themes
through the constrained projection problem
$$
    \max_{\mathcal{S}\subseteq Z_{\mathrm{sum}}(\rho)}
    f(\mathcal{S})
    \quad
    \mathrm{s.t.}
    \quad
    |\mathcal{S}|\le\kappa,
    \qquad
    Z_{\mathrm{sch}}(\rho)\ \text{fixed},
$$
which gives a compact and theoretically supported mechanism for preserving
schema feasibility while compressing semantic state.


\begin{figure*}[t]
\centering
\resizebox{0.96\textwidth}{!}{%
    \begin{tikzpicture}[
        font=\scriptsize,
        node distance=1.4cm and 1.1cm,
        >=Latex,
        box/.style={draw, rounded corners=2pt, align=center, minimum height=8mm, minimum width=19mm},
        sbox/.style={draw, rounded corners=2pt, align=center, minimum height=7mm, minimum width=16mm},
        map/.style={draw, dashed, rounded corners=2pt, align=center, fill=gray!8, minimum height=7mm},
        title/.style={font=\bfseries},
        flow/.style={->, thick},
        eqv/.style={->, dashed, thick},
        note/.style={align=center, font=\scriptsize}
    ]
    
    \node[title] (ltitle) at (-2.0,4.2) {Our hierarchical framework};
    
    \node[box, fill=blue!8] (s0) at (-8.0,0.8) {$S_t$\\current system state};
    \node[box, fill=green!8, right=3mm of s0] (piu) {$a_t \sim \pi_\theta(\cdot\mid S_t)$\\adaptive action};
    \node[box, fill=orange!10, right=3mm of piu] (dom) {$S_t \in D?$\\feasible domain check};
    \node[box, fill=red!8, below=4mm of dom] (proj) {$S_t \leftarrow P_O(S_t)$\\projection};
    \node[box, fill=purple!8, right=4mm of dom] (rew) {$R_s + R_p$\\semantic + protocol\\ reward};
    \node[box, fill=yellow!12, right=3mm of rew] (upd) {update $\theta$\\distill / adapt};
    
    \draw[flow] (s0) -- (piu);
    \draw[flow] (piu) -- (dom);
    \draw[flow] (dom) -- node[above] {yes} (rew);
    \draw[flow] (dom) -- node[right] {no} (proj);
    \draw[flow] (proj.west) -| +(-4.1,0) |- (s0.south);
    \draw[flow] (rew) -- (upd);
    
    \node[sbox, fill=green!5, above=1.75cm of piu] (oracle) {$a_t^\star=O(\phi,S_t)$\\oracle supervision};
    \draw[flow] (oracle) -- (rew);
    
    \node[note, below=0.8cm of s0] (l1) {$S_t \equiv$ accumulated prompt/context};
    \node[note, below=0.3cm of piu] (l2) {policy over intermediate states};
    \node[note, below=1.2cm of dom] (l3) {explicit feasibility enforcement};
    \node[note, below=0.1cm of rew] (l4) {protocol + semantic credit};
    
    \node[title] (rtitle) at (8.6,4.2) {\texttt{PivotRL}};
    
    \node[box, fill=blue!8] (st) at (5.0,0.8) {$s_t$\\intermediate state};
    \node[box, fill=green!8, right=of st] (pip) {$a_t \sim \pi_\theta(\cdot\mid s_t)$\\sampled completion};
    \node[box, fill=orange!10, right=of pip] (ver) {$a_t \in \mathcal{M}(s_t)?$\\verifier};
    \node[box, fill=purple!8, right=of ver] (rf) {$r_{\mathrm{func}}(s_t,a_t)$\\functional reward};
    \node[box, fill=yellow!12, right=of rf] (updp) {local RL\\update $\theta$};
    
    \draw[flow] (st) -- (pip);
    \draw[flow] (pip) -- (ver);
    \draw[flow] (ver) -- (rf);
    \draw[flow] (rf) -- (updp);
    
    \node[sbox, fill=green!5, above=0.75cm of pip] (dsft) {$\mathcal{D}_{\mathrm{SFT}}$\\expert trajectories};
    \draw[flow] (dsft) -- (rf);
    
    \node[sbox, fill=red!8, below=1.0cm of ver] (pivot) {pivot filter\\high variance \\ hard states};
    \draw[flow] (st) -- (pivot);
    \draw[flow] (pivot.east) -| (updp.south);
    
    \node[note, above=0.2cm of st] (r1) {$s_t$ from expert trajectory prefix};
    \node[note, above=0.1cm of pip] (r2) {policy over intermediate states};
    \node[note, below=2.0cm of ver] (r3) {reward-based feasibility};
    \node[note, below=1.0cm of rf] (r4) {verifier-derived credit};
    
    \draw[eqv] (s0.north east) to[bend left=15] (st.north west);
    \draw[eqv] (oracle.east) to[bend left=8] (dsft.west);
    \draw[eqv] (rew.south east) to[bend right=10] (rf.south west);
    \draw[eqv] (proj.east) to[bend right=12] (ver.south west);
    \draw[eqv] (upd.south east) to[bend right=10] (pivot.south west);
    
    \node[draw, rounded corners=2pt, fill=gray!6, align=center, minimum width=12.8cm, minimum height=8mm]
        at (4.2,-3.4)
        {\textbf{Main distinction:} PivotRL operates on a \emph{flat state space} and enforces correctness through \emph{verifier rewards}, whereas our framework imposes an explicit \emph{feasible-domain geometry} $D$ and corrects violations by \emph{projection}.};
    
    \end{tikzpicture}
}
\caption{Sequential correspondence between the proposed hierarchical framework and \texttt{PivotRL}. Both optimize a policy over intermediate states, but \texttt{PivotRL} enforces feasibility through verifier-based local rewards, while our method enforces feasibility through projection into a feasible domain. This visualizes the Appendix X equivalences: $S_t \leftrightarrow s_t$, oracle supervision $\leftrightarrow$ expert trajectories, protocol reward $\leftrightarrow$ verifier reward, and drift detection $\leftrightarrow$ pivot selection.}
\label{fig:pivotrl_equivalence}
\end{figure*}
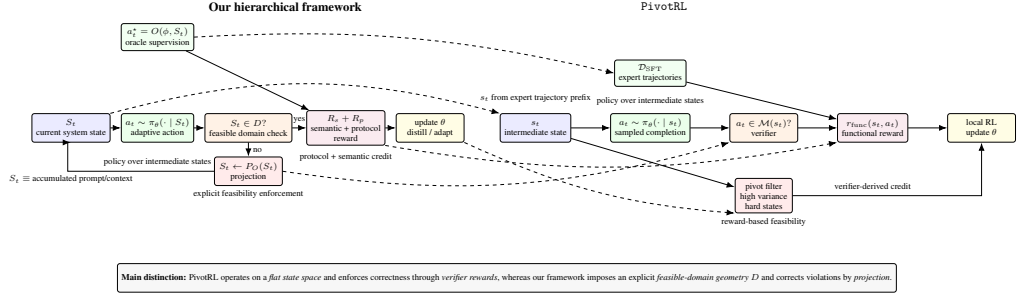

\end{document}